%% file: hym_preprint.tex
\documentclass[10pt,letterpaper]{article}

\usepackage[T1]{fontenc}
\usepackage[utf8]{inputenc}
\usepackage{times}
\usepackage[letterpaper,margin=1in]{geometry}
\usepackage{microtype}
\usepackage{amsthm}
\usepackage{amsmath}
\usepackage{amssymb}
\usepackage{mathtools}
\usepackage{mathrsfs}
\usepackage{bm}
\usepackage{graphicx}
\usepackage{algorithm}
\usepackage{algpseudocode}
\usepackage{caption}
\usepackage{placeins}
\usepackage{xcolor}
\usepackage{natbib}
\usepackage{url}
\usepackage{xurl}
\usepackage{hyperref}

\graphicspath{{./}}
\input{macros.tex}

\newtheorem{proposition}{Proposition}
\newtheorem{corollary}{Corollary}
\newtheorem{theorem}{Theorem}
\newtheorem{lemma}{Lemma}

\newtheorem{assumption}{Assumption}

\newcommand{\hym}{\texttt{HyM}}
\newcommand{\hymuon}{\texttt{HyMUON}}

\newcommand{\bDelta}{\boldsymbol{\Delta}}
\newcommand{\beps}{\boldsymbol{\varepsilon}}

\DeclareRobustCommand{\critical}[1]{\textcolor{red}{\textbf{[C: #1]}}}
\DeclareRobustCommand{\recom}[1]{\textcolor{blue}{\textbf{[R: #1]}}}
\DeclareRobustCommand{\minor}[1]{\textcolor{green!50!black}{\textbf{[M: #1]}}}

\hypersetup{
  colorlinks=true,
  linkcolor=blue!55!black,
  citecolor=green!40!black,
  urlcolor=blue!60!black,
  pdftitle={Hybrid Momentum Methods for Optimizing Deep Learning Models},
  pdfauthor={Haifeng Xu, Yijin Ren, Chenyu Zhou, Qi Deng, Dongdong Ge, Yinyu Ye}
}

\makeatletter
\renewcommand{\@maketitle}{%
  \newpage
  \null
  \vskip 2em%
  \begin{center}%
    {\LARGE\@title \par}%
    \vskip 1.6em%
    {\large \@author \par}%
  \end{center}%
  \par
  \vskip 1.5em}
\makeatother

\title{Hybrid Momentum Methods for Optimizing Deep Learning Models}
\author{%
  Haifeng Xu\textsuperscript{1}\thanks{Haifeng Xu and Yijin Ren contributed equally.}\quad
  Yijin Ren\textsuperscript{2}\footnotemark[1]\quad
  Chenyu Zhou\textsuperscript{3}\\[0.3em]
  Qi Deng\textsuperscript{3}\quad
  Dongdong Ge\textsuperscript{3}\quad
  Yinyu Ye\textsuperscript{3,4,5}\\[0.9em]
  {\normalsize
  \textsuperscript{1}Shanghai University of Finance and Economics\\
  \textsuperscript{2}Columbia University\\
  \textsuperscript{3}Shanghai Jiao Tong University\\
  \textsuperscript{4}SIMIS\\
  \textsuperscript{5}Stanford University}\\[0.9em]
  {\small\ttfamily
  \href{mailto:fengs10@outlook.com}{fengs10@outlook.com}\quad
  \href{mailto:yr2549@columbia.edu}{yr2549@columbia.edu}\quad
  \href{mailto:chenyuzhou@sjtu.edu.cn}{chenyuzhou@sjtu.edu.cn}\\
  \href{mailto:qdeng24@sjtu.edu.cn}{qdeng24@sjtu.edu.cn}\quad
  \href{mailto:ddge@sjtu.edu.cn}{ddge@sjtu.edu.cn}\quad
  \href{mailto:yinyu-ye@stanford.edu}{yinyu-ye@stanford.edu}}%
}
\date{}

\begin{document}

\maketitle

\input{abstract.tex}

\input{main.tex}

\bibliographystyle{plainnat}
\bibliography{ref}

\clearpage
\input{appendix.tex}

\end{document}

%% file: macros.tex
\makeatletter
\global\long\def\pnorm#1{\@ifnextchar_{\@normwith{#1}}{\@normdefault{#1}}}%
\global\long\def\@normwith#1_#2{\lVert#1\rVert_{#2}}%
\global\long\def\@normdefault#1{\lVert#1\rVert_2}%
\makeatother

\global\long\def\vertiii#1{\left\vert\kern-0.25ex\left\vert\kern-0.25ex\left\vert#1\right\vert\kern-0.25ex\right\vert\kern-0.25ex\right\vert}%

\global\long\def\zerobf{\mathbf{0}}%
\global\long\def\and{\mathrm{and}}%
\global\long\def\Ocal{\mathcal{O}}%
\global\long\def\bbf{\mathbf{b}}

\global\long\def\gbf{\mathbf{g}}

\global\long\def\mbf{\mathbf{m}}
\global\long\def\nbf{\mathbf{n}}

\global\long\def\ubf{\mathbf{u}}
\global\long\def\vbf{\mathbf{v}}
\global\long\def\wbf{\mathbf{w}}
\global\long\def\xbf{\mathbf{x}}
\global\long\def\ybf{\mathbf{y}}
\global\long\def\zbf{\mathbf{z}}

\global\long\def\Bbf{\mathbf{B}}

\global\long\def\Gbf{\mathbf{G}}

\global\long\def\Mbf{\mathbf{M}}

\global\long\def\Xbf{\mathbf{X}}
\global\long\def\Ybf{\mathbf{Y}}

\global\long\def\deltabf{\boldsymbol{\delta}}%

%% file: abstract.tex
\begin{abstract}
Momentum is central to stochastic optimization in deep learning, including the
training of large language models. Classical momentum methods typically rely
on a fast exponential moving average (EMA) scheme, which discounts past gradients
on a fixed timescale and may underweight older gradients that remain
informative. This creates a trade-off between retaining long-term history and
adapting to recent changes. We introduce Hybrid Momentum (\hym{}), which combines
a linearly weighted average of past gradients with a geometrically discounted fast
momentum. The former assigns weights proportional to the gradients' iteration
indices and retains a growing span of history, while the latter emphasizes recent
gradients. This hybrid construction is modular and can be combined with
adaptive preconditioning and other update transformations. In particular, we
propose \hymuon{}, which applies Newton--Schulz orthogonalization to the hybrid
direction. Illustrative examples and controlled experiments demonstrate the
benefits of combining gradient information across these two timescales.
Across image classification, machine translation, and language modeling, \hym{} and \hymuon{}
match or improve upon strong baselines. \hymuon{} improves over Muon on IWSLT14,
and in both the 100k- and 600k-step GPT-2 124M OpenWebText regimes. For
LLaMA-1.3B pretraining on C4, \hymuon{} achieves a final validation loss of
$2.777$, compared with $2.903$ for Muon.

\end{abstract}

%% file: main.tex
\section{Introduction}

Momentum methods improve stochastic optimization by combining the current gradient with information from earlier updates. In deep learning, this history is usually maintained through exponential moving averaging (EMA), which assigns geometrically decaying weights to past gradients, as in Heavy-Ball momentum and commonly used Nesterov updates \citep{polyak1964some,sutskever2013importance}. EMA emphasizes recent gradients while strongly reducing the influence of older gradients on subsequent updates. When gradient contributions that reduce a particular target loss occur intermittently across mini-batches, retaining a longer history may help accumulate these contributions.

By combining momentum with iterate averaging, \citet{defazio2024road} proposed the Schedule-Free method, which achieves strong empirical performance without a prescribed schedule for learning rate decay. The empirical success of Schedule-Free optimization motivates revisiting how this history is accumulated. To see how its averaging mechanism retains past gradients, consider the Schedule-Free SGD recurrence with a constant step size
\begin{equation}
\label{eq:sf-original}
\begin{aligned}
\ybf_t&=(1-\beta)\zbf_t+\beta\xbf_t,\\
\zbf_{t+1}&=\zbf_t-\eta\gbf_t,\\
\xbf_{t+1}&=(1-c_{t+1})\xbf_t+c_{t+1}\zbf_{t+1},
\end{aligned}
\end{equation}
where $\zbf_t$ is the base iterate, $\xbf_t$ is the averaged iterate, and $\ybf_t$ is the training iterate. Here $c_{t+1}=1/(t+1)$, $\zbf_1=\xbf_1$, and $\gbf_t$ is a stochastic gradient evaluated at $\ybf_t$. Eliminating $\xbf_t$ and $\zbf_t$ gives a representation using one sequence that exposes its gradient memory~\citep{morwani2025connections}:
\begin{equation}
\label{eq:sf-single}
\begin{aligned}
\nbf_{t} &= (1 - c_{t})\nbf_{t-1} + \gbf_{t},  \\
\ybf_{t+1} &= \ybf_{t} - \eta [{\beta c_{t+1} \nbf_{t} + (1 - \beta) \gbf_{t}}],
\end{aligned}
\end{equation}
with $c_t=1/t$ and $\nbf_0=\mathbf0$; Appendix~\ref{app:sf-reformulation} gives the derivation for completeness. In this form, the update combines a running weighted sum of past gradients with the current gradient. As training proceeds, this weighted sum spans a longer history, but each new gradient contributes a smaller fraction of its total weight. The historical component therefore responds more slowly to new gradients later in training, suggesting that a separate momentum state emphasizing recent gradients can complement this long-term accumulation.

Building on this perspective, we propose Hybrid Momentum (\hym{}), which combines a linearly weighted gradient average with a geometrically discounted fast momentum. The average assigns weights proportional to the gradients' iteration indices and retains a growing span of history, while the fast momentum emphasizes recent gradients. \hym{} also maintains an averaged iterate formed from its training iterates for evaluation and checkpointing.
Our convergence analysis via cumulative gradient weights accommodates this growing gradient memory
and establishes an $\Ocal(T^{-1/2})$ bound on the expected suboptimality of the
averaged iterate for smooth convex objectives.
Moreover, this hybrid momentum scheme is modular and can be combined with
adaptive preconditioning and other update transformations.
In particular, applying Muon's Newton--Schulz (NS) orthogonalization to the combined momentum~\citep{jordan2024muon}, we develop Hybrid Momentum Orthogonalized by NS (\hymuon{}), which applies to matrix-valued parameters.

We characterize the weights assigned to past gradients and use controlled
experiments to examine why this history can remain useful. These experiments show
that a gradient contribution retained in the optimizer state can still lower a target
loss thousands of steps later, and that replacing this contribution with a gradient
recomputed on the same batch at the current parameters can lower the loss further. We
also examine how gradients from different training examples vary in their usefulness
for a given target and how \hym{}'s relative performance changes with batch size.
Across the evaluated benchmarks, \hym{} matches or improves upon Schedule-Free on
image classification, while \hymuon{} improves upon Muon on IWSLT14 and across the
evaluated language modeling settings.

\textbf{Notation.}
The training and averaged iterates are $\ybf_t$ and $\xbf_t$; uppercase letters such as $\Ybf_t$ and $\Xbf_t$ denote the matrix variables. Vector norms are Euclidean, while $\|\cdot\|_{\mathrm{op}}$ and $\|\cdot\|_*$ denote the matrix operator and nuclear norms. Probabilistic notation is introduced where it is used.

\section{Related work}

Momentum methods use past information in two standard ways: by accumulating
gradients or by introducing inertial terms based on differences between successive
iterates \citep{polyak1964some,nesterov1983method}. For smooth convex stochastic
optimization, \citet{lan2012optimal} established an optimal convergence rate for
accelerated stochastic approximation; \citet{yang2016unified} and
\citet{ghadimi2016accelerated} analyzed stochastic momentum and accelerated gradient
methods for smooth nonconvex objectives. More recently,
\citet{wang2026generalized} unified Polyak and Nesterov variants within a generalized
SGDM framework and analyzed both convex and nonconvex smooth objectives. In deep
learning, EMA-Nesterov replaces the standard lookahead direction with an exponential
moving average of parameter updates, which stabilizes the extrapolation performed by a
base optimizer \citep{yau2026emanesterov}.

Several methods combine momentum components to capture different
timescales. Aggregated Heavy-Ball studies the convergence of schemes with multiple
components \citep{danilova2022aggregated}. AdEMAMix pairs two EMAs whose decay
coefficients differ by orders of magnitude, so that one retains gradients over
short histories and the other over long ones \citep{pagliardini2024ademamix}.
\hym{} also uses a fast state and a slow state, but replaces the slow EMA with
a linearly weighted gradient average. This replacement changes not only the length of the
retained history but also the distribution of weight over that history. AdEMAMix raises its
slow decay coefficient during warmup and then holds it at a prescribed value; consequently,
the fraction of its total gradient weight assigned to the first $\lfloor at\rfloor$ gradients
converges to zero for any fixed $a\in(0,1)$. For \hym{}, the corresponding fraction converges
to the positive constant $\frac{\beta}{2-\beta}a^2$.

\citet{morwani2025connections} show that Schedule-Free optimization
\citep{defazio2024road} can be expressed as accelerated SGD followed by
weight averaging, which separates the roles of gradient memory and
the averaged iterate. ScheduleFree+ restores Adam's inner state for the first moment and
extends this construction with Polyak adaptation of the learning rate, delayed
averaging, and an annealed interpolation between its raw and averaged iterates
\citep{defazio2026schedulefreeplus}. In ScheduleFree+, the averaged iterate is
still fed into the next gradient query; hence, the sequence that will be returned
also steers training. \hym{} separates the two. Its averaged iterate never
enters the gradient or the momentum recursion, and the update direction combines
the two momentum states alone. The convergence analyses also differ in kind.
The Schedule-Free analysis uses an online-to-batch conversion for convex
optimization \citep{defazio2024road}. Ours tracks how a gradient memory that grows
with training interacts with iterate averaging, and bounds the averaged iterate
directly.

Preconditioning is also important for training large language models.
Two common approaches are diagonal preconditioning, which scales each coordinate
by its own factor as in Adam and AdamW
\citep{kingma2015adam,loshchilov2019decoupled}, and structured matrix
preconditioning, as in Shampoo, K-FAC, and SOAP
\citep{gupta2018shampoo,martens2015kfac,vyas2025soap}.
More recently, Muon has been used to train large models
by applying NS orthogonalization to matrix momentum \citep{jordan2024muon}, which
motivated subsequent work on efficiency, stability, and distributed training
\citep{li2025normuon,kimiteam2025k2,ahn2025dion}.

\section{The Hybrid Momentum Methods}
\label{sec:methods}

\subsection{Hybrid Momentum (\hym{})}
\label{sec:our-method}

We consider stochastic optimization problems of the form
\begin{equation}
 \min_{\xbf\in\mathbb R^d} f(\xbf):=\mathbb E_\xi[f(\xbf;\xi)],
 \label{eq:stochastic-objective}
\end{equation}
where $\xi$ is a random variable. For this problem, we develop \hym{}
(Algorithm~\ref{alg:hym}) to retain
long-term gradient information while remaining responsive to recent gradients.
The method maintains two momentum states. The slow state
$\mathbf m_t$ is a linearly weighted average of the
gradients and draws on an increasingly long gradient history as training proceeds,
while the fast state $\mathbf b_t$ combines the gradients using weights similar to
those of an EMA and therefore assigns little weight to older gradients. \hym{} then combines the two
states to form the update direction. As in many classical optimization algorithms,
\hym{} maintains a training iterate and an averaged iterate. \hym{} evaluates the
gradient at the training iterate $\mathbf y_t$ and uses the combined update direction to obtain
$\mathbf y_{t+1}$. The averaged iterate
$\mathbf x_t$ is an average of the training iterates, and we evaluate the model
performance at $\mathbf x_t$. The
decay factor $\mu$ determines the geometric decay of the fast state, and the mixing
weight $\beta$ determines both the relative contributions of the two states and the
weight assigned to the current training iterate in the averaged iterate.

\begin{algorithm}[H]
\renewcommand{\algorithmicrequire}{\textbf{Input:}}
\renewcommand{\algorithmicensure}{\textbf{Output:}}
\caption{Hybrid Momentum (\hym{})}
\label{alg:hym}
\begin{algorithmic}[1]
\Require $\mathbf y_1\in\mathbb R^d$, step size $\eta>0$, coefficient
$\mu\in(0,1)$ for the fast momentum, mixing weight $\beta\in[0,1)$, batch size
$B\in\mathbb N$, and iterations $T\in\mathbb N$; initialize
$\mathbf m_0=\mathbf b_0=\mathbf0$, $\mathbf x_0=\mathbf y_1$
\For{$t=1,\ldots,T$}
\State Sample a mini-batch $\mathcal B_t=(\xi_{t,j})_{j=1}^B$
\State $\mathbf g_t=\frac{1}{B}\sum_{j=1}^B\nabla f(\mathbf y_t;\xi_{t,j})$
\State $\mathbf m_t=\frac{t-1}{t+1}\mathbf m_{t-1}+\frac{2}{t+1}\mathbf g_t$
\Comment{Slow state}\label{algline:hym-slow}
\State $\mathbf b_t=\mu\left(1-\frac{1}{t}\right)\mathbf b_{t-1}+\mathbf g_t$
\Comment{Fast state}\label{algline:hym-fast}
\State $\mathbf\Delta_t=\beta r_t\mathbf m_t+(1-\beta)\mathbf b_t$, where
$r_t=\frac{t+1}{2t}\frac{1-\mu^t}{1-\mu}$
\label{algline:hym-direction}
\State $\mathbf x_t=(1-a_t)\mathbf x_{t-1}+a_t\mathbf y_t$, where $a_t=\frac{1}{1+(1-\beta)(t-1)}$
\label{algline:hym-average}
\State $\mathbf y_{t+1}=\mathbf y_t-\eta\mathbf\Delta_t$
\label{algline:hym-step}
\EndFor
\Ensure $\mathbf x_T$
\end{algorithmic}
\end{algorithm}

\paragraph{Two momentum states.}
The two states can be written as the following weighted sums of past gradients:
\begin{equation}
 \begin{aligned}
 \mathbf m_t&=\frac{\sum_{i=1}^t i\,\mathbf g_i}{\sum_{i=1}^t i}
 =\sum_{i=1}^t\frac{2i}{t(t+1)}\mathbf g_i, \ \text{and} \quad
 \mathbf b_t&=\sum_{i=1}^t\frac{i}{t}\mu^{t-i}\mathbf g_i.
 \end{aligned}
 \label{eq:hym-momentum-weights}
\end{equation}
At iteration $t$, the slow state assigns $\mathbf g_i$ the weight
$2i/[t(t+1)]$, which increases linearly with $i$ across the entire training history.
The fast state also contains the factor $i$, together with the additional factor
$\mu^{t-i}$. This discount concentrates its weight on recent gradients. The decay
coefficient $\mu(1-1/t)$ converges to $\mu$. Hence, for large $t$, the fast state
behaves like conventional EMA momentum and differs from an EMA with a fixed decay rate mainly
early in training.
Because both states are maintained recursively, neither requires storing past gradients
individually.

\hym{} uses the mixing weight $\beta$ to form the update direction as a linear
combination of the scaled slow state $r_t\mathbf m_t$ and the fast state $\mathbf b_t$.

\paragraph{Training and averaged iterates.}
Unrolling the averaging recursion gives
$\mathbf x_t=\sum_{i=1}^t\pi_{t,i}\mathbf y_i$, where
$\pi_{t,i}=a_i\prod_{j=i+1}^t(1-a_j)$ and
$\sum_{i=1}^t\pi_{t,i}=1$. The averaging coefficient $a_t$ is determined by $\beta$,
which also controls the relative contributions of the two momentum states. For fixed
$t$, a larger $\beta$ places more weight on the current training iterate
$\mathbf y_t$. The averaged iterate is a running
average of the training iterates rather than an independent sequence, and it never
enters the gradient or the momentum recursion. The training trajectory is therefore
$\{\mathbf y_t\}$ alone, and the sequence $\{\mathbf x_t\}$ is used for evaluation,
with the algorithm returning $\mathbf x_T$.

\subsection{Extension to Matrix-Valued Parameters: \hymuon{}}
\label{sec:hymuon}

In practical machine learning, many model parameters are represented as matrices,
and Muon exploits this structure effectively. Motivated by this observation, we
consider the following matrix-valued stochastic optimization problem:
\begin{equation}
 \min_{\Xbf\in\mathbb R^{m\times n}} f(\Xbf):=\mathbb E_\xi[f(\Xbf;\xi)].
 \label{eq:matrix-stochastic-objective}
\end{equation}
For hidden weight matrices, \hymuon{} applies the two momentum recurrences of \hym{} in Algorithm~\ref{alg:hym} to matrix gradients. It then applies Muon's NS orthogonalization to the combined momentum $\mathbf D_t$ \citep{jordan2024muon}:
\begin{equation}
    \Ybf_{t+1}=\Ybf_t-\eta\,\mathrm{NS}(\mathbf D_t),
    \label{eq:hymuon-matrix-update}
\end{equation}
where $\mathrm{NS}(\cdot)$ denotes the output of finitely many NS iterations approximating the polar factor. The averaged matrix iterate $\Xbf_t$ follows the same averaging rule as in \hym{}. When the optimizer is applied to the entire model, hidden weight matrices form the matrix group above; all remaining parameters form an auxiliary group that applies the same construction with slow and fast states separately to estimates of the first and second moments, followed by coordinatewise scaling as in Adam. Appendix~\ref{app:hymuon-analysis} gives both update recurrences and the implementation for the entire model.

\section{Why Past Gradients Can Remain Useful}
\label{sec:hym-understanding}

From Algorithm~\ref{alg:hym}, we can derive that, for any proportion
$a\in(0,1)$, the first $\lfloor at\rfloor$ gradients account for a fraction of
\hym{}'s total gradient weight that converges to
$\frac{\beta}{2-\beta}a^2$, which is positive, as $t\to\infty$, whereas the
corresponding fraction under EMA momentum converges to zero.
Thus, compared with EMA momentum, \hym{} asymptotically retains more weight on
older gradients. In this section, we examine why retaining such a long history can
be useful. We first
characterize how the two states distribute weight over the gradient history. We then
test whether a retained gradient can still affect training thousands of steps later.
Finally, we examine whether gradients computed from examples in the actual training
data differ substantially in their usefulness for a given target.

\subsection{Retaining history while responding to change}
\label{sec:hym-weights}

\hym{} combines two momentum states that cover different ranges of the gradient
history. The slow state draws on an increasingly long history, whereas the fast state
remains concentrated on recent gradients. We quantify this difference using the
average lag and then examine the resulting weights in the update direction.

At iteration $t$, a gradient $\mathbf g_i$ computed at iteration $i$ is $t-i$
iterations old. To summarize the age of the gradients used by a momentum state, we
normalize its weights to sum to one and define the average lag as the weighted average
of $t-i$ over $i=1,\ldots,t$. It is $(t-1)/3$ for the slow state and converges to
$\mu/(1-\mu)$ for the fast state as $t$ grows. Hence, the
average lag of the slow state increases with $t$, whereas that of the fast state
remains bounded.

Unrolling the two momentum states gives the weight $w_{t,i}$ assigned to
$\mathbf g_i$ in the update direction:
\begin{equation}
 \mathbf\Delta_t=\sum_{i=1}^t w_{t,i}\mathbf g_i,\qquad
 w_{t,i}=\frac{i}{t}
 \left(\frac{\beta(1-\mu^t)}{t(1-\mu)}+(1-\beta)\mu^{t-i}\right).
 \label{eq:hym-kernel}
\end{equation}
Figure~\ref{fig:weight_profiles} compares these weights with those of classical
EMA momentum and AdEMAMix at $t=1{,}000$, $10{,}000$, and $50{,}000$. For each
method, it plots the weight on $\mathbf g_i$ against $i$, normalized by the weight
on the most recent gradient $\mathbf g_t$. The figure therefore shows how the
relative weight on older gradients changes as training proceeds.

In the left panel of Figure~\ref{fig:weight_profiles}, \hym{}'s weights rise
approximately linearly with $i$. In the middle and right panels, the geometric term is
increasingly concentrated near $i=t$, producing a peak over the most recent gradients.
For fixed $\mu,\beta\in(0,1)$ and any $a\in(0,1)$, the gradients with $i\le at$
retain a positive fraction of the total weight in the \hym{} update direction as
$t\to\infty$ (Appendix~\ref{app:hym-history-proof}). In contrast, for an EMA with
any fixed decay rate below one, the corresponding fraction converges to zero.
Together, the two states allow the update direction to retain weight over an expanding
history while remaining responsive to recent gradients. Whether a gradient retained
in this history can still affect training after many subsequent updates is an empirical
question, which we examine next.

\begin{figure}[t]
 \centering
 \includegraphics[width=0.32\linewidth]{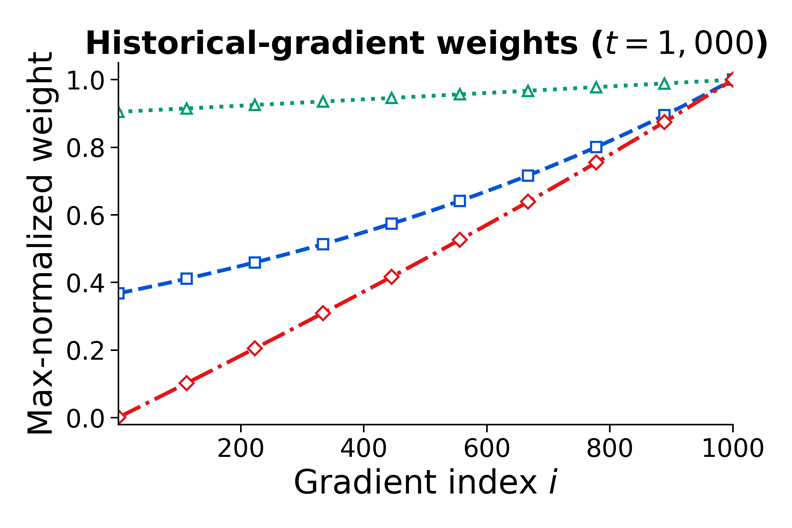}
 \includegraphics[width=0.32\linewidth]{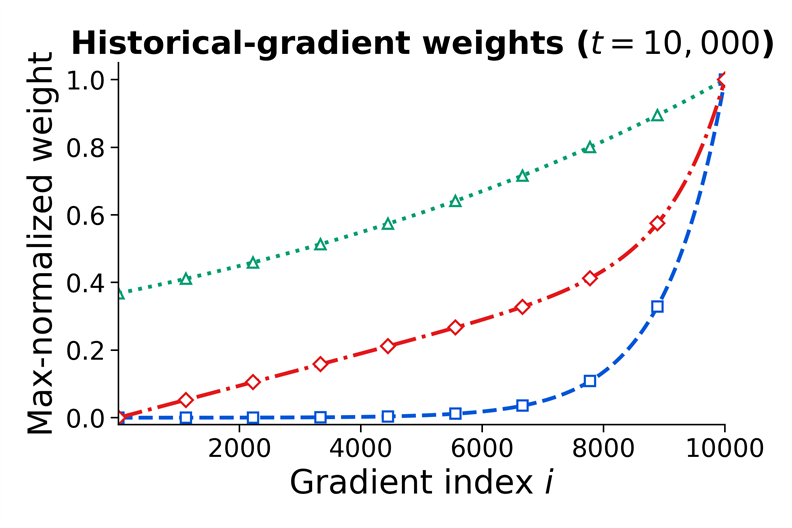}
 \includegraphics[width=0.32\linewidth]{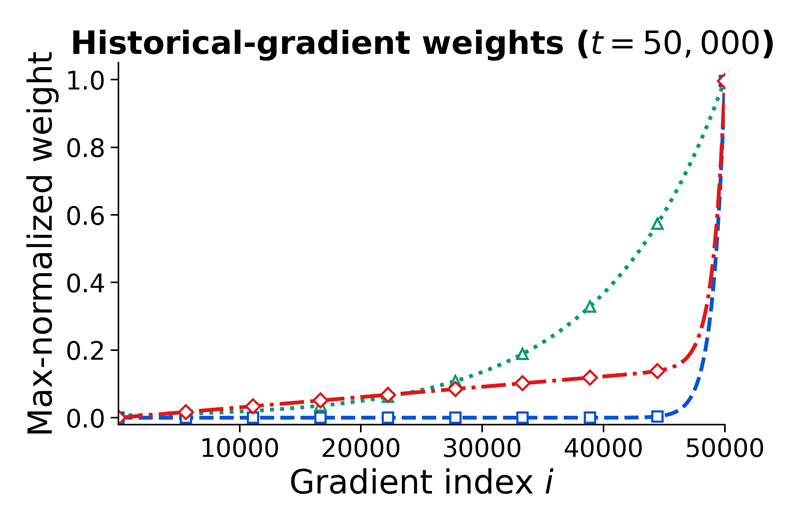}\\[-2pt]
 \includegraphics[width=0.96\linewidth]{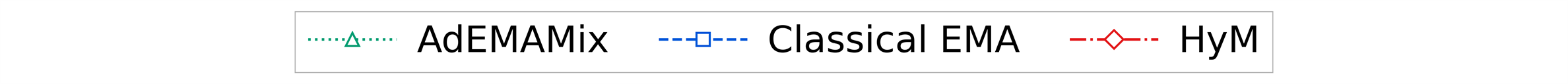}
 \caption{Gradient weight profiles for
 AdEMAMix ($\beta_3=0.9999$), classical EMA ($\mu=0.999$), and \hym{}
 ($\mu=0.999$, $\beta=0.9$) at $t\in\{1000,10000,50000\}$ (left to
 right). Each curve is normalized by the weight on the most recent gradient.}
 \label{fig:weight_profiles}
\end{figure}

\subsection{Old gradients can still affect later steps}
\label{sec:hym-forgetting}
\label{sec:hym-single-batch}

To test whether the effect of a single batch persists during subsequent
training, we adapt an experiment from AdEMAMix
\citep{pagliardini2024ademamix} to GPT-2 124M trained on OpenWebText. We exclude
a fixed batch $\mathcal B$ from the regular training sequence and use its loss
to track the lasting effect of a single update. Starting from a common checkpoint
immediately before step 5,000, we continue training along two branches. At step
5,000, the control branch uses the batch specified by the original training
sequence, whereas the inject branch replaces that batch with $\mathcal B$. Both
branches subsequently use the same sequence of training batches. At step 10,000,
the loss on $\mathcal B$ in the control branch exceeds that in the inject branch
by approximately $0.030$ for Muon and $0.016$ for \hymuon{}
(Figure~\ref{fig:forgetting}, left). Thus, the effect of a single
update based on $\mathcal B$ remains measurable after 5,000 subsequent training
steps.

\label{sec:hym-state-intervention}
We next conduct a further experiment by modifying the contribution of
$\mathcal B$ retained in the slow state. At step 7,500, we split the inject
branch from the preceding experiment into three branches: Keep, Delete, and
Refresh. At this point, the contribution introduced at step 5,000 still has
weight $5{,}000/7{,}500=2/3$ in the slow state. Keep leaves this contribution
unchanged, Delete removes it, and Refresh replaces it with a gradient computed
on $\mathcal B$ at the current model parameters and rescaled to have the same
norm. At step 7,500, the three branches differ only in this contribution to the
slow state; their model parameters, averaged iterate $\mathbf x_t$, and all
other optimizer states are otherwise identical.
Appendix~\ref{app:forgetting-state} gives the precise construction.

\begin{figure}[!t]
 \centering
 \includegraphics[width=0.49\linewidth]{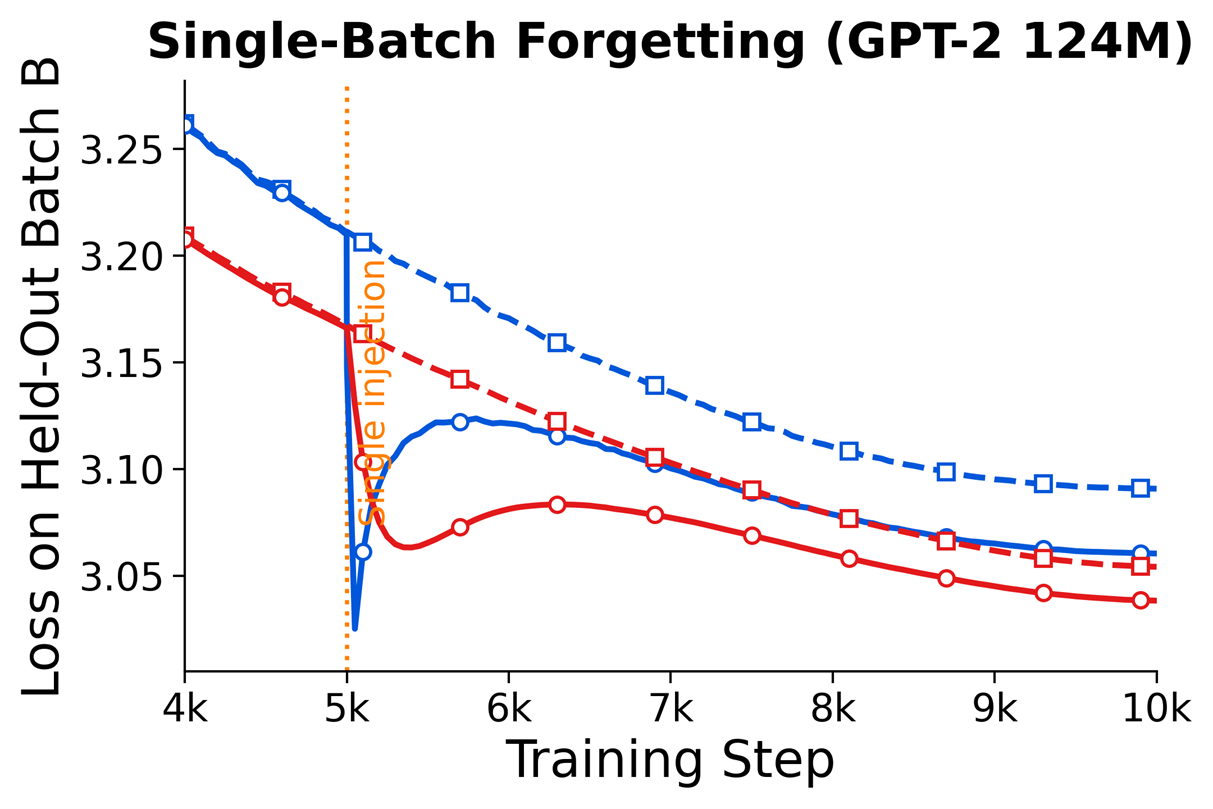}\hfill
 \includegraphics[width=0.49\linewidth]{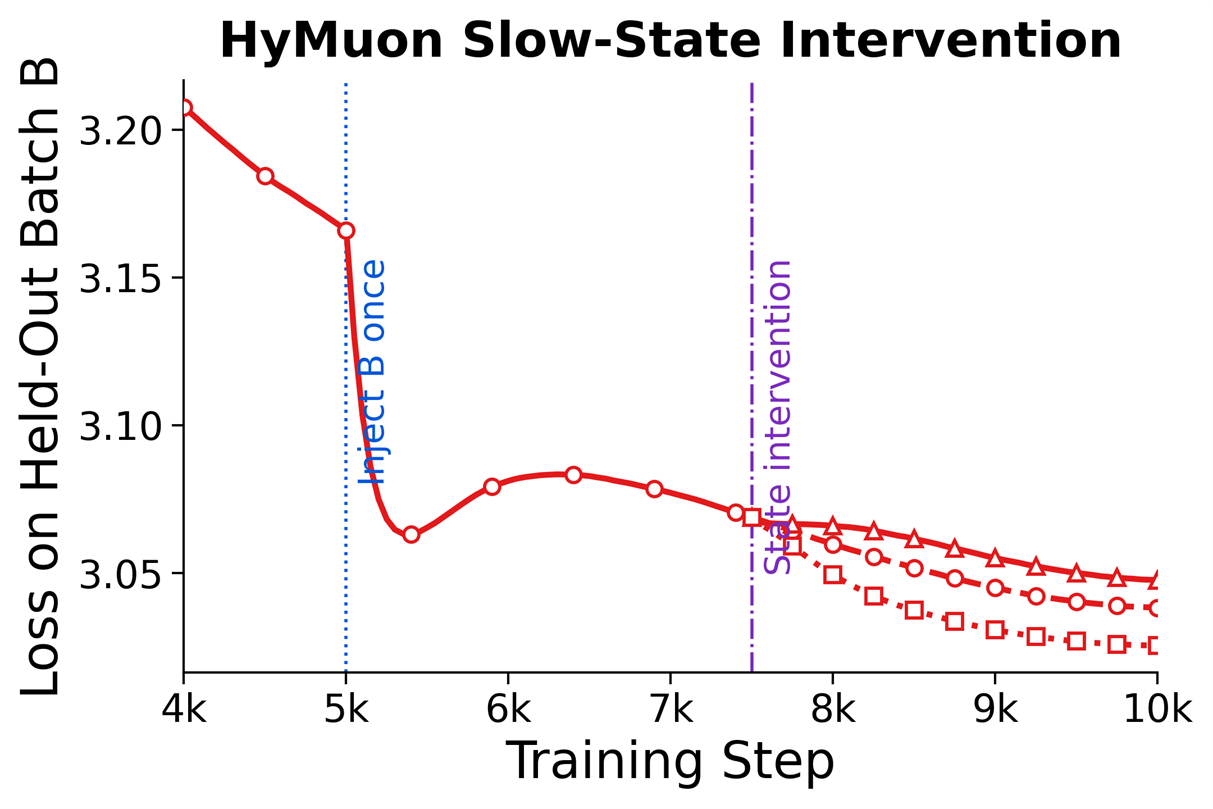}\\[-2pt]
 \includegraphics[width=0.98\linewidth]{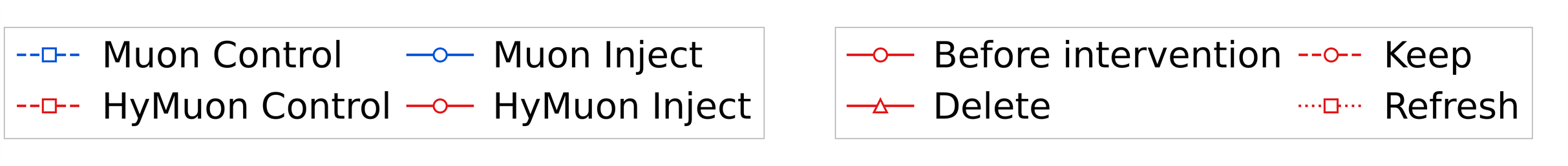}
 \caption{Persistence of the effect of a single batch and of its contribution
 retained in the slow state. Left: loss on batch $\mathcal B$ for
 the control and inject branches. The two branches differ at step 5,000, when
 the inject branch takes one optimization step using the gradient computed on
 $\mathcal B$; they then continue with the same sequence of batches. Right: at
 step 7,500, the inject branch is split into Keep, Delete, and Refresh, which
 retain, remove, or replace the contribution of $\mathcal B$ in the slow state.
 At this step, the three branches have the same model parameters, averaged
 iterate, and all other optimizer states. Refresh uses a gradient computed on
 $\mathcal B$ at the current model parameters and rescales it to match the
 groupwise norm of the original contribution. \hymuon{} curves are evaluated
 at the averaged iterate, whereas Muon curves are evaluated at the training
 iterate; all curves are unsmoothed. Full protocols are given in
 Appendices~\ref{app:forgetting} and~\ref{app:forgetting-state}.}
 \label{fig:forgetting}
 \label{fig:forgetting-state}
\end{figure}

At step 10k, \hymuon{} Delete, Keep, and Refresh give target losses of
$3.048$, $3.038$, and $3.025$, respectively. Its stored contribution
therefore still lowers the target loss, and refreshing it lowers the loss
further. Muon's Keep, Delete, and Refresh values differ by less than $0.002$
and overlap at the scale of the plot. Together, the two panels of
Figure~\ref{fig:forgetting} show that a single update on $\mathcal B$ can have
a lasting effect and that, in \hymuon{}, the contribution retained in the slow
state continues to affect the loss on $\mathcal B$.

\subsection{Useful contributions are uneven across training examples}
\label{sec:hym-narrow-signal}

The preceding experiment deliberately inserts a fixed batch $\mathcal B$ to isolate
the persistence of its effect. We next ask whether examples drawn from the training data
naturally produce gradients that differ in their usefulness for the same target.
In language model pretraining, data are drawn from a distribution that can be divided
into several domains. For a particular domain, gradients computed from different
training examples may reduce its loss by different amounts. We examine this variation
across candidate passages at fixed model parameters.

To test this idea, we use groups of semantically related passages as semantic
targets, each representing a specific part of the data distribution. We conduct
this experiment at three checkpoints of GPT-2
124M and use 2,000 candidate passages from the training data, each containing
256 tokens. For each candidate passage, we restore the model to the same
checkpoint, compute the gradient on that passage, normalize the gradient, and
take one step with a fixed length. All candidates are therefore compared from
the same model parameters, without any influence from optimizer state. We use
E5 embeddings \citep{wang2022e5} to construct ten semantic targets from
passages in a separate validation set. Each semantic target is divided into a
reference set and an evaluation set. We use the same pretrained E5 encoder to
measure the semantic similarity between each candidate and the reference set.
After taking the candidate's gradient step, we measure how much this step
reduces the loss on the evaluation set. Appendix~\ref{app:narrow-signal}
provides additional construction details and statistics.

Figure~\ref{fig:narrow-percentiles} summarizes the results of this experiment.
It groups the candidate passages by their E5 similarity to each semantic target
and reports the mean improvement within each percentile. At all three checkpoints,
the improvement generally increases with similarity and rises particularly sharply
among the most similar candidates. Moreover, at each checkpoint, the $20\%$ of
candidates with the highest similarity account for approximately one third of the
positive improvement specific to the corresponding targets.
Appendix~\ref{app:narrow-signal} provides the precise definition, exact statistics,
and results for individual targets.

\begin{figure}[!htbp]
    \centering
    \includegraphics[width=\linewidth]{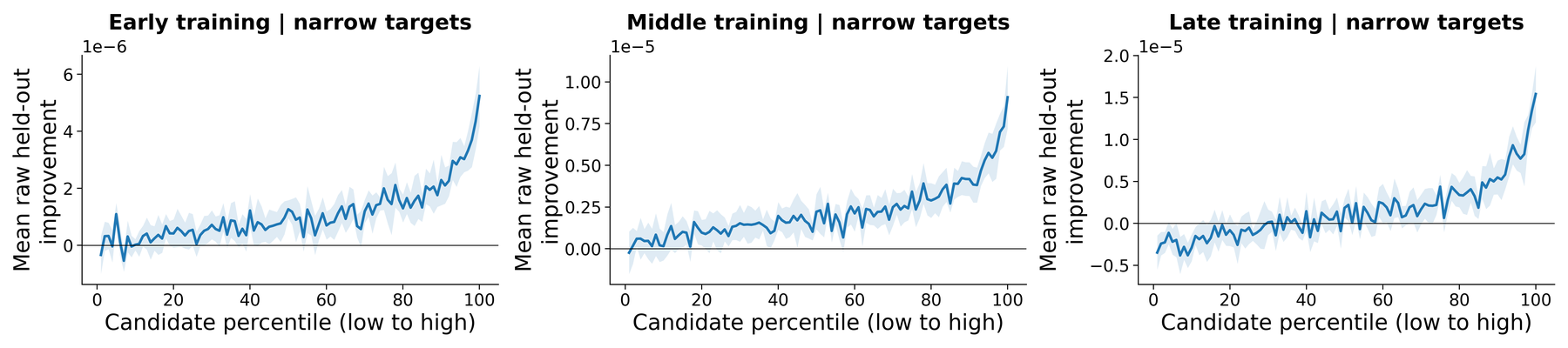}
    \caption{Loss reduction on each semantic target's evaluation set after a single
    update using the gradient of a candidate passage, at the 10k, 50k, and 100k
    checkpoints (left to right). For each target, the 2,000 candidate passages are
    ranked by their E5 similarity to the target's reference set and divided into 100
    groups of 20. The blue curves show the mean across ten semantic targets, and the
    shaded regions show error bands across targets. Positive values indicate a
    reduction in loss, and the vertical scales differ across panels.}
    \label{fig:narrow-percentiles}
\end{figure}

The experiment in Figure~\ref{fig:forgetting} shows that the effect of a single
gradient update can remain measurable thousands of steps later. The experiment in
Figure~\ref{fig:narrow-percentiles} shows that gradient usefulness varies substantially
across training passages: at each checkpoint, the candidates with the highest
similarity constitute only one fifth of all candidates but account for approximately
one third of the positive improvement specific to the corresponding targets. Thus,
gradients that are particularly useful for a target are not distributed evenly across
training examples. Together, these results support retaining gradients over a long
history. When such gradients occur intermittently, retaining their contributions
allows subsequent updates to continue using this information when later training
examples provide little useful gradient information for that target. The slow state
of \hym{} is designed to retain such contributions effectively in subsequent updates.

\section{Convergence Analysis}
\label{sec:hym-guarantees}

In this section, we establish the convergence of \hym{} for smooth convex objectives.
For EMA momentum with a fixed decay factor $\mu$, the average lag approaches the
constant $\mu/(1-\mu)$. Therefore, the average age of the gradients remains bounded as training
proceeds. In contrast, the slow state of \hym{} has average lag $(t-1)/3$, which grows
with $t$. To accommodate this growing history, we analyze a general representation in
which the cumulative weight of each stochastic gradient may vary with the iteration,
and show that the averaged iterate of \hym{} attains an $\Ocal(T^{-1/2})$ rate.

For the analysis, we introduce $\vbf_t$ through
$\ybf_t=(1-\beta)\vbf_t+\beta\xbf_t$.
From the recurrences in Algorithm~\ref{alg:hym}, we obtain
$\vbf_t=t\xbf_t-(t-1)\xbf_{t-1}
=\vbf_1-\eta\sum_{k=1}^{t-1}K_{t,k}\gbf_k$.
The first equality implies $\xbf_t=\frac1t\sum_{i=1}^t\vbf_i$, and
$-\eta K_{t,k}$ is the coefficient of $\gbf_k$ in $\vbf_t$.
For a fixed $\Gamma>0$, define $e_{t,k}=K_{t,k}-\Gamma$.
We consider the general form
\begin{equation}
 \begin{aligned}
 \xbf_t&=\frac1t\sum_{i=1}^t\vbf_i,\qquad
 \ybf_t=(1-\zeta_t)\vbf_t+\zeta_t\xbf_t,\\
 \vbf_t&=\vbf_1-\eta\sum_{k=1}^{t-1}
 (\Gamma+e_{t,k})\gbf_k.
 \end{aligned}
 \label{eq:kernel-framework}
\end{equation}
The coefficients $e_{t,k}$ are deterministic, and $\zeta_t\in[0,1]$ is chosen using
information available before the mini-batch at iteration $t$ is sampled.
For \hym{}, $\zeta_t=\beta$.
For Schedule-Free SGD, $K_{t,k}=1$ for all $1\le k<t$, and
$\zeta_t=\beta$. Choosing $\Gamma=1$ therefore gives $e_{t,k}=0$, in which case
\eqref{eq:kernel-framework} reduces to Schedule-Free SGD
\citep{defazio2024road}.

\begin{assumption}
\label{ass:convex-objective}
The objective $f:\mathbb R^d\to\mathbb R$ is convex, admits a minimizer
$\xbf^*$, and is $L$-smooth for some $L>0$: for all $\ubf,\wbf\in\mathbb R^d$,
$\|\nabla f(\ubf)-\nabla f(\wbf)\|\le L\|\ubf-\wbf\|$.
\end{assumption}

\begin{assumption}
\label{ass:minibatch-oracle}
Let $(\mathcal F_t)_{t\ge0}$ be the filtration generated by the training
history, where $\mathcal F_{t-1}$ contains all information available before
sampling $\mathcal B_t$. The iterate $\ybf_t$ is
$\mathcal F_{t-1}$-measurable, and $\gbf_t$ is $\mathcal F_t$-measurable.
Conditional on $\mathcal F_{t-1}$, the random variables
$\xi_{t,1},\ldots,\xi_{t,B}$ are independent copies of the random variable
$\xi$ in~\eqref{eq:stochastic-objective}. For some finite $\sigma\ge0$,
\[
 \mathbb E_\xi[\nabla f(\xbf;\xi)]=\nabla f(\xbf),\qquad
 \mathbb E_\xi\!\left[
 \|\nabla f(\xbf;\xi)-\nabla f(\xbf)\|^2
 \right]\le \sigma^2, \quad \forall \xbf\in\mathbb R^d.
\]
\end{assumption}

\begin{theorem}
\label{thm:kernel-square-stability}
Under Assumptions~\ref{ass:convex-objective} and~\ref{ass:minibatch-oracle},
suppose \eqref{eq:kernel-framework} holds with deterministic $\vbf_1$ and
$e_{t,k}$, and with each $\zeta_t\in[0,1]$ determined by
$\mathcal F_{t-1}$. Assume that
$\sup_{t\ge1}\sum_{k=1}^{t-1}e_{t,k}^2\le C_E<\infty$ for some $C_E>0$.
For a run of $T$ iterations, set $\eta=\eta_0/\sqrt T$,
where $\eta_0\in(0,[2L(\Gamma+2\sqrt{C_E})]^{-1}]$. Then
\begin{equation}
 \mathbb E[f(\xbf_T)-f(\xbf^*)]
 \le
 \frac{\|\vbf_1-\xbf^*\|^2/(2\eta_0\Gamma)+\eta_0\Gamma\varsigma^2/2}{\sqrt T}
 +\frac{2L\eta_0^2 C_E\varsigma^2}{T}.
 \label{eq:kernel-square-rate}
\end{equation}
where $\varsigma=\sigma/\sqrt B$.
\end{theorem}

For fixed momentum parameters, \hym{} satisfies the condition on $e_{t,k}$
in Theorem~\ref{thm:kernel-square-stability}, yielding the following convergence
guarantee.

\begin{corollary}
\label{cor:hym-convex-x}
Under Assumptions~\ref{ass:convex-objective} and~\ref{ass:minibatch-oracle},
consider Algorithm~\ref{alg:hym} with deterministic $\ybf_1$ and fixed
$\mu\in(0,1)$ and $\beta\in[0,1)$.
Choose $\eta=\eta_0/\sqrt T$, where $\eta_0>0$ is sufficiently small and
depends only on $L$, $\mu$, and $\beta$, not on $T$.
Then $\mathbb E[f(\xbf_T)-f(\xbf^*)]=\Ocal(T^{-1/2})$.
\end{corollary}
For fixed $\eta_0$, $\mu$, and $\beta$, the constants are independent of $T$.
They grow, however, as $\mu$ or $\beta$ approaches $1$.
The result requires neither strong convexity nor uniformly bounded gradients.

\section{Empirical studies}
\label{sec:experiments}

We evaluate \hym{} and \hymuon{} on image classification, translation, and
language modeling. Figures~\ref{fig:image-classification} and
\ref{fig:language-tasks} show the main learning curves. Appendix~\ref{app:full-exp}
gives the complete configurations and evaluation protocols, and
Table~\ref{tab:main-results} summarizes the numerical results. CIFAR and SVHN use
their test splits, IWSLT reports test metrics, and ImageNet uses its official
validation split. The vision and IWSLT curves show means across repeated runs, and the shaded
regions show one standard error. The values in the vision panel legends are endpoint
means with sample standard deviations.

\begin{figure}[!htb]
    \centering
    \includegraphics[width=0.325\linewidth]{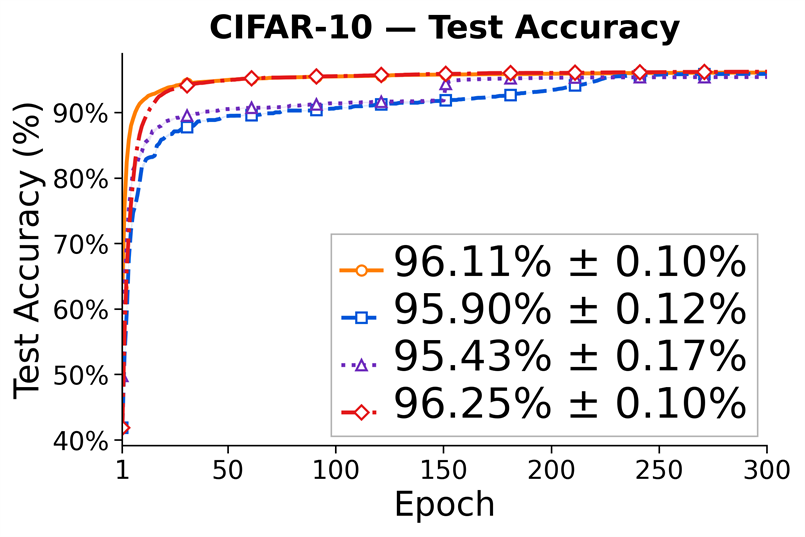}\hfill
    \includegraphics[width=0.325\linewidth]{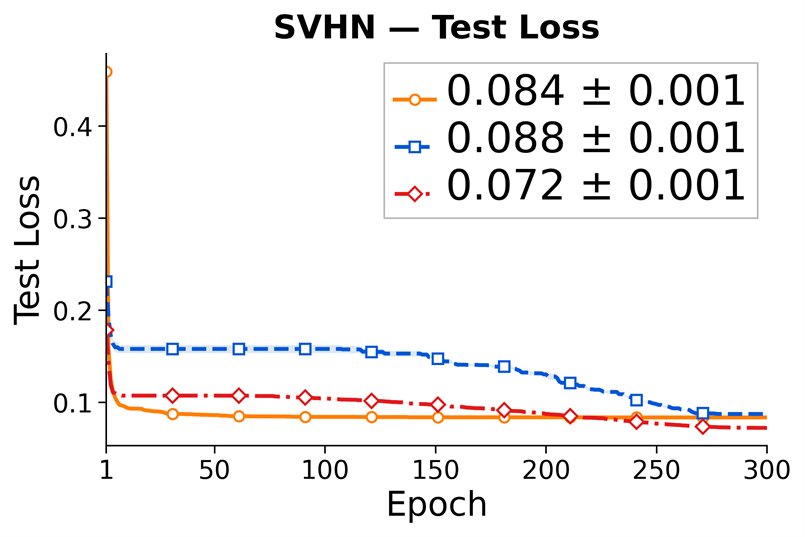}\hfill
    \includegraphics[width=0.325\linewidth]{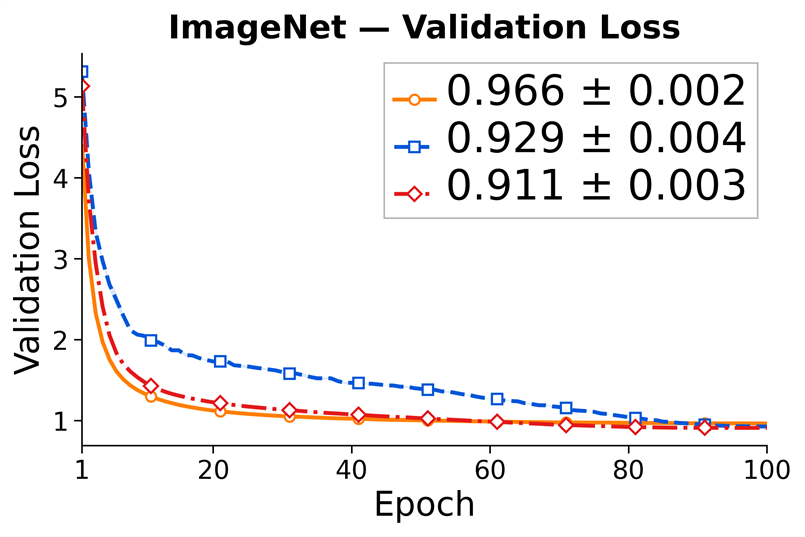}\\[4pt]
    \includegraphics[width=0.325\linewidth]{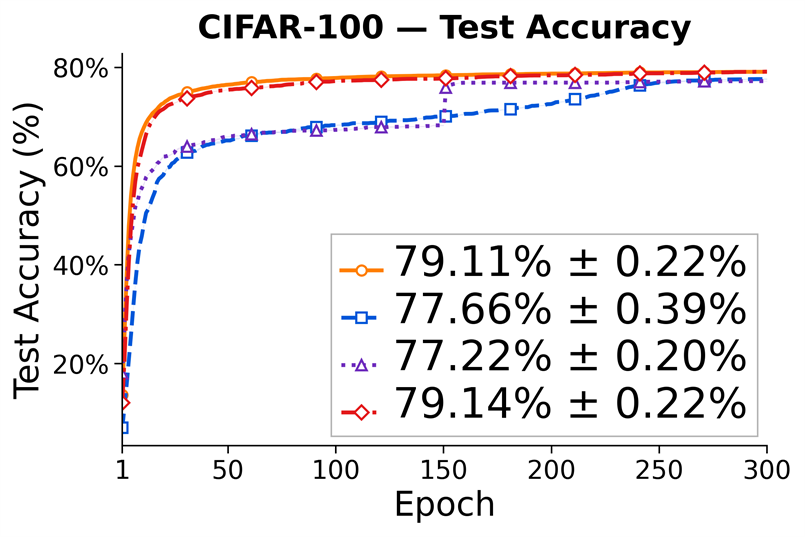}\hfill
    \includegraphics[width=0.325\linewidth]{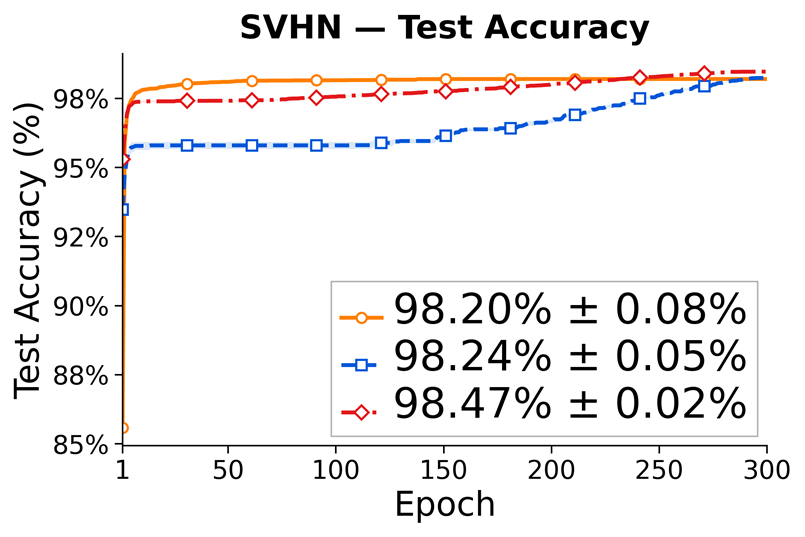}\hfill
    \includegraphics[width=0.325\linewidth]{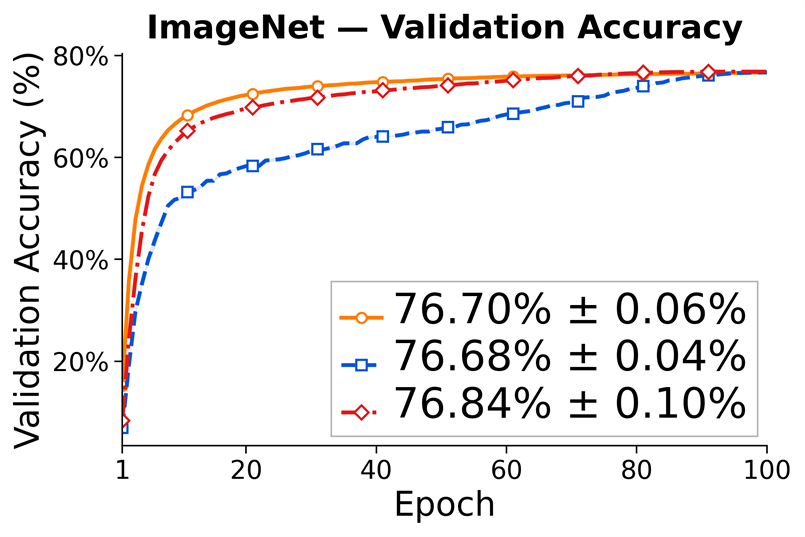}\\[-2pt]
    \includegraphics[width=0.975\linewidth]{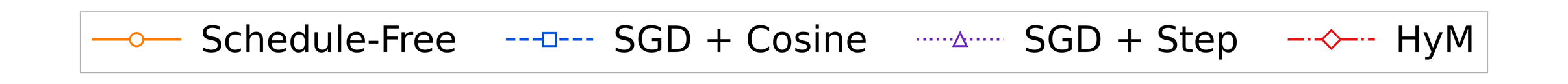}
    \caption{Image classification. Left: CIFAR-10 (top) and CIFAR-100 (bottom)
    test accuracy. Middle: SVHN test loss and accuracy. Right: ImageNet-1K
    validation loss and accuracy. Curves show means across repeated runs, and the
    shaded regions show one standard error. Values in the panel legends are endpoint
    means with sample standard deviations. Appendix
    Figure~\ref{fig:cifar-test-loss} adds the CIFAR test loss curves.}
    \label{fig:image-classification}
    \label{fig:small-scale}
    \label{fig:large-scale}
\end{figure}

\begin{figure}[!htb]
    \centering
    \includegraphics[width=0.325\linewidth]{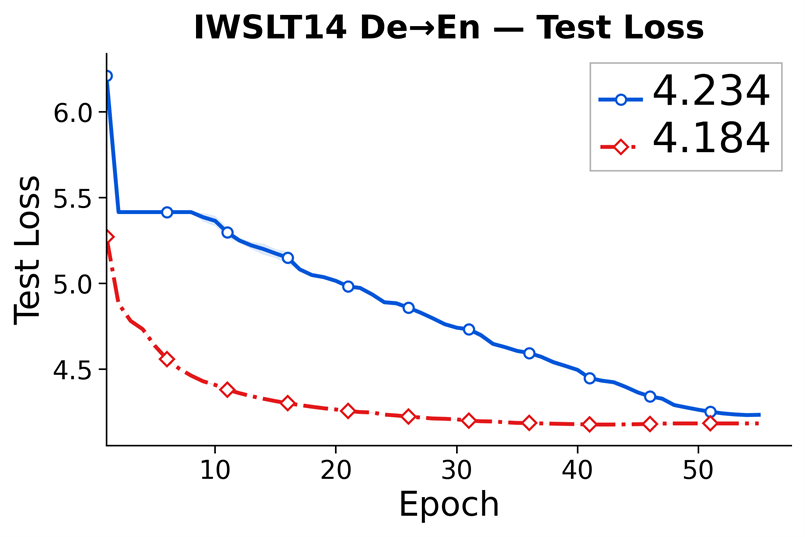}\hfill
    \includegraphics[width=0.325\linewidth]{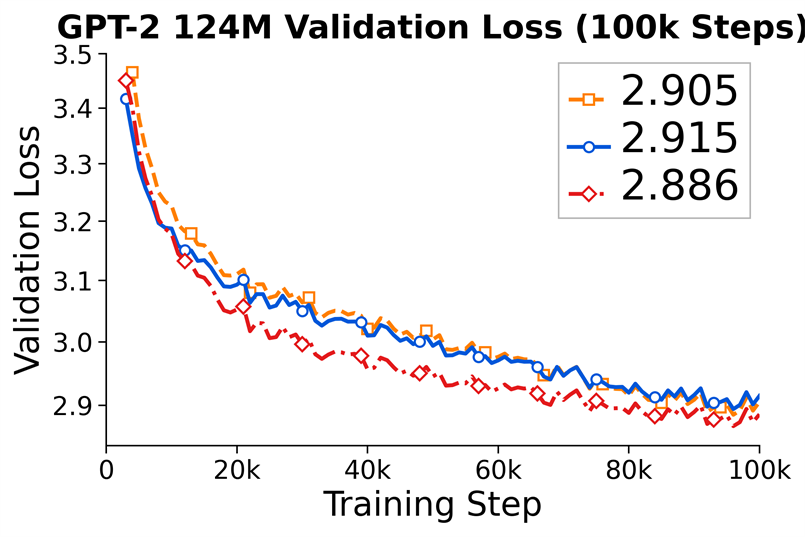}\hfill
    \includegraphics[width=0.325\linewidth]{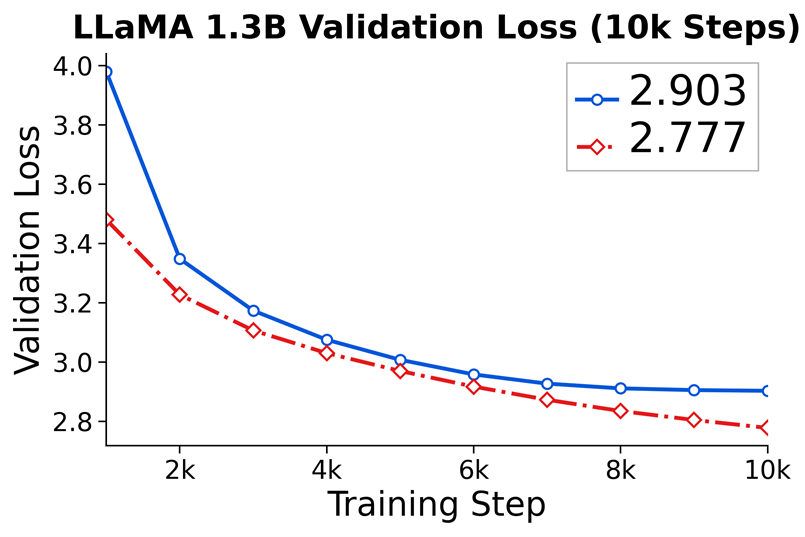}\\[4pt]
    \includegraphics[width=0.325\linewidth]{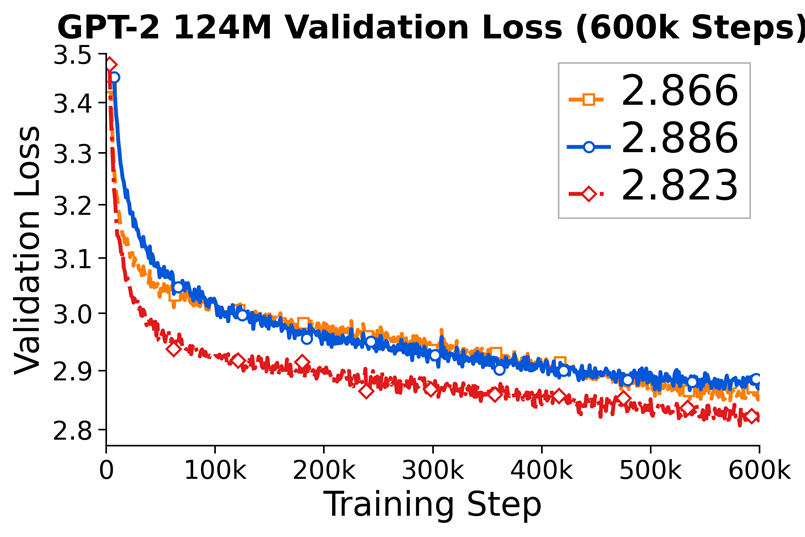}\hfill
    \includegraphics[width=0.325\linewidth]{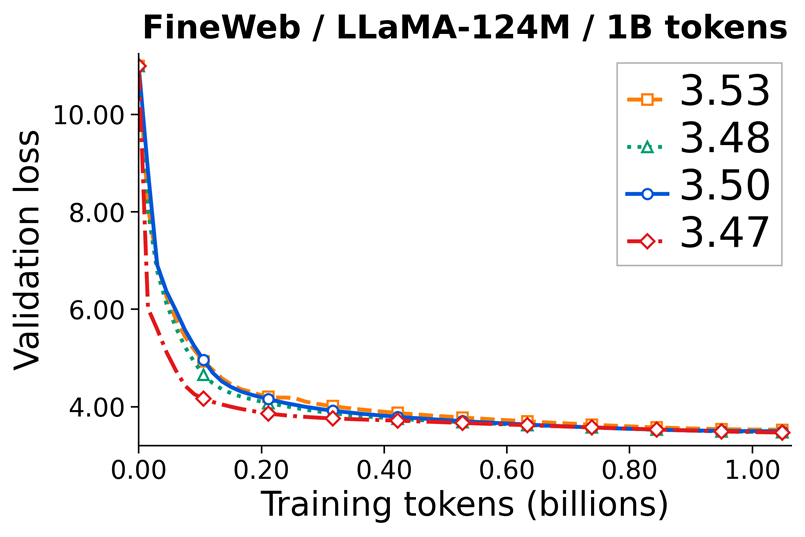}\hfill
    \includegraphics[width=0.325\linewidth]{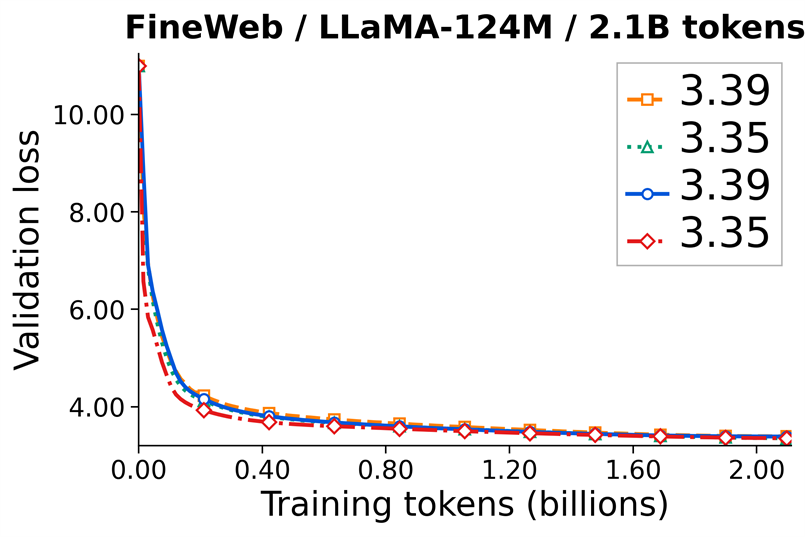}\\[-2pt]
    \includegraphics[width=0.975\linewidth]{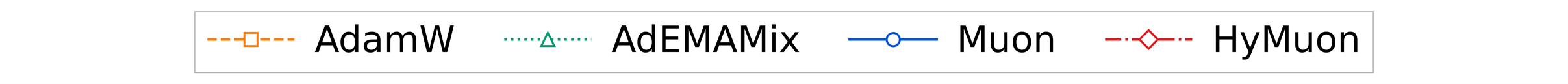}
    \caption{Language tasks. Top row, left to right: IWSLT14 German--English test
    loss for the LSTM model; GPT-2 124M validation loss on OpenWebText through
    100k steps; and LLaMA-1.3B validation loss on C4 through 10k steps, compared
    with the reproduced Muon baseline. Bottom row, left to right: GPT-2 124M
    validation loss on OpenWebText through 600k steps, followed by LLaMA-124M
    validation loss on FineWeb with token budgets of 1B and 2.1B. The IWSLT
    curves show means across repeated runs, and the shaded regions show one
    standard error. Appendix
    Figures~\ref{fig:iwslt-bleu} and~\ref{fig:c4-perplexity} give BLEU and C4
    perplexity.}
    \label{fig:language-tasks}
    \label{fig:iwslt-repeated}
    \label{fig:llm}
    \label{fig:mano-llama-1p3b}
    \label{fig:fineweb-124m}
\end{figure}

\subsection{Image classification}
\label{sec:image-classification}

We compare \hym{} with momentum SGD using cosine or step schedules for the learning rate
and with Schedule-Free SGD on CIFAR-10 (WRN-16-8), CIFAR-100 (DenseNet-BC), SVHN
(ResNet-3-96), and ImageNet-1K (ResNet-50)
\citep{krizhevsky2009learning,netzer2011reading,russakovsky2015imagenet}.
\hym{} is evaluated at its averaged iterate, as described in Appendix~\ref{app:full-exp}.

\hym{} reaches mean test accuracy
of $96.25\%$ on CIFAR-10
and $98.47\%$ on SVHN, compared with $96.11\%$ for Schedule-Free SGD on
CIFAR-10 and $98.24\%$ for cosine SGD on SVHN
(Table~\ref{tab:main-results}). On CIFAR-100, the mean test accuracies are
$79.14\%$ for \hym{} and $79.11\%$ for Schedule-Free SGD. ImageNet mean
validation top-1 accuracies are $76.84\%$ for \hym{}, $76.70\%$ for
Schedule-Free SGD, and $76.68\%$ for cosine SGD. Schedule-Free SGD has higher
validation accuracy during early training on ImageNet, while its curve and that
of \hym{} become close later in training
(Figure~\ref{fig:image-classification}, right panels).

\subsection{Language tasks}

\paragraph{IWSLT14 machine translation.}
On IWSLT14 German--English \citep{cettolo2014iwslt}, we train Fairseq's
LSTM preset with attention \citep{ott2019fairseq} for 55 epochs. We tune each
optimizer separately and evaluate the selected configurations over repeated runs.
At each epoch, the reported test metrics come from the checkpoint with the highest
validation BLEU observed up to that epoch. Averaged over runs, \hymuon{} reaches test loss
$4.184\pm0.004$ and BLEU $33.874\pm0.072$, against $4.234\pm0.006$ and
$33.424\pm0.048$ for Muon, where the values are means with sample standard
deviations (Table~\ref{tab:main-results};
Figure~\ref{fig:language-tasks}, top left;
Figure~\ref{fig:iwslt-bleu} in Appendix~\ref{app:iwslt-runs}).

\paragraph{GPT-2 124M on OpenWebText.}
We pretrain GPT-2 124M \citep{radford2019language} from scratch on
OpenWebText \citep{gokaslan2019openwebtext} in two separate experiments, one
for 100k steps and the other for 600k steps. \hymuon{} reaches final validation losses of $2.886$
and $2.823$, respectively, lower than AdamW and Muon at both training lengths
(Figure~\ref{fig:language-tasks}, top middle and bottom left;
full results in Table~\ref{tab:main-results}).

\paragraph{LLaMA-124M on FineWeb: comparison with AdEMAMix.}
We compare AdamW, AdEMAMix, Muon, and \hymuon{} using the same FineWeb subset
and the 124M protocol for large batches from \citet{semenov2025benchmarking},
with separate runs using 1B and 2.1B tokens.
At 1B tokens, \hymuon{} reaches final validation loss $3.47$, compared with
$3.48$ for AdEMAMix. At 2.1B tokens, \hymuon{} and AdEMAMix both have validation
loss $3.35$ after rounding to two decimal places, while Muon and AdamW both have
$3.39$ (Figure~\ref{fig:language-tasks}, bottom middle and bottom right).
Protocol details are in Appendix~\ref{app:fineweb-124m}.

\paragraph{LLaMA-1.3B on C4.}
We compare \hymuon{} against Muon reproduced with the hyperparameters selected by
\citet{gu2026mano} for LLaMA-1.3B/C4. We tune the \hymuon{} optimizer
hyperparameters; Appendix~\ref{app:c4-runs} gives the selected configuration.
Training uses C4/en with document packing for 10k steps.
\hymuon{} reaches validation loss $2.777$, compared with $2.903$ for Muon
(Figure~\ref{fig:language-tasks}, top right). The corresponding perplexity
curves are reported in Figure~\ref{fig:c4-perplexity} in Appendix~\ref{app:c4-runs}.
Appendix~\ref{app:track3} reports an additional evaluation of \hymuon{} on the
modded-nanoGPT Track~3 benchmark.

\FloatBarrier
\section{Conclusion}

EMA momentum reduces a gradient's weight geometrically according to the number
of iterations that have passed since the gradient was computed. The controlled
experiments show that a gradient contribution retained in the optimizer state
can still reduce a target loss thousands of steps later and that gradients
computed from different training examples can make substantially different
contributions to the same loss. \hym{} separates long-term retention from
short-term adaptation by using a slow state that draws on a growing gradient
history and a fast state that emphasizes recent gradients. For smooth convex
objectives, \hym{} achieves an $\Ocal(T^{-1/2})$ bound at its final averaged
iterate, while \hym{} and \hymuon{} match or improve upon the baselines
considered across image classification, machine translation, and language
modeling. The broader lesson is that the age of a gradient need not determine
its usefulness. Momentum need not serve only as a short-term smoother; it can
also provide long-term memory for useful gradient information that appears
intermittently during training.

%% file: appendix.tex
\appendix

\numberwithin{theorem}{section}
\numberwithin{lemma}{section}
\numberwithin{proposition}{section}
\numberwithin{corollary}{section}
\numberwithin{definition}{section}
\numberwithin{remark}{section}

{\LARGE\bfseries Appendix\par}
\vspace{1.2em}

\begingroup
\small
\setlength{\parindent}{0pt}
\newcommand{\approadmapsection}[2]{%
  \textbf{\hyperref[#1]{\ref*{#1}\quad #2}}\dotfill%
  & \hyperref[#1]{\pageref*{#1}}\\[2pt]}
\newcommand{\approadmapsubsection}[2]{%
  \hspace{1.6em}\hyperref[#1]{\ref*{#1}\quad #2}\dotfill%
  & \hyperref[#1]{\pageref*{#1}}\\}

\noindent\rule{\textwidth}{0.8pt}
\vspace{0.35em}

\begin{tabular*}{\textwidth}{@{}p{0.91\textwidth}@{\extracolsep{\fill}}r@{}}
\approadmapsection{app:sf-reformulation}{Reduction of Schedule-Free to a representation using one sequence}
\approadmapsection{app:hym-mechanisms}{Gradient weighting and averaging along a curved valley}
\approadmapsubsection{app:hym-history-proof}{Weights on early and recent gradients}
\approadmapsubsection{app:hym-filter-variance}{Asymptotic covariance under a fixed gradient mean}
\approadmapsubsection{app:hym-filter-proof}{Weights after a persistent change in gradient mean}
\approadmapsubsection{app:hym-average-geometry}{Geometry of averaging along a curved valley}
\approadmapsection{app:hym-convex-x}{Convergence analysis via cumulative gradient weights}
\approadmapsubsection{app:kernel-geometry}{The averaged iterate and the Bregman term}
\approadmapsubsection{app:kernel-refined}{A theorem for cumulative gradient weights over a finite horizon}
\approadmapsubsection{app:kernel-hym-certificate}{Bounding cumulative gradient weight deviations for \hym{}}
\approadmapsubsection{app:kernel-hym-returned}{Convergence of the averaged iterate for \hym{}}
\approadmapsubsection{app:kernel-related}{Finite mixtures of EMAs with fixed decay rates}
\approadmapsection{app:hymuon-analysis}{\hymuon{} matrix updates and implementation}
\approadmapsection{app:mechanistic-experiments}{Diagnostic experiment details}
\approadmapsubsection{app:hym-batch-size}{Dependence of the \hym{} advantage on batch size}
\approadmapsubsection{app:narrow-signal}{Semantic relevance and improvement specific to each target}
\approadmapsubsection{app:forgetting}{Experiment on the lasting effect of one batch}
\approadmapsection{app:full-exp}{Benchmark experiment details}
\approadmapsubsection{app:vision-runs}{Image classification runs}
\approadmapsubsection{app:cifar-readout-comparison}{\hym{} at the training and averaged iterates}
\approadmapsubsection{app:cifar10-component-ablation}{Component ablation of \hym{} on CIFAR-10}
\approadmapsubsection{app:iwslt-runs}{IWSLT14 German--English}
\approadmapsubsection{app:openwebtext-runs}{NanoGPT on OpenWebText}
\approadmapsubsection{app:c4-runs}{LLaMA-1.3B pretraining on C4}
\approadmapsubsection{app:fineweb-124m}{LLaMA-124M FineWeb benchmark}
\approadmapsection{app:track3}{Optimization efficiency on modded-nanoGPT Track 3}
\end{tabular*}

\vspace{0.25em}
\noindent\rule{\textwidth}{0.8pt}
\medskip
\endgroup

\clearpage

\section{Reduction of Schedule-Free to a representation using one sequence}
\label{app:sf-reformulation}

The Schedule-Free method of \citet{defazio2024road} is originally stated with three
sequences: the averaged iterate $\xbf_t$, the base iterate $\zbf_t$, and the
training iterate $\ybf_t$,
\begin{align}
\ybf_t &= (1-\beta)\zbf_t + \beta\xbf_t, \label{eq:sf-original-y} \\
\zbf_{t+1} &= \zbf_t-\eta\,\gbf_t,\quad
\gbf_t=\frac1B\sum_{j=1}^B\nabla f(\ybf_t;\xi_{t,j}), \label{eq:sf-original-z} \\
\xbf_{t+1} &= (1-c_{t+1})\xbf_t + c_{t+1}\zbf_{t+1}, \label{eq:sf-original-x}
\end{align}
initialized at $\zbf_1=\xbf_1$, with averaging weights $c_{t+1}$ (for a constant step size,
$c_{t+1}=1/(t+1)$). The representation in \eqref{eq:sf-single}, which uses one sequence,
follows from the change of
variable
\begin{equation}
\label{eq:sf-mdef}
\nbf_t := \frac{\xbf_t-\zbf_{t+1}}{\eta},
\end{equation}
which is the rescaled difference between the averaged and base iterates and
serves as the state variable in this representation.

Using $\xbf_t=(1-c_t)\xbf_{t-1}+c_t\zbf_t$ and $\zbf_{t+1}=\zbf_t-\eta\,\gbf_t$, we have
\begin{equation}
\nbf_t
= \frac{\xbf_t-\zbf_{t+1}}{\eta}
= \frac{(1-c_t)(\xbf_{t-1}-\zbf_t)+\eta\,\gbf_t}{\eta}
= (1-c_t)\nbf_{t-1}+\gbf_t .
\end{equation}
Since $\zbf_1=\xbf_1$, \eqref{eq:sf-mdef} gives
$\nbf_1=(\xbf_1-\zbf_2)/\eta=\gbf_1$, consistent with $c_1=1$.

Differencing \eqref{eq:sf-original-y} and using
$\zbf_{t+1}-\zbf_t=-\eta\,\gbf_t$,
$\xbf_{t+1}-\xbf_t=c_{t+1}(\zbf_{t+1}-\xbf_t)$, and
$\zbf_{t+1}-\xbf_t=-\eta\,\nbf_t$ from \eqref{eq:sf-mdef}, we obtain
\begin{equation}
\begin{aligned}
\ybf_{t+1}-\ybf_t
&= (1-\beta)(\zbf_{t+1}-\zbf_t)+\beta(\xbf_{t+1}-\xbf_t)\\
&= -(1-\beta)\eta\,\gbf_t+\beta\,c_{t+1}(\zbf_{t+1}-\xbf_t)\\
&= -\eta\big[(1-\beta)\gbf_t+\beta\,c_{t+1}\nbf_t\big],
\end{aligned}
\end{equation}
which is exactly \eqref{eq:sf-single}.

For $c_t=1/t$, the recurrence has the closed form $\nbf_t=\frac{t-1}{t}\nbf_{t-1}+\gbf_t=\sum_{i=1}^t\frac{i}{t}\gbf_i$, whose weights grow linearly with $i$. This is the same linear weighting pattern used in Nesterov's construction based on estimate sequences \citep{nesterov1983method,nesterov2005smooth}.
\hym{} uses the same recurrence after the rescaling
$\nbf_t=(t+1)\mbf_t/2$, but scales its contribution by the geometric sum
$(1-\mu^t)/(1-\mu)$ and replaces the term that directly uses the current gradient
with the fast state. Its direction is
$\beta(1-\mu^t)\nbf_t/[t(1-\mu)]+(1-\beta)\bbf_t$, where $\mbf_t$ denotes the
slow state in Algorithm~\ref{alg:hym}.

\section{Gradient weighting and averaging along a curved valley}
\label{app:hym-mechanisms}

\subsection{Weights on early and recent gradients}
\label{app:hym-history-proof}

We give the calculations underlying Section~\ref{sec:hym-weights}, for
$c_t=1/t$, fixed $\mu\in(0,1)$, and $\beta\in[0,1)$. Throughout we write
$q_t=(1-\mu^t)/(1-\mu)$ for the recurring geometric sum.

\paragraph{Decomposition and normalization of gradient weights.}
Unrolling the two momentum states gives
\begin{equation}
 \mathbf m_t=\sum_{i=1}^t\frac{2i}{t(t+1)}\mathbf g_i,
 \qquad
 \mathbf b_t=\sum_{i=1}^t\frac{i}{t}\mu^{t-i}\mathbf g_i.
 \label{eq:hym-accumulator-kernels}
\end{equation}
Substitution into Algorithm~\ref{alg:hym}, line~\ref{algline:hym-direction},
yields the gradient weights
\eqref{eq:hym-kernel}. Define the coefficient sum $Z_t$ for the fast state and
the two component coefficient sums by
\begin{equation}
 Z_t=\sum_{i=1}^t\frac{i}{t}\mu^{t-i},\qquad
 A_t=\beta r_t,\qquad C_t=(1-\beta)Z_t.
 \label{eq:hym-component-masses}
\end{equation}
The total coefficient sum is $S_t=\sum_{i=1}^tw_{t,i}=A_t+C_t$, with
\begin{equation}
 S_1=1,\qquad \lim_{t\to\infty}S_t=
 \frac{1-\beta/2}{1-\mu}.
 \label{eq:hym-total-gain}
\end{equation}
To separate the distribution over past gradients from this total coefficient sum,
we normalize the coefficients as
\begin{equation}
 p_{t,i}=\frac{w_{t,i}}{S_t}
 =\frac{w_{t,i}}{\sum_{j=1}^t w_{t,j}},
 \qquad \sum_{i=1}^t p_{t,i}=1.
 \label{eq:hym-sum-normalization}
\end{equation}
Thus $p_{t,i}$ is the fraction of the direction's total scalar gradient
weight assigned to $\mathbf g_i$. It is an exact mixture of two distributions:
\begin{equation}
 p_{t,i}=\lambda_t\underbrace{\frac{2i}{t(t+1)}}_{U_{t,i}}
 +(1-\lambda_t)\underbrace{\frac{(i/t)\mu^{t-i}}{Z_t}}_{V_{t,i}},
 \qquad \lambda_t=\frac{A_t}{A_t+C_t}.
 \label{eq:hym-kernel-mixture}
\end{equation}
Both $U_{t,i}$ and $V_{t,i}$ sum to one. The proposition below gives the aggregate
weight on an early fraction of history and on a fixed recent window.

\begin{proposition}
\label{prop:hym-history}
Fix $\mu\in(0,1)$ and $\beta\in(0,1)$. Then
\[
 \lim_{t\to\infty}\lambda_t=\frac{\beta}{2-\beta},
\qquad
 \lim_{t\to\infty}(1-\lambda_t)=
 \rho:=\frac{2(1-\beta)}{2-\beta}.
\]
For every fixed $a\in(0,1)$ and positive integer $K$,
\begin{align}
 \lim_{t\to\infty}\sum_{i=1}^{\lfloor at\rfloor}p_{t,i}
 &=\frac{\beta}{2-\beta}a^2,
 \label{eq:hym-old-history-mass}\\
 \lim_{t\to\infty}\sum_{i=t-K+1}^t p_{t,i}
 &=\rho(1-\mu^K).
 \label{eq:hym-recent-history-mass}
\end{align}
\end{proposition}

\begin{proof}
Recall that
\[
 r_t=\frac{t+1}{2t}q_t,
 \qquad
 Z_t=\sum_{i=1}^t\frac{i}{t}\mu^{t-i}.
\]
 After setting $k=t-i$, we have
 \[
  Z_t=\sum_{k=0}^{t-1}\left(1-\frac{k}{t}\right)\mu^k,
  \qquad
  0\le q_t-Z_t
  =\frac1t\sum_{k=0}^{t-1}k\mu^k
  \le\frac{\mu}{t(1-\mu)^2}.
 \]
 Since $\lim_{t\to\infty}q_t=(1-\mu)^{-1}$, it follows that
 \[
  \lim_{t\to\infty}Z_t=\frac{1}{1-\mu}.
 \]
 Consequently,
 \[
  \lim_{t\to\infty}A_t=\frac{\beta}{2(1-\mu)},\qquad
  \lim_{t\to\infty}C_t=\frac{1-\beta}{1-\mu}.
 \]
 Substituting these limits into $\lambda_t=A_t/(A_t+C_t)$ proves the
 two limits for $\lambda_t$ stated in the proposition.

Let $k_t=\lfloor at\rfloor$. The slow component satisfies
\[
 \lim_{t\to\infty}\sum_{i=1}^{k_t}U_{t,i}
 =\lim_{t\to\infty}\frac{k_t(k_t+1)}{t(t+1)}=a^2.
\]
For the fast component, $Z_t\ge1$ and
\[
 0\le\sum_{i=1}^{k_t}V_{t,i}
 \le\sum_{k=t-k_t}^{\infty}\mu^k
 =\frac{\mu^{t-k_t}}{1-\mu}.
 \]
 Because $t-k_t$ tends to infinity,
 $\lim_{t\to\infty}\sum_{i=1}^{k_t}V_{t,i}=0$.
Combining these limits with the mixture representation
\eqref{eq:hym-kernel-mixture} proves
\eqref{eq:hym-old-history-mass}.

For a fixed positive integer $K$,
 \[
 \lim_{t\to\infty}\sum_{i=t-K+1}^tU_{t,i}=0,
 \qquad
 \sum_{i=t-K+1}^tV_{t,i}
 =\frac{\sum_{k=0}^{K-1}(1-k/t)\mu^k}{Z_t}.
 \]
 Hence
 $\lim_{t\to\infty}\sum_{i=t-K+1}^tV_{t,i}=1-\mu^K$.
The same mixture representation therefore gives
\eqref{eq:hym-recent-history-mass}.
\end{proof}

\paragraph{Peak near recent gradients for large $t$.}
When $t(1-\mu)\ll1$, the two component profiles are similar, and their combined
weight increases approximately linearly with $i$. For larger $t$, the weights of the
fast state become concentrated within $\Ocal((1-\mu)^{-1})$ iterations of the current
step, whereas the weights of the slow state extend over $\Ocal(t)$ iterations.
For $0<\beta<1$, the ratio of the fast and slow coefficients at lag $k=t-i$ is
$\frac{(1-\beta)t\mu^k}{\beta q_t}$.
Using $q_t\to(1-\mu)^{-1}$, the two coefficients on the newest gradient are equal at
\begin{equation}
 t_{\mathrm{cross}}:=
 \frac{\beta}{(1-\beta)(1-\mu)}.
 \label{eq:hym-head-scale}
\end{equation}
When $t$ is large relative to $t_{\mathrm{cross}}$, the fast term dominates the
coefficient on the newest gradients and produces the peak over recent gradients, while
the slow component retains the limiting share $\beta/(2-\beta)$ of the total weight.

A finite mixture of EMAs with fixed decay coefficients below one behaves differently:
for fixed $a\in(0,1)$, the share it places on $i\le at$ vanishes as $t$ grows, since
each EMA tail is bounded by a geometrically decreasing term. In contrast, \hym{}
retains the positive aggregate weight given in \eqref{eq:hym-old-history-mass}.

\subsection{Asymptotic covariance under a fixed gradient mean}
\label{app:hym-filter-variance}

\begin{proposition}
\label{prop:hym-variance}
Fix $\mu\in(0,1)$ and $\beta\in[0,1)$.
Suppose $\gbf_i=\gbf+\boldsymbol\varepsilon_i$, where the
noises are independent, zero mean, and have covariance $\Sigma/B$.
Let $\widetilde{\bDelta}_t=\sum_{i=1}^t p_{t,i}\gbf_i$,
$\rho=2(1-\beta)/(2-\beta)$, and $\chi=\rho(1-\mu)$.
Then $\widetilde{\bDelta}_t$ is unbiased,
$\lim_{t\to\infty}p_{t,t}=\chi$, and
\begin{equation}
 \lim_{t\to\infty}\operatorname{Cov}(\widetilde{\bDelta}_t)
 =
 \rho^2\frac{1-\mu}{1+\mu}\frac{\Sigma}{B}.
 \label{eq:hym-asymptotic-covariance}
\end{equation}
For $0<\beta<1$, its limiting scalar covariance factor is strictly
below that of a normalized EMA with weight $\chi$ on the newest gradient:
\[
 \rho^2\frac{1-\mu}{1+\mu}<\frac{\chi}{2-\chi}.
\]
\end{proposition}

\begin{proof}
Define
\[
Z_t:=\sum_{i=1}^t\frac{i}{t}\mu^{t-i},\qquad
S_t:=\sum_{i=1}^tw_{t,i}
=\frac{\beta q_t(t+1)}{2t}+(1-\beta)Z_t.
\]
The normalized weights can be decomposed as
\[
\frac{w_{t,i}}{S_t}=\lambda_t U_{t,i}+(1-\lambda_t)V_{t,i},
\]
where
\[
U_{t,i}=\frac{2i}{t(t+1)},\qquad
V_{t,i}=\frac{(i/t)\mu^{t-i}}{Z_t},\qquad
\lambda_t=\frac{\beta q_t(t+1)/(2t)}{S_t},\qquad
1-\lambda_t=\frac{(1-\beta) Z_t}{S_t}.
\]
Both \(U_t\) and \(V_t\) are probability vectors. The calculations in
the proof of Proposition~\ref{prop:hym-history} give
\[
\lim_{t\to\infty}q_t=\lim_{t\to\infty}Z_t=\frac1{1-\mu},
\qquad
\lim_{t\to\infty}(1-\lambda_t)
=\rho=\frac{2(1-\beta)}{2-\beta},
\qquad
\lim_{t\to\infty}\lambda_t=1-\rho.
\]
For the squared weights,
\[
\lim_{t\to\infty}\sum_{i=1}^t U_{t,i}^2
=\lim_{t\to\infty}\frac{2(2t+1)}{3t(t+1)}=0.
\]
Moreover,
\[
\sum_{i=1}^t V_{t,i}^2
=\frac{1}{Z_t^2}\sum_{k=0}^{t-1}\left(1-\frac{k}{t}\right)^2\mu^{2k}.
\]
Dominated convergence and the limit of $Z_t$ therefore yield
\[
\lim_{t\to\infty}\sum_{i=1}^t V_{t,i}^2
=\frac{(1-\mu)^2}{1-\mu^2}
=\frac{1-\mu}{1+\mu}.
\]
By Cauchy--Schwarz,
\[
\left|\sum_{i=1}^t U_{t,i}V_{t,i}\right|
\le
\left(\sum_{i=1}^tU_{t,i}^2\right)^{1/2}
\left(\sum_{i=1}^tV_{t,i}^2\right)^{1/2},
\]
which implies that the cross term has limit zero. Independence of the noise now gives
\[
\operatorname{Cov}(\widetilde{\bDelta}_t)
=\left(\sum_{i=1}^t\frac{w_{t,i}^2}{S_t^2}\right)\frac{\Sigma}{B}
\]
and hence
\[
\lim_{t\to\infty}\operatorname{Cov}(\widetilde{\bDelta}_t)
=\rho^2\frac{1-\mu}{1+\mu}\frac{\Sigma}{B}.
\]
The newest normalized weight is
\[
\lim_{t\to\infty}\frac{\beta q_t/t+1-\beta}{S_t}
=\rho(1-\mu)=\chi.
\]
A normalized EMA with newest weight \(\chi\) has weights
\(\chi(1-\chi)^k\).  Its scalar covariance factor is
\[
 \chi^2\sum_{k=0}^{\infty}(1-\chi)^{2k}=\frac{\chi}{2-\chi}.
\]
The difference between the factors for the corresponding EMA and \hym{} is
\[
\frac{\rho(1-\mu)}{2-\rho(1-\mu)}
-\frac{\rho^2(1-\mu)}{1+\mu}
=\frac{\rho(1-\mu)(1-\rho)
 [2-(1-\mu)(1+\rho)]}
 {[2-\rho(1-\mu)](1+\mu)}>0
\]
for \(0<\rho<1\) and \(0<\mu<1\), proving the proposition.
\end{proof}

\subsection{Weights after a persistent change in gradient mean}
\label{app:hym-filter-proof}

The comparison in Proposition~\ref{prop:hym-variance} matches the coefficient on
the newest gradient and therefore describes the weight assigned immediately after
a change in mean. To examine a
persistent change, suppose the deterministic gradient mean is
$\mathbf g^-$ through step $\tau$ and $\mathbf g^+=\mathbf g^-+\mathbf d$
afterward. At $t=\tau+k$, the fraction of the normalized \hym{} direction
assigned to gradients with mean $\mathbf g^+$ is
\begin{equation}
 H_{t,\tau}
 =\sum_{i=\tau+1}^t p_{t,i}
 =\lambda_t\!\left(1-\frac{\tau(\tau+1)}{t(t+1)}\right)
 +(1-\lambda_t)
 \frac{\sum_{j=0}^{k-1}(1-j/t)\mu^j}{Z_t}.
 \label{eq:hym-persistent-shift-response}
\end{equation}
For fixed $k$, the limiting weight assigned after the change is
\[
 \lim_{\tau\to\infty}H_{\tau+k,\tau}=\rho(1-\mu^k).
\]
A normalized EMA with constant newest weight
$\chi=\rho(1-\mu)$, after its initialization bias has vanished, assigns
the fraction $1-(1-\chi)^k$ to gradients after the change. Its fraction at
$k=1$ matches \hym{}'s, while for $k>1$, $0<\rho<1$, and
$0<\mu<1$,
\[
 1-\bigl(1-\rho(1-\mu)\bigr)^k
 >\rho(1-\mu^k).
\]
The inequality follows from strict concavity of $h(a)=1-(1-a)^k$ on
$(0,1)$. The two methods match the weight on the newest gradient but assign
different total weights after a persistent change: over a fixed recent window,
the EMA assigns a larger fraction to gradients after the change, whereas
\hym{} assigns the remaining share to its slow component. Under this comparison,
\hym{} has lower limiting covariance, while the EMA assigns weight to gradients
after the change more quickly.

\subsection{Geometry of averaging along a curved valley}
\label{app:hym-average-geometry}

\subsubsection{Rosenbrock illustration and the role of the averaged iterate}
\label{sec:hym-valley}

The preceding subsections analyze how the two momentum states distribute weight
over the gradient history and respond to changes in the gradient mean. We now turn
to a separate component of \hym{}: the averaged iterate. The following example
illustrates how averaging the training iterates interacts with motion along a curved
valley.

Consider gradients whose dominant component changes sign across iterations while a
smaller component remains consistently aligned. Averaging attenuates the component
whose sign varies relative to the consistently aligned component. For the Rosenbrock function
$f_{\rm RB}$ defined below, the component of $\nabla f_{\rm RB}$ normal to the curve
$v=u^2$ changes sign across iterates, while the tangential component can retain its sign.
This geometric motivation also appears in AdEMAMix and Schedule-Free
analyses~\citep{pagliardini2024ademamix,song2025river}.

The slow state averages gradients over the trajectory, allowing normal components that
change sign to cancel while tangential components that retain their sign accumulate. The
fast state places greater weight on gradients computed near the current part of the valley.
The two states therefore capture complementary parts of the gradient history.

We visualize this behavior on
$f_{\rm RB}(u,v)=(1-u)^2+100(v-u^2)^2$. Figure~\ref{fig:trajectory} shows the
training iterates; Figure~\ref{fig:trajectory-x} shows the corresponding
averaged iterates.

\begin{figure}[htbp]
 \centering
 \includegraphics[width=0.99\linewidth]{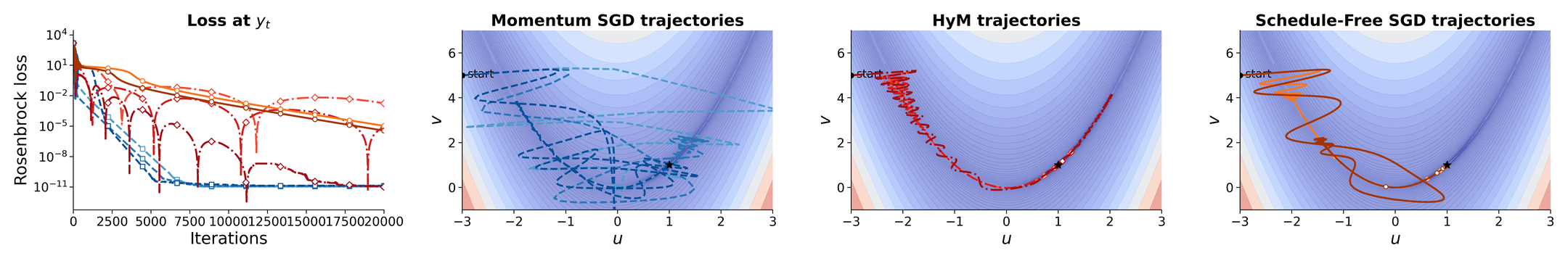}\\[-2pt]
 \includegraphics[width=0.99\linewidth]{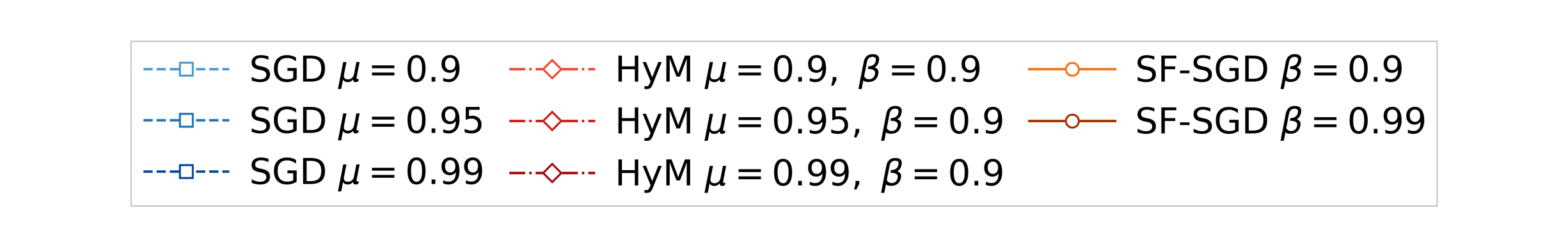}
 \caption{Behavior of the Rosenbrock training iterates from $(-3,5)$ over 20,000
 steps using exact gradients. Momentum SGD uses $\eta=4\times10^{-3}$ and
 $\mu\in\{0.9,0.95,0.99\}$; \hym{} uses $\eta=3\times10^{-4}$,
 $\beta=0.9$, and $\mu\in\{0.9,0.95,0.99\}$; Schedule-Free SGD uses
 $\eta=7\times10^{-4}$ and $\beta\in\{0.9,0.99\}$. The left panel shows
 objective values at $\mathbf y_t$; the remaining panels show the trajectories
 over objective contours. The star marks the minimizer $(1,1)$.}
\label{fig:trajectory}
\end{figure}

Gradient memory determines the training trajectory, whereas iterate averaging forms a
separate sequence from the training iterates. In \hymuon{}, the two momentum states are
combined before the nonlinear NS transformation. Consequently, averaging gradients is
distinct from averaging the transformed update directions.

\paragraph{Effect of iterate averaging in a curved valley.}

Let $v=\phi(u)$ denote a valley centerline and let
$\{(u_i,v_i)\}_{i=1}^N$ be $N$ training iterates with
$\pi_i\ge0$ and $\sum_{i=1}^N\pi_i=1$. Write
\[
 \bar u=\sum_{i=1}^N\pi_i u_i,
 \qquad
 \bar v=\sum_{i=1}^N\pi_i v_i.
\]
The averaged iterate has the exact centerline residual
\begin{equation}
 \bar v-\phi(\bar u)
 =\underbrace{\sum_{i=1}^N\pi_i\bigl(v_i-\phi(u_i)\bigr)}_{\text{averaged tracking residual}}
 +\underbrace{\left(\sum_{i=1}^N\pi_i\phi(u_i)
                  -\phi\!\left(\sum_{i=1}^N\pi_i u_i\right)\right)}
             _{\text{averaging gap due to curvature}}.
 \label{eq:curved-average-decomposition}
\end{equation}
Thus even points lying exactly on the centerline can average to a point away
from it.  If $\phi$ is twice differentiable and
$|\phi''(u)|\le M$ on the convex hull of the $u_i$, Taylor's theorem gives
\begin{equation}
 \left|\sum_{i=1}^N\pi_i\phi(u_i)-\phi(\bar u)\right|
 \le\frac{M}{2}\sum_{i=1}^N\pi_i(u_i-\bar u)^2.
 \label{eq:curved-average-gap-bound}
\end{equation}
For the Rosenbrock valley centerline $\phi(u)=u^2$, the bound holds with $M=2$.

This decomposition separates two effects of iterate averaging. Averaging can reduce
tracking residuals that change sign across iterates, while curvature can move the
averaged point away from the valley centerline.

\begin{figure}[htbp]
 \centering
 \includegraphics[width=0.99\linewidth]{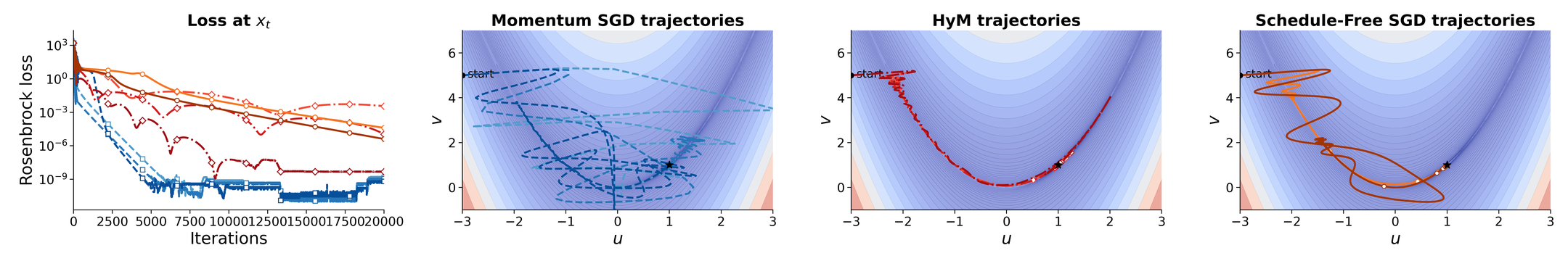}\\[-2pt]
 \includegraphics[width=0.99\linewidth]{figures/rosenbrock_shared_legend.png}
 \caption{Averaged iterates for the Rosenbrock runs in
 Figure~\ref{fig:trajectory}, using the same initialization, optimizer
 settings, and 20,000 steps using exact gradients. \hym{} uses $\beta=0.9$ for
 iterate averaging; Momentum SGD is evaluated at an averaged iterate formed with
 $\beta=0.99$, and this average does not enter its updates; Schedule-Free SGD uses its respective
 $\beta\in\{0.9,0.99\}$. The left panel shows objective values at
 $\mathbf x_t$; the remaining panels show the averaged trajectories over
 objective contours. The star marks the minimizer $(1,1)$.}
 \label{fig:trajectory-x}
\end{figure}

The CIFAR-10 component ablation evaluates these roles separately by measuring each
training trajectory at both the training and averaged iterates
(Appendix~\ref{app:cifar10-component-ablation}).
\FloatBarrier

\section{Convergence analysis via cumulative gradient weights}
\label{app:hym-convex-x}

Sections~\ref{app:kernel-geometry}--\ref{app:kernel-hym-returned} develop the convex
analysis of the averaged iterate through the Bregman identity, a general theorem over
a finite horizon, a bound on \hym{}'s cumulative gradient weight deviations, and the
resulting convergence rate. Section~\ref{app:kernel-related} shows that the same uniform bound
holds for finite EMA mixtures.

We prove the general result of Theorem~\ref{thm:kernel-square-stability}
and specialize it to \hym{}. Throughout this appendix, $\eta>0$ is constant within
a run, $\Gamma>0$, and $\varsigma=\sigma/\sqrt B\ge0$.
The initial point is deterministic. All iterates and interpolation
coefficients are $\mathcal F_{t-1}$-measurable before $\gbf_t$ is sampled at
iteration $t$; a sequence with this property is predictable.
Because the samples in $\mathcal B_t$ are conditionally independent given
$\mathcal F_{t-1}$, Assumption~\ref{ass:minibatch-oracle} implies that the
mini-batch gradient at the training iterate $\ybf_t$ satisfies
\begin{equation}
 \mathbb E[\gbf_t\mid\mathcal F_{t-1}]=\nabla f(\ybf_t),\qquad
 \mathbb E[\|\gbf_t-\nabla f(\ybf_t)\|^2\mid\mathcal F_{t-1}]
 \le\varsigma^2.
 \label{eq:hym-convex-oracle}
\end{equation}
We write $\beps_t=\gbf_t-\nabla f(\ybf_t)$ and
$Q_T=\sum_{t=1}^T\mathbb E\|\nabla f(\ybf_t)\|^2$.
For the recursion in \eqref{eq:kernel-framework}, all iterates and gradients
have finite second moments at each finite horizon. This follows by induction
from the deterministic initial point and cumulative gradient weights, the bounded
interpolation coefficients, the Lipschitz gradient, and the oracle variance bound.
The representation of the derived base iterates through cumulative
gradient weights is exact; weight decay, nonlinear preconditioning, and NS
orthogonalization are outside this analysis with deterministic scalar weights.

\subsection{The averaged iterate and the Bregman term}
\label{app:kernel-geometry}

For \hym{}, let $\lambda=1-\beta>0$ and define
\begin{equation}
 \vbf_t=\frac{\ybf_t-\beta\xbf_t}{\lambda}.
 \label{eq:convex-base}
\end{equation}
The averaging recurrence in Algorithm~\ref{alg:hym} gives
\begin{equation}
 \vbf_t=t\xbf_t-(t-1)\xbf_{t-1},\qquad
 \xbf_t=\frac1t\sum_{i=1}^t\vbf_i,\qquad
 \ybf_t=\lambda\vbf_t+\beta\xbf_t.
 \label{eq:convex-template}
\end{equation}
Thus the algorithm's final averaged iterate is the arithmetic mean of its
derived base sequence. The following conversion also permits predictable interpolation
that varies with time.

\begin{lemma}
\label{lem:hym-convex-conversion}
Let $f$ be convex and $L$-smooth on $\mathbb R^d$, $L>0$, with minimizer
$\xbf^*$. Suppose
$\xbf_t=t^{-1}\sum_{i=1}^t\vbf_i$ and
$\ybf_t=(1-\zeta_t)\vbf_t+\zeta_t\xbf_t$ with predictable
$\zeta_t\in[0,1]$ and predictable, square-integrable $\vbf_t$.
Under \eqref{eq:hym-convex-oracle},
\begin{equation}
 T\mathbb E[f(\xbf_T)-f(\xbf^*)]+\frac{Q_T}{2L}
 \le\sum_{t=1}^T\mathbb E\langle\gbf_t,\vbf_t-\xbf^*\rangle.
 \label{eq:convex-conversion}
\end{equation}
\end{lemma}

\begin{proof}
Set $\mathcal D_t(\ubf)=f(\ubf)-f(\ybf_t)
-\langle\nabla f(\ybf_t),\ubf-\ybf_t\rangle$,
$\Psi_t=t(f(\xbf_t)-f(\xbf^*))$, and $\Psi_0=0$.
For $t\ge2$, the identity $\vbf_t=t\xbf_t-(t-1)\xbf_{t-1}$ yields
\begin{equation}
 \langle\nabla f(\ybf_t),\vbf_t-\xbf^*\rangle-(\Psi_t-\Psi_{t-1})
 =\mathcal D_t(\xbf^*)+(t-1)\mathcal D_t(\xbf_{t-1})
 -t\mathcal D_t(\xbf_t).
 \label{eq:kernel-bregman-identity}
\end{equation}
Moreover,
\[
 \xbf_t=(1-\nu_t)\ybf_t+\nu_t\xbf_{t-1},
 \qquad
 \nu_t=\frac{(1-\zeta_t)(t-1)}{1+(1-\zeta_t)(t-1)},
 \qquad t\nu_t\le t-1.
\]
Convexity of $\mathcal D_t$, together with $\mathcal D_t(\ybf_t)=0$, implies
$t\mathcal D_t(\xbf_t)\le(t-1)\mathcal D_t(\xbf_{t-1})$.
At $t=1$, $\vbf_1=\xbf_1=\ybf_1$; hence, the same conclusion holds with
the term involving $\xbf_{t-1}$ omitted because its coefficient is zero.
For an unconstrained minimizer $\nabla f(\xbf^*)=0$; the standard inequality
for convex $L$-smooth functions
\[
 f(\ubf)\ge f(\vbf)+\langle\nabla f(\vbf),\ubf-\vbf\rangle
 +\frac1{2L}\|\nabla f(\ubf)-\nabla f(\vbf)\|^2
\]
therefore gives $\mathcal D_t(\xbf^*)\ge\|\nabla f(\ybf_t)\|^2/(2L)$.
Combining these inequalities with
\eqref{eq:kernel-bregman-identity}, summing over $t=1,\ldots,T$, and
using $\Psi_0=0$ yields
\[
 T\mathbb E[f(\xbf_T)-f(\xbf^*)]+\frac{Q_T}{2L}
 \le \sum_{t=1}^T
 \mathbb E\langle\nabla f(\ybf_t),\vbf_t-\xbf^*\rangle.
\]
Finally, since $\vbf_t$ is $\mathcal F_{t-1}$-measurable, the conditional
unbiasedness in \eqref{eq:hym-convex-oracle} gives
\[
 \mathbb E\langle\gbf_t,\vbf_t-\xbf^*\rangle
 =\mathbb E\langle\nabla f(\ybf_t),\vbf_t-\xbf^*\rangle.
\]
Substitution proves \eqref{eq:convex-conversion}.
\end{proof}

The term $Q_T/(2L)$ retained on the left side of
\eqref{eq:convex-conversion} is the Bregman contribution used below to control
the terms involving deviations of the cumulative gradient weights. This strengthens
the online-to-batch inequality with equal weights used in the Schedule-Free analysis
\citep{defazio2024road}.

\subsection{A theorem for cumulative gradient weights over a finite horizon}
\label{app:kernel-refined}

Suppose the base sequence in \eqref{eq:kernel-framework} has deterministic
weight deviations $e_{t,k}=K_{t,k}-\Gamma$. Define $E_T\in\mathbb R^{T\times T}$ by
$(E_T)_{t,k}=e_{t,k}$ for $1\le k<t\le T$, and zero otherwise. Set
\begin{equation}
 \mathcal E_T=\|E_T\|_{\mathrm F}^2
 =\sum_{t=1}^T\sum_{k=1}^{t-1}e_{t,k}^2,\qquad
 \chi_T=\left[-\lambda_{\min}\!\left(\frac{E_T+E_T^\top}{2}\right)\right]_+.
 \label{eq:kernel-certificates}
\end{equation}
Here $\|\cdot\|_{\mathrm F}$ is the Frobenius norm,
$\lambda_{\min}$ denotes the smallest eigenvalue of a symmetric matrix, and
$[r]_+=\max\{r,0\}$. Acting on a stacked vector sequence, $E_T$ means
$E_T\otimes I_d$, where $\otimes$ is the Kronecker product and $I_d$ is the
$d\times d$ identity matrix. Its norms are those of the scalar time matrix.

\begin{theorem}
\label{thm:kernel-refined}
Under the assumptions of
Lemma~\ref{lem:hym-convex-conversion}, suppose
\eqref{eq:kernel-framework} holds with deterministic $\vbf_1$ and
$0<\eta\le[2L(\Gamma+2\chi_T)]^{-1}$. Then
\begin{equation}
 T\mathbb E[f(\xbf_T)-f(\xbf^*)]+\frac{Q_T}{8L}
 \le \frac{\|\vbf_1-\xbf^*\|^2}{2\eta\Gamma}
 +\frac{\eta\Gamma\varsigma^2T}{2}
 +2L\eta^2\varsigma^2\mathcal E_T.
 \label{eq:kernel-refined-bound}
\end{equation}
\end{theorem}

\begin{proof}
\emph{Reference sequence.}
Define $\ubf_1=\vbf_1$ and
$\ubf_{t+1}=\ubf_t-\eta\Gamma\gbf_t$. Expanding
$\|\ubf_{t+1}-\xbf^*\|^2$ and summing over $t$ gives
\begin{equation}
 \sum_{t=1}^T\mathbb E\langle\gbf_t,\ubf_t-\xbf^*\rangle
 \le\frac{\|\vbf_1-\xbf^*\|^2}{2\eta\Gamma}
 +\frac{\eta\Gamma}{2}(Q_T+\varsigma^2T).
 \label{eq:convex-ogd}
\end{equation}
Here we used
$\mathbb E\|\gbf_t\|^2=\mathbb E\|\nabla f(\ybf_t)\|^2+
\mathbb E\|\beps_t\|^2\le\mathbb E\|\nabla f(\ybf_t)\|^2+\varsigma^2$,
which follows from $\mathbb E[\beps_t\mid\mathcal F_{t-1}]=0$.
\emph{Deviation from the reference sequence.}
The predictable difference between base and reference sequences is
\begin{equation}
 \deltabf_t^v=\vbf_t-\ubf_t
 =-\eta\sum_{k=1}^{t-1}e_{t,k}\gbf_k.
 \label{eq:convex-comparison-sequence}
\end{equation}
For compact notation in the following matrix calculation, write
$\mathbf h_t=\nabla f(\ybf_t)$ and define the stacked vectors
$\mathbf h=(\mathbf h_1^\top,\ldots,\mathbf h_T^\top)^\top$ and
$\beps=(\beps_1^\top,\ldots,\beps_T^\top)^\top$ in $\mathbb R^{Td}$.
Since $\deltabf_t^v$ is $\mathcal F_{t-1}$-measurable,
$\mathbb E\langle\beps_t,\deltabf_t^v\rangle=0$, and hence
\begin{equation}
 \sum_{t=1}^T\mathbb E\langle\gbf_t,\deltabf_t^v\rangle
 =-\eta\mathbb E\langle\mathbf h,E_T\mathbf h\rangle
  -\eta\mathbb E\langle\mathbf h,E_T\beps\rangle.
 \label{eq:kernel-error-interaction}
\end{equation}
The symmetric part of $E_T$ gives
$-\mathbb E\langle\mathbf h,E_T\mathbf h\rangle\le\chi_TQ_T$.
For $k<\ell$, conditional mean zero gives
$\mathbb E\langle\beps_k,\beps_\ell\rangle=0$.
Since the coefficients are deterministic,
\begin{equation}
 \mathbb E\|E_T\beps\|^2
 =\sum_{t=1}^T\sum_{k=1}^{t-1}e_{t,k}^2
       \mathbb E\|\beps_k\|^2
 \le\varsigma^2\mathcal E_T.
 \label{eq:kernel-noise-energy}
\end{equation}
Historical noise can be correlated with later gradients. Hence its interaction
is bounded by Cauchy--Schwarz:
\[
 \left|\mathbb E\langle\mathbf h,E_T\beps\rangle\right|
 \le
 \sqrt{\mathbb E\|\mathbf h\|^2}
 \sqrt{\mathbb E\|E_T\beps\|^2}
 \le\varsigma\sqrt{\mathcal E_TQ_T}.
\]
Together with \eqref{eq:kernel-error-interaction}, this gives
\[
 \sum_{t=1}^T\mathbb E\langle\gbf_t,\deltabf_t^v\rangle
 \le\eta\chi_TQ_T+\eta\varsigma\sqrt{\mathcal E_TQ_T}.
\]

\emph{Conclusion.}
Combining with \eqref{eq:convex-conversion} and \eqref{eq:convex-ogd},
\begin{align*}
 T\mathbb E[f(\xbf_T)-f(\xbf^*)]
 +\left(\frac1{2L}-\frac{\eta\Gamma}{2}-\eta\chi_T\right)Q_T
 \le \frac{\|\vbf_1-\xbf^*\|^2}{2\eta\Gamma}
 +\frac{\eta\Gamma\varsigma^2T}{2}
 +\eta\varsigma\sqrt{\mathcal E_TQ_T}.
\end{align*}
The condition on the step size leaves at least $Q_T/(4L)$ on the left. Young's inequality
$\eta\varsigma\sqrt{\mathcal E_TQ_T}
\le Q_T/(8L)+2L\eta^2\varsigma^2\mathcal E_T$
proves \eqref{eq:kernel-refined-bound}.
\end{proof}

\begin{proof}[Proof of Theorem~\ref{thm:kernel-square-stability}]
The row bound implies $\mathcal E_T\le C_E T$ and
$\chi_T\le\|E_T\|_{\mathrm{op}}\le\|E_T\|_{\mathrm F}
\le\sqrt{C_E T}$.
For $T\ge1$, the stated $\eta_0$ ensures
$\eta(\Gamma+2\chi_T)
\le\eta_0(\Gamma/\sqrt T+2\sqrt{C_E})
\le\eta_0(\Gamma+2\sqrt{C_E})\le1/(2L)$.
Substitute $\eta=\eta_0/\sqrt T$ into
\eqref{eq:kernel-refined-bound}, divide by $T$, and use
$\mathcal E_T\le C_E T$ in the noise term.
Dropping the nonnegative term $Q_T/(8LT)$ gives the stated
bound on the objective value.
\end{proof}

The two quantities control different terms: $\chi_T$ bounds
$-\mathbb E\langle\mathbf h,E_T\mathbf h\rangle$, whereas $\mathcal E_T$
bounds $\mathbb E\|E_T\beps\|^2$ through
\eqref{eq:kernel-noise-energy}.
For Schedule-Free SGD with equal weights, $E_T=0$, $\Gamma=1$, and both
penalties vanish.

\subsection{Bounding cumulative gradient weight deviations for \hym{}}
\label{app:kernel-hym-certificate}

Fix $\mu\in(0,1)$, $\beta\in[0,1)$, $\lambda=1-\beta$, and
$\Gamma=(1-\mu)^{-1}$.
Write $q_t=(1-\mu^t)/(1-\mu)$ and
$a_t=[1+\lambda(t-1)]^{-1}$, as in Algorithm~\ref{alg:hym}.
For $1\le i\le t$, set
\[
 \Pi_{i+1,t}=\prod_{j=i+1}^t(1-a_j),\qquad
 \omega_{i,t}=1+\frac{\beta}{\lambda}\Pi_{i+1,t},
\]
where empty products equal one.  Define
\begin{equation}
 C_H=\frac{\beta\Gamma(1+\Gamma)+\mu\Gamma^2}{\lambda},
 \qquad
 C_E^{\hym{}}=\frac{\pi^2C_H^2}{3}
 +\frac{2\Gamma^2\mu^2}{1-\mu^2}.
 \label{eq:convex-kernel-constants}
\end{equation}

\begin{lemma}
\label{lem:hym-convex-kernel}
The base sequence $\{\vbf_t\}$ has the exact representation
\begin{align}
 \vbf_t&=\ybf_1-\eta\sum_{k=1}^{t-1}K_{t,k}\gbf_k,
 \label{eq:convex-kernel}\\
 K_{t,k}
 &=\sum_{i=k}^{t-1}\omega_{i,t}
 \left(\frac{\beta q_i k}{i^2}
 +\lambda\frac{k}{i}\mu^{i-k}\right).
 \label{eq:convex-kernel-coeff}
\end{align}
If $e_{t,k}=K_{t,k}-\Gamma$, then
\begin{equation}
 |e_{t,k}|\le\frac{C_H}{k}+\Gamma\mu^{t-k},\qquad
 \sum_{k=1}^{t-1}e_{t,k}^2\le C_E^{\hym{}}.
 \label{eq:convex-kernel-error}
\end{equation}
The coefficients and bounds are deterministic.
\end{lemma}

\begin{proof}
\emph{Representation.}
Unrolling the two momentum states gives
\begin{equation}
 r_t\mathbf m_t=\frac{q_t}{t^2}\sum_{k=1}^t k\gbf_k,
 \qquad
 \bbf_t=\sum_{k=1}^t\frac{k}{t}\mu^{t-k}\gbf_k.
 \label{eq:convex-stream-kernels}
\end{equation}
Let $\deltabf_t=\ybf_t-\xbf_t$. The recurrence for iterate averaging implies
$\deltabf_{t+1}=(1-a_{t+1})(\deltabf_t-\eta\mathbf\Delta_t)$.
Since $\deltabf_1=0$,
\[
 \vbf_t=\ybf_t+\frac{\beta}{\lambda}\deltabf_t
 =\ybf_1-\eta\sum_{i=1}^{t-1}\omega_{i,t}\mathbf\Delta_i.
\]
Substitution of \eqref{eq:convex-stream-kernels} proves
\eqref{eq:convex-kernel}--\eqref{eq:convex-kernel-coeff}.

\emph{Constant reference weights.}
The constant reference weight $\Gamma$ admits the identity
\begin{equation}
 \Gamma\beta\sum_{i=k}^{t-1}
 \omega_{i,t}\frac{k}{i(i+1)}
 +\Gamma\lambda\omega_{k,t}=\Gamma.
 \label{eq:convex-ideal-kernel}
\end{equation}
For $\beta>0$, this follows from
\begin{equation}
 \sum_{i=k}^{t-1}\frac{k\Pi_{i+1,t}}{i(i+1)}
 =\frac{\lambda}{\beta}
 \left(\frac kt-\Pi_{k+1,t}\right)
 \label{eq:convex-product-telescope}
\end{equation}
and $\sum_{i=k}^{t-1}k/(i(i+1))=1-k/t$.
To verify \eqref{eq:convex-product-telescope}, let
$r_i=\Pi_{i+1,t}/i$. From
$\Pi_{i+2,t}=\Pi_{i+1,t}(i+\lambda^{-1})/i$, one obtains
$r_{i+1}-r_i=(\lambda^{-1}-1)\Pi_{i+1,t}/(i(i+1))$, and the result follows by telescoping.
For $\beta=0$, \eqref{eq:convex-ideal-kernel} is immediate.

\emph{Deviation bound.}
Subtracting \eqref{eq:convex-ideal-kernel} from
\eqref{eq:convex-kernel-coeff} expresses $e_{t,k}$ as the sum of the differences
in the slow and fast terms. Since
$1\le\omega_{i,t}\le1/\lambda$ and $q_i=\Gamma(1-\mu^i)$, the slow difference satisfies
\[
\begin{aligned}
 &\left|\beta\sum_{i=k}^{t-1}\omega_{i,t}k
 \left(\frac{q_i}{i^2}-\frac{\Gamma}{i(i+1)}\right)\right|\\
 &\quad\le\frac{\beta\Gamma k}{\lambda}
 \sum_{i=k}^{t-1}\left(\frac{1}{i^2(i+1)}+\frac{\mu^i}{i^2}\right)\\
 &\quad\le\frac{\beta\Gamma}{\lambda k}(1+\Gamma\mu^k)
 \le\frac{\beta\Gamma(1+\Gamma)}{\lambda k}.
\end{aligned}
\]
Here we used
$\sum_{i=k}^{\infty}[i^2(i+1)]^{-1}
\le k^{-1}\sum_{i=k}^{\infty}[i(i+1)]^{-1}=k^{-2}$ and
$\sum_{i=k}^{\infty}\mu^i/i^2\le\Gamma\mu^k/k^2$.

For the fast difference, the product recurrence gives
$\omega_{j+1,t}-\omega_{j,t}=(\beta/\lambda)a_{j+1}\Pi_{j+2,t}$.
Since $\omega_{i,t}\ge1$, $\omega_{i,t}-\omega_{k,t}\le\beta(i-k)/(\lambda^2k)$, and
$\omega_{k,t}\le1/\lambda$ for $k\le i<t$,
\[
0\le\omega_{i,t}-\omega_{k,t}
 =\frac{\beta}{\lambda}\sum_{j=k}^{i-1}a_{j+1}\Pi_{j+2,t}
 \le\frac{\beta(i-k)}{\lambda^2 k}.
\]
Consequently, since $\beta+\lambda=1$,
\[
\begin{aligned}
 \left|\omega_{i,t}\frac ki-\omega_{k,t}\right|
 &\le\frac ki|\omega_{i,t}-\omega_{k,t}|
   +\omega_{k,t}\left(1-\frac ki\right)\\
 &\le\frac{\beta(i-k)}{\lambda^2 k}+\frac{i-k}{\lambda k}
 =\frac{i-k}{\lambda^2 k}.
\end{aligned}
\]
Separating the finite geometric sum from its tail therefore gives
\[
\begin{aligned}
 &\left|\lambda\sum_{i=k}^{t-1}\omega_{i,t}\frac ki\mu^{i-k}
 -\lambda\Gamma\omega_{k,t}\right|\\
 &\quad\le\frac1{\lambda k}\sum_{j=0}^{t-k-1}j\mu^j
 +\lambda\omega_{k,t}\sum_{j=t-k}^{\infty}\mu^j\\
 &\quad\le\frac{\mu\Gamma^2}{\lambda k}+\Gamma\mu^{t-k},
\end{aligned}
\]
where $\sum_{j=0}^{\infty}j\mu^j=\mu\Gamma^2$ and
$\lambda\omega_{k,t}\le1$. Adding the two differences proves
$|e_{t,k}|\le C_H/k+\Gamma\mu^{t-k}$. Finally, using
$(a+b)^2\le2a^2+2b^2$,
\[
\begin{aligned}
 \sum_{k=1}^{t-1}e_{t,k}^2
 &\le2C_H^2\sum_{k=1}^{t-1}\frac1{k^2}
 +2\Gamma^2\sum_{j=1}^{t-1}\mu^{2j}\\
 &\le\frac{\pi^2 C_H^2}{3}
 +\frac{2\Gamma^2\mu^2}{1-\mu^2}
 =C_E^{\hym{}}.
\end{aligned}
\]
This also holds at $t=1$, when the sum is empty, and completes the
uniform bound independently of the training horizon.
\end{proof}

\subsection{Convergence of the averaged iterate for \hym{}}
\label{app:kernel-hym-returned}

\begin{proof}[Proof of Corollary~\ref{cor:hym-convex-x}]
Equations~\eqref{eq:convex-template} and \eqref{eq:convex-kernel}
give \eqref{eq:kernel-framework} with $\vbf_1=\ybf_1$, $\zeta_t=\beta$,
and $\Gamma=(1-\mu)^{-1}$. Lemma~\ref{lem:hym-convex-kernel} gives
\begin{equation}
 |e_{t,k}|\le\frac{C_H}{k}+\Gamma\mu^{t-k},\qquad 1\le k<t,
 \label{eq:hym-kernel-certificate-main}
\end{equation}
where $C_H=[\beta\Gamma(1+\Gamma)+\mu\Gamma^2]/(1-\beta)$.
Using $(a+b)^2\le2a^2+2b^2$ and summing over $k$ yields
\begin{equation}
 \sum_{k=1}^{t-1}e_{t,k}^2
 \le \frac{\pi^2 C_H^2}{3}
 +\frac{2\Gamma^2\mu^2}{1-\mu^2}
 =C_E^{\hym{}}<\infty.
 \label{eq:hym-square-sum-main}
\end{equation}
Thus, the bound in \eqref{eq:hym-square-sum-main} holds uniformly in $t$.
Since $C_E^{\hym{}}$ depends only on $\mu$ and $\beta$, we may choose any
$0<\eta_0\le[2L(\Gamma+2\sqrt{C_E^{\hym{}}})]^{-1}$ independently of $T$.
Theorem~\ref{thm:kernel-square-stability} therefore yields
\begin{equation}
 \mathbb E[f(\xbf_T)-f(\xbf^*)]
 \le
 \frac{\|\ybf_1-\xbf^*\|^2/(2\eta_0\Gamma)+\eta_0\Gamma\varsigma^2/2}
      {\sqrt T}
 +\frac{2L\eta_0^2C_E^{\hym{}}\varsigma^2}{T}.
 \label{eq:convex-explicit-rate}
\end{equation}
The right-hand side is $\Ocal(T^{-1/2})$, which proves the claim.
\end{proof}

All constants are independent of $T$ for fixed $\mu,\beta$. They can
deteriorate as either parameter approaches $1$.
The proof uses convexity and the stochastic oracle, without
strong convexity, bounded iterates, or uniformly bounded gradients. The analysis
allows the effective history of the slow state to grow with $T$.

\subsection{Finite mixtures of EMAs with fixed decay rates}
\label{app:kernel-related}

As an illustrative specialization, consider normalized EMA states initialized at zero,
$\mbf_t^{(r)}=\rho_r\mbf_{t-1}^{(r)}+(1-\rho_r)\gbf_t$ for
$r=1,\ldots,R$, where $0\le\rho_r<1$, and
$\vbf_{t+1}=\vbf_t-\eta\sum_{r=1}^R a_r\mbf_t^{(r)}$ with
$a_r\ge0$ and $\Gamma=\sum_{r=1}^R a_r>0$.
Taking $\zeta_t=0$ and $\xbf_t=t^{-1}\sum_{i=1}^t\vbf_i$ gives an averaged
iterate that does not enter the update of $\vbf_t$.
Unrolling gives
\[
 e_{t,k}=-\sum_{r=1}^R a_r\rho_r^{t-k},\qquad
 C_E^{\rm EMA}=\sum_{r=1}^R\sum_{j=1}^R
 \frac{a_ra_j\rho_r\rho_j}{1-\rho_r\rho_j},\qquad
 \chi_T\le\sum_{r=1}^R\frac{a_r\rho_r}{1-\rho_r}.
\]
Here $e_{t,k}$ depends on the lag $t-k$ alone. The bound on the row square sums follows
by summing the geometric cross terms; the
operator bound follows from the sum of the absolute convolution weights.
Theorem~\ref{thm:kernel-square-stability} therefore recovers a
convergence guarantee for the averaged iterate under these linear combinations
of estimates of the first moment. This includes the first moment component of
AdEMAMix, which combines two EMA states with fixed decay rates and corresponds
to $a_1=1$, $a_2=\alpha$; the full adaptive AdEMAMix
denominator and schedules are not part of this example.
Aggregated Heavy-Ball analyzes multiple momentum states
\citep{danilova2022aggregated}. Generalized SGDM analyzes a single momentum
state combined with the current gradient \citep{wang2026generalized}.

The analysis builds on online-to-batch conversion
\citep{defazio2024road} and on perturbation arguments that compare the iterates
with a reference sequence, as used in analyses of error feedback, delay, and compression
\citep{karimireddy2019error,stich2020error}.
Analyses of optimizers with stored gradient histories also appear in
\citep{mcrae2022memory}.
Here the uniform bound is imposed directly on the deterministic cumulative weight
deviations $e_{t,k}=K_{t,k}-\Gamma$ in \eqref{eq:kernel-framework}. The analysis
uses the Bregman term to absorb the contribution of these weight deviations and bound the
gap in objective value at the final averaged iterate.
For \hym{}, the slow state has average lag $(t-1)/3$, and its cumulative
weight deviations satisfy the error bound
\eqref{eq:hym-kernel-certificate-main}.

\section{\hymuon{} matrix updates and implementation}
\label{app:hymuon-analysis}
\label{app:fast-hymuon}

\begin{algorithm}[H]
\renewcommand{\algorithmicrequire}{\textbf{Input:}}
\renewcommand{\algorithmicensure}{\textbf{Output:}}
\caption{\hymuon{} matrix updates}\label{alg:hm_muon}
\begin{algorithmic}[1]
\Require $\Ybf_1\in\mathbb R^{m\times n}$, $\eta>0$, $\mu\in(0,1)$,
$\beta\in[0,1)$, $B,T\in\mathbb N$; initialize
$\Mbf_0=\Bbf_0=\zerobf$, $\Xbf_0=\Ybf_1$
\For{$t=1,\ldots,T$}
\State Sample $\mathcal B_t=(\xi_{t,j})_{j=1}^B$ and set
$\Gbf_t=B^{-1}\sum_{j=1}^B\nabla f(\Ybf_t;\xi_{t,j})$
\State $\Mbf_t=(1-1/t)\Mbf_{t-1}+\Gbf_t$
\Comment{Slow state}
\State $\Bbf_t=\mu(1-1/t)\Bbf_{t-1}+\Gbf_t$
\Comment{Fast state}
\State $\mathbf D_t=\frac{\beta(1-\mu^t)}{t(1-\mu)}\Mbf_t
 +(1-\beta)\Bbf_t$
\State $a_t=\frac{1}{1+(1-\beta)(t-1)}$, \quad
$\Xbf_t=(1-a_t)\Xbf_{t-1}+a_t\Ybf_t$
\State $\Ybf_{t+1}=\Ybf_t-\eta\,\operatorname{NS}(\mathbf D_t)$
\EndFor
\Ensure $\Xbf_T$
\end{algorithmic}
\end{algorithm}

Here $2\Mbf_t/(t+1)$ equals the matrix-valued slow state in
Algorithm~\ref{alg:hym}; $\operatorname{NS}(\cdot)$ denotes the output of the
finite Newton--Schulz iteration used to approximate the polar factor.

We extend the matrix updates above to the optimizer used for the entire model in the
\hymuon{} experiments. The combined momentum for hidden matrices follows the
matrix-valued recurrence in Algorithm~\ref{alg:hm_muon}. The implementation for the entire model
uses five NS iterations and additionally includes shape scaling, decoupled weight decay,
the learning rate schedules used in individual experiments, and the adaptive recurrences below. We set
$c_t=1/t$. For a hidden
parameter $\Ybf_t\in\mathbb R^{m\times n}$ with gradient $\Gbf_t$, the
implementation initializes $\Mbf_0=\Bbf_0=\zerobf$ and forms
\begin{align}
\Mbf_t &= (1-c_t)\Mbf_{t-1}+\Gbf_t, \\
\Bbf_t &= \mu(1-c_t)\Bbf_{t-1}+\Gbf_t, \\
\mathbf D_t &= \beta c_t q_t(\mu)\Mbf_t+(1-\beta)\Bbf_t,
\qquad q_t(\rho)=\frac{1-\rho^t}{1-\rho}, \\
\mathbf U_t &= s_{m,n}\operatorname{NS}(\mathbf D_t),
\qquad s_{m,n}=\sqrt{\max\{1,m/n\}}.
\label{eq:fast-matrix-hymuon}
\end{align}
Here $\operatorname{NS}(\cdot)$ denotes the output of five quintic NS iterations;
the shape factor $s_{m,n}$ is applied afterwards.
Both $\Gbf_t$ and the auxiliary gradient $\gbf_t$ are mini-batch gradients
at the current training iterates, as in Algorithm~\ref{alg:hm_muon}.
For an auxiliary parameter tensor, let
$\mathbf h_t^{(1)}=\gbf_t$ and
$\mathbf h_t^{(2)}=\gbf_t\odot\gbf_t$.
The coefficients $\rho_1$ and $\rho_2$, corresponding to Adam's $\beta_1$
and $\beta_2$, control the fast recurrences for estimates of the first and second moments.
With all four unnormalized auxiliary moment states initialized
to zero, the adaptive \hym{} branch computes, for $r\in\{1,2\}$,
\begin{align}
\mbf_t^{(r)} &= (1-c_t)\mbf_{t-1}^{(r)}+\mathbf h_t^{(r)}, \\
\bbf_t^{(r)} &= \rho_r(1-c_t)\bbf_{t-1}^{(r)}+\mathbf h_t^{(r)}, \\
\mathbf d_t^{(r)} &= \beta c_t q_t(\rho_r)\mbf_t^{(r)}
                 +(1-\beta)\bbf_t^{(r)}, \\
\mathbf u_t &= \mathbf d_t^{(1)}\oslash
          \left(\sqrt{\mathbf d_t^{(2)}}+\varepsilon\right).
\label{eq:fast-aux-hymuon}
\end{align}
Here $\odot$ and $\oslash$ denote elementwise multiplication and division.
Write $\ybf_t$ for the auxiliary training iterate and $\xbf_t$ for the
corresponding averaged iterate. Initialize
$\Xbf_0=\Ybf_1$ and $\xbf_0=\ybf_1$. Before each training step, average
both parameter groups using $a_t=1/[1+(1-\beta)(t-1)]$:
\begin{equation}
 \Xbf_t=(1-a_t)\Xbf_{t-1}+a_t\Ybf_t,\qquad
 \xbf_t=(1-a_t)\xbf_{t-1}+a_t\ybf_t.
\end{equation}
Here $\rho_1$ and $\rho_2$ set the decay rates, and hence the effective history
lengths, of the auxiliary estimates of the first and second moments. The mixing weight
$\beta$ combines the slow and fast estimates in both
recurrences and determines the coefficient $a_t$ for iterate averaging.
For a hidden matrix and an auxiliary tensor, respectively, the subsequent
updates with decoupled weight decay are
\begin{align}
\Ybf_{t+1}
&=(1-\eta_t^{\mathrm{mat}}\lambda_{\mathrm{mat}})\Ybf_t
  -\eta_t^{\mathrm{mat}}\mathbf U_t, \\
\ybf_{t+1}
&=(1-\eta_t^{\mathrm{aux}}\lambda_{\mathrm{aux}})\ybf_t
  -\eta_t^{\mathrm{aux}}\mathbf u_t.
\end{align}
For evaluation, all parameter groups are temporarily set to their respective
$\Xbf_t$ and $\xbf_t$ values; the training iterates are restored before the next step.
Since the language models considered here do not use Batch Normalization, evaluation
requires no recalibration of running statistics.

\FloatBarrier
\section{Diagnostic experiment details}
\label{app:mechanistic-experiments}

This section provides the full protocols and additional results for the diagnostic
experiments reported in Section~\ref{sec:hym-understanding}. It covers the comparison
across batch sizes, the construction and statistics of the semantic relevance experiment,
and the interventions on one batch and on the slow momentum state.

\subsection{Dependence of the \hym{} advantage on batch size}
\label{app:hym-batch-size}

Because the variance of a mini-batch gradient generally increases as the batch size
decreases, we examine how the relative performance of \hym{} varies with batch size.
We use CIFAR-10 with WideResNet-16-8 and CIFAR-100 with DenseNet-BC, training each model
for 300 epochs on its full training set.
We use $B\in\{8,16,32,128,256,1024,2048\}$ on CIFAR-10 and
$B\in\{64,128,256,512,1024\}$ on CIFAR-100.
Within each dataset, the model, augmentation, and optimizer settings other than the learning rate
are fixed across batch sizes as in Appendix~\ref{app:vision-runs}; both
methods use a 5\% warmup followed by cosine decay, ending at half the peak
learning rate for \hym{} and zero for SGD.

Peak learning rates use candidates specific to each method and batch size,
obtained by linear scaling with batch size. Figure~\ref{fig:cifar-batch-size} reports test loss
for the selected configuration of each method and batch size.

\begin{figure}[H]
 \centering
 \includegraphics[width=0.49\linewidth]{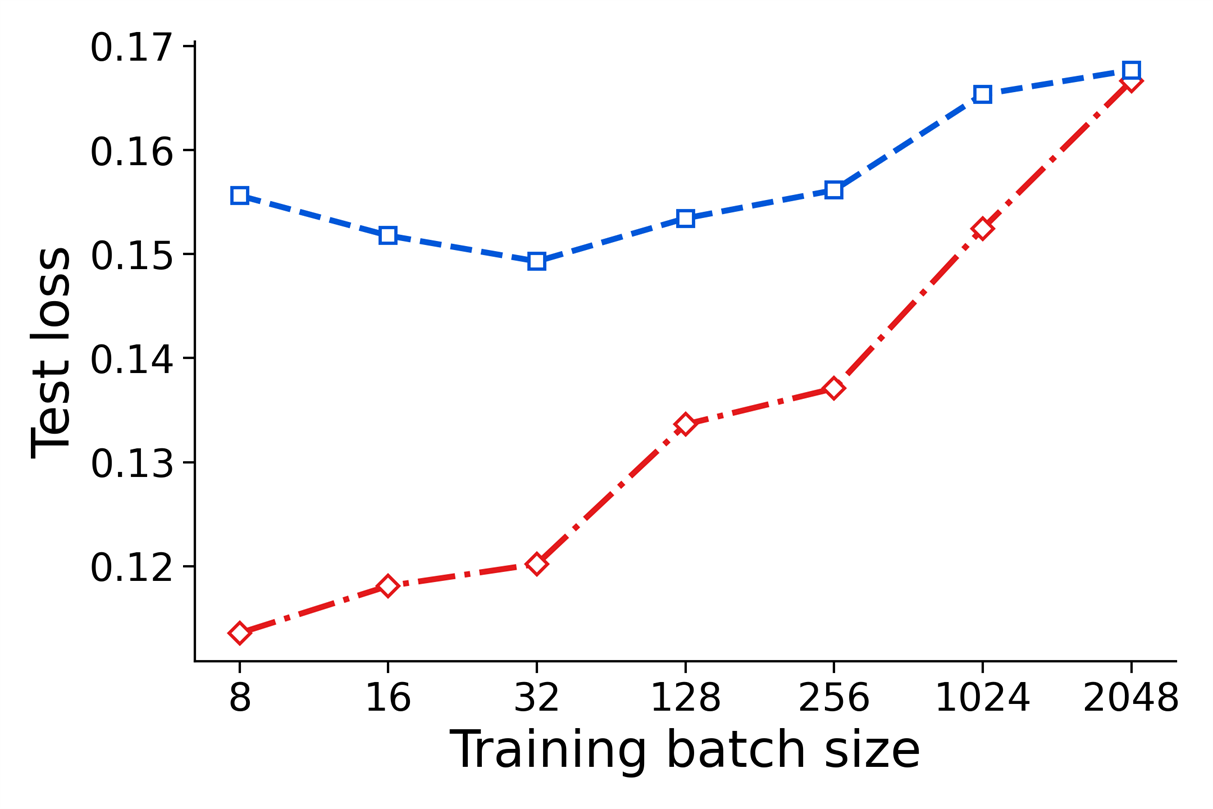}\hfill
 \includegraphics[width=0.49\linewidth]{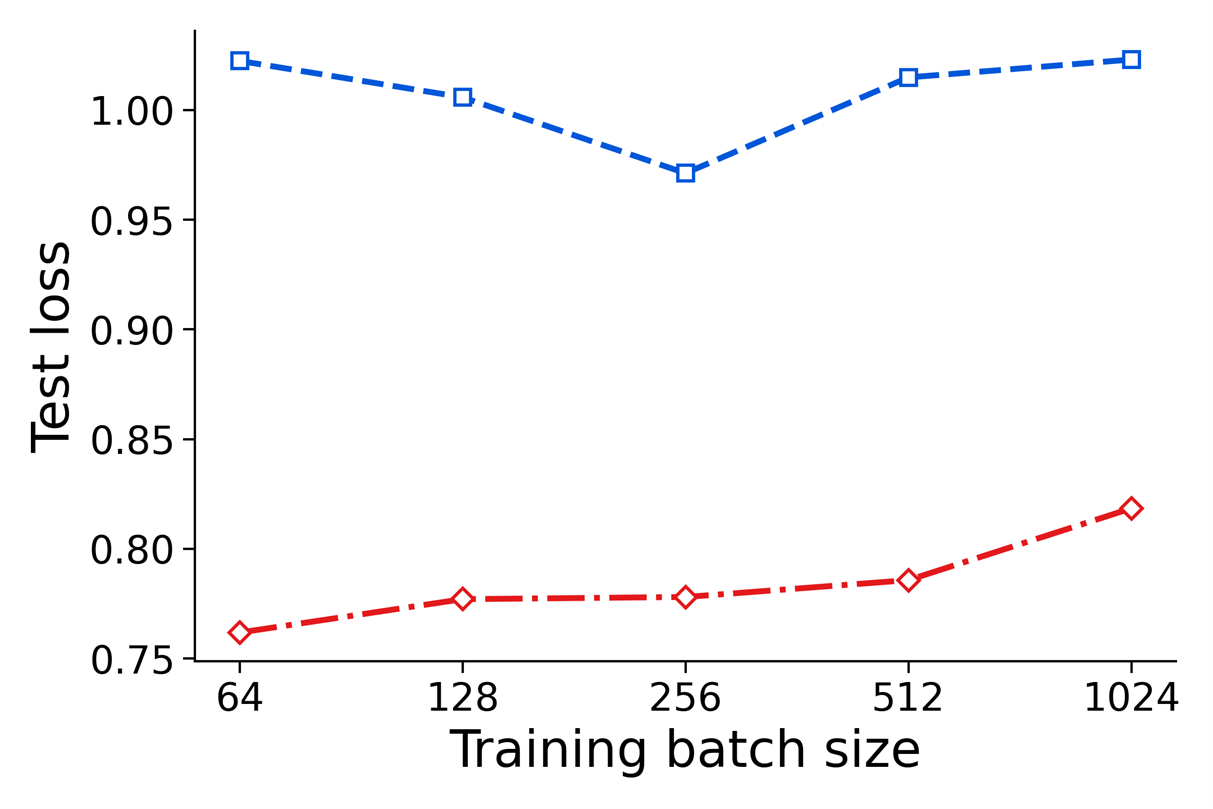}\\[-2pt]
 \includegraphics[width=0.98\linewidth]{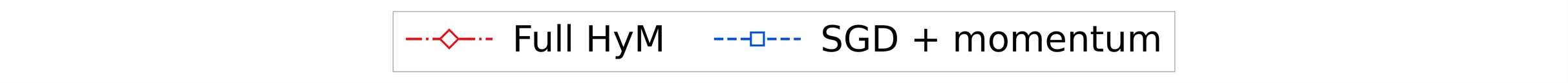}
 \caption{Test loss against training batch size on CIFAR-10 (left) and
 CIFAR-100 (right). The learning rate is chosen separately for each method and
 batch size.}
 \label{fig:cifar-batch-size}
\end{figure}

\hym{} has lower observed test loss at every plotted batch size. On CIFAR-10,
the gap narrows from about $0.042$ at $B=8$ to $0.001$ at $B=2048$ as the
batch size increases. On CIFAR-100, the gap is about $0.261$ at $B=64$ and
$0.205$ at $B=1024$, but does not vary monotonically across the intervening batch sizes.
\subsection{Semantic relevance and improvement specific to each target}
\label{app:narrow-signal}

\paragraph{Checkpoints and data.}
Section~\ref{sec:hym-narrow-signal} uses GPT-2 124M checkpoints at 10k, 50k,
and 100k steps of an AdamW run initialized from public GPT-2 weights and
continued on OpenWebText. Each step processes 491,520 tokens. AdamW uses
$(\beta_1,\beta_2)=(0.9,0.95)$, weight decay $0.1$, a warmup over 2k steps to
$6\times10^{-4}$, and cosine decay to $6\times10^{-5}$ at 100k.  The experiment
loads model parameters only. Its 2,000 candidates are uniformly sampled
training passages containing 256 tokens and lying within document boundaries.
They are screened to exclude passages that are nearly duplicates of the target texts.

\paragraph{Targets and similarity.}
A fixed pretrained \texttt{intfloat/e5-small-v2} encoder~\citep{wang2022e5} embeds
16,384 validation passages, with a \texttt{passage:} prefix and masked mean
pooling.  We subtract the embedding mean and its first three principal
components estimated on 4,096 background training passages, then normalize
each embedding. MiniBatchKMeans partitions the embeddings into 128 clusters.
We select ten clusters and take 64 passages from distinct documents in each
cluster, dividing them evenly into 32 reference and 32 evaluation passages.
Ten controls each mix
64 randomly chosen passages from other documents and use the same split.
All passages contain 256 tokens.  These targets are defined by the encoder
rather than by annotated grammatical or factual categories. Semantic similarity
is computed solely from the reference passages, whereas loss reduction is
evaluated solely on the evaluation passages.

For a candidate's embedding, its E5 score is the mean of its maximum
and mean cosine similarity to the 32 reference embeddings.  We also regress
this score on two measures of lexical overlap within each target: the cosine
similarity between vectors of token frequencies weighted by inverse document frequency,
and the cosine similarity between vectors of bigram frequencies after centering with
background statistics. The residual measures the part not linearly explained by
these two measures. Gradient cosine similarity compares the candidate gradient with
the gradient computed on the reference set at the same checkpoint.

\paragraph{Gradient updates with a fixed norm.}
For each candidate $C$, restore the checkpoint $\ybf$ and evaluate
\[
 \ybf_C=\ybf-\delta\frac{\nabla\mathcal L_C(\ybf)}{\|\nabla\mathcal L_C(\ybf)\|},
 \qquad \delta=3\times10^{-4}.
\]
The update uses no momentum, weight decay, iterate averaging, or optimizer state
carried over between candidates. Dropout is zero. For target $V$, define its loss reduction as
$I(C,V)=\mathcal L_{V_{\rm eval}}(\ybf)-\mathcal L_{V_{\rm eval}}(\ybf_C)$.
We define the improvement specific to target $V$ by subtracting from $I(C,V)$ the mean
of $I(C,V')$ over the other 19 targets. The analysis contains 120,000 measurements
over candidates, targets, and checkpoints.

\paragraph{Aggregation.}
We use the two quantities for different purposes. The percentile curves and net
contribution shares use the loss reduction $I(C,V)$. The correlation analysis uses
the improvement specific to the target. We also use this quantity when computing
shares after setting negative values to zero.
For each target and checkpoint, we rank all candidates by similarity and form
100 groups of 20 candidates. The percentile curves report the mean value of
$I(C,V)$ within each group. We also compute the Spearman correlation between
similarity and the improvement specific to that target, and then average
these correlations across targets. To measure how much of the total loss reduction
is contributed by the most similar candidates, we sum $I(C,V)$ over the top 400
candidates for each semantic target and divide by the corresponding sum over all
2,000 candidates and all ten semantic targets. Negative values are retained in
both sums; therefore, cancellation can increase this ratio. As a complementary statistic,
we set negative values of the improvement specific to the target to zero and average the resulting
shares across targets. Under this definition, the top 20\% accounts for $31.9\%$,
$33.4\%$, and $33.3\%$ at 10k, 50k, and 100k, respectively. The two statistics
measure different quantities.

\paragraph{Additional statistics.}
The randomly mixed controls provide a baseline in which the target passages are not
grouped by semantic coherence. Their mean loss reductions over 2,000 candidates and
ten controls are $9.272\times10^{-7}$, $2.250\times10^{-6}$, and
$1.398\times10^{-6}$ at 10k, 50k, and 100k, respectively. To determine whether the
association with E5 similarity is explained only by surface overlap, we remove the
component linearly explained by token and bigram overlap; the remaining association
is still positive. Gradient alignment provides a more direct comparison between a
candidate gradient and the target gradient. Its Spearman correlation with the
improvement specific to the target is $0.410$, $0.354$, and $0.399$, higher than the
corresponding correlation for semantic similarity. At 100k, $52.67\%$ of candidate
and target combinations have positive loss reduction. The net contribution share of
the top 20\% is $92.15\%$, with positive and negative values canceling in the total.

\paragraph{Percentile curves and contribution shares.}
Figures~\ref{fig:narrow-percentiles-full} and~\ref{fig:narrow-tail-share}
give the complete plots accompanying Section~\ref{sec:hym-narrow-signal}.
Both figures use the loss reduction $I(C,V)$, whereas the reported
correlations use the improvement specific to each target. Both figures rank candidates
by E5 similarity.

\begin{figure}[!htbp]
    \centering
    \includegraphics[width=0.88\linewidth]{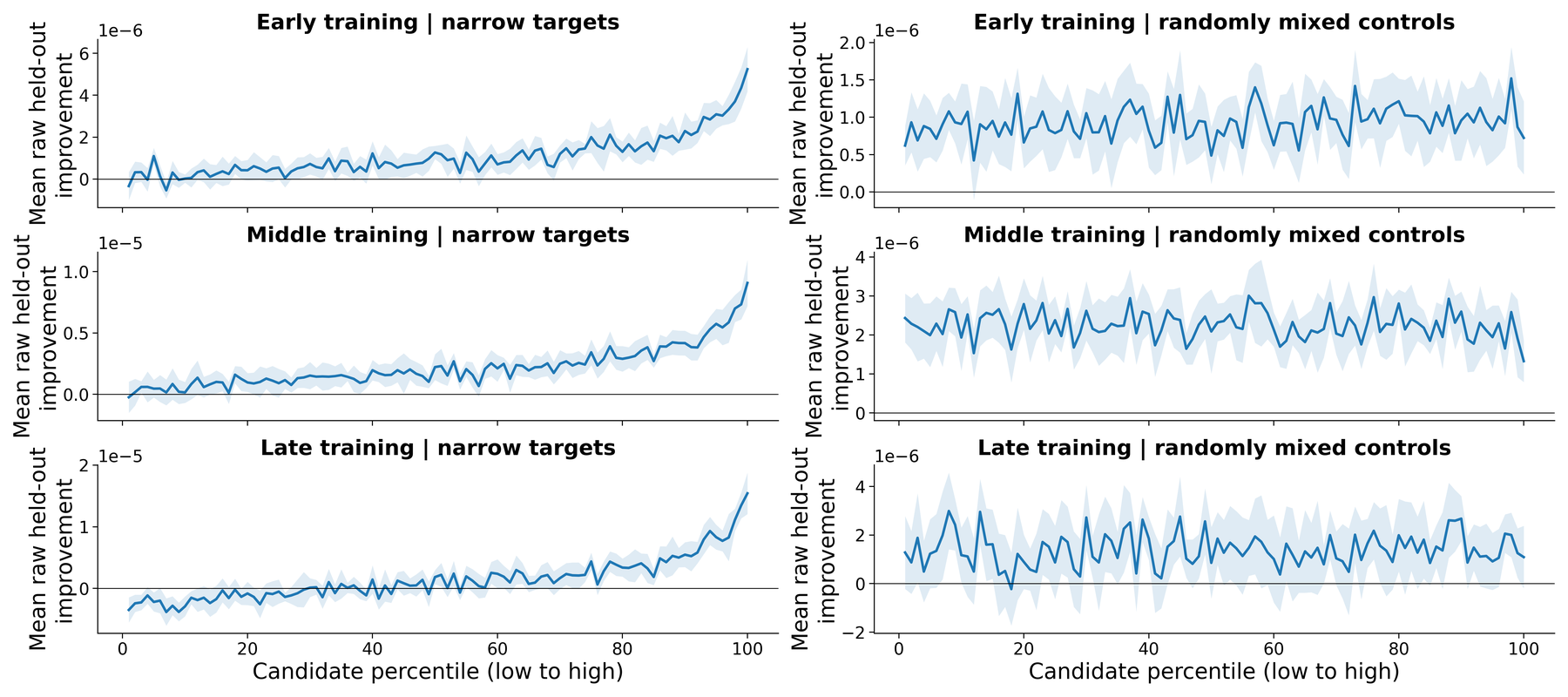}\\[-2pt]
    \includegraphics[width=0.88\linewidth]{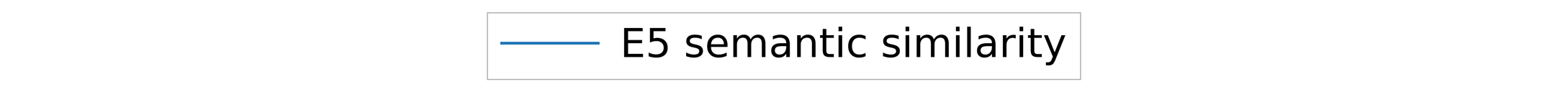}
    \caption{Curves of loss reduction across 100 percentiles of E5 similarity.
    Rows are the 10k, 50k, and 100k checkpoints; columns are semantic targets
    (left) and randomly mixed
    controls (right). Candidates are ranked within each target,
    with 20 candidates per percentile. Curves are means over ten targets of
    each type; shaded bands are based on standard errors across targets. The
    blue curve uses E5 similarity. Positive values indicate loss reduction;
    vertical scales differ. The panels in
    Figure~\ref{fig:narrow-percentiles} are the left column of this figure,
    rearranged in one row.}
    \label{fig:narrow-percentiles-full}
\end{figure}

\begin{figure}[!htbp]
    \centering
    \includegraphics[width=\linewidth]{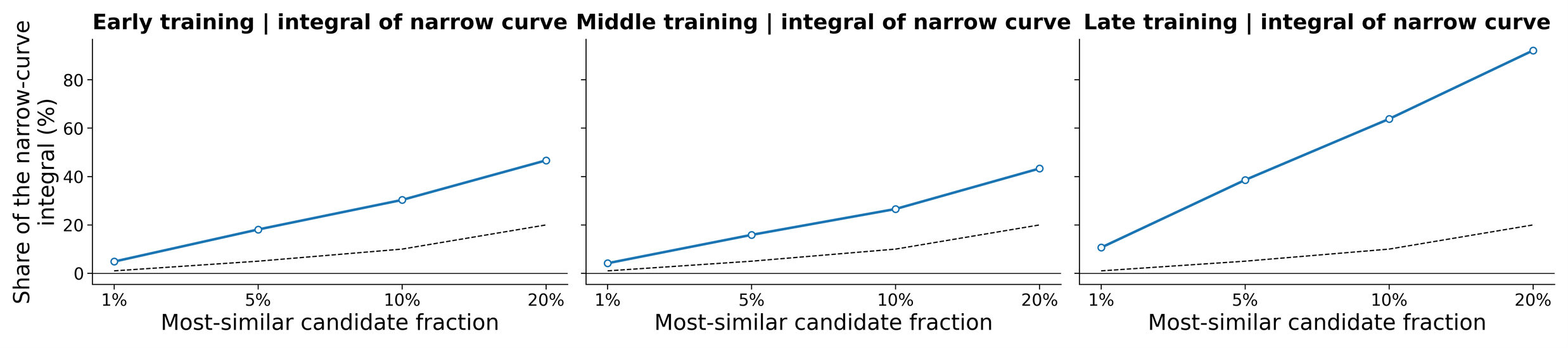}\\[-2pt]
    \includegraphics[width=\linewidth]{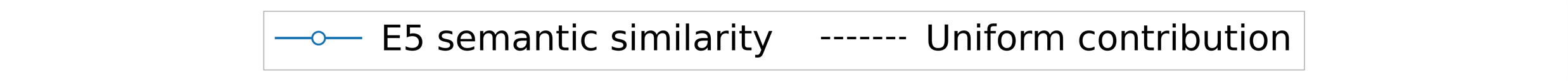}
    \caption{Share of the net loss reduction on semantic targets contributed by the
    candidates in the top $1\%,5\%,10\%$, and $20\%$ by E5 similarity. Panels are the
    10k, 50k, and 100k checkpoints, left to right. The blue curve
    uses the E5 ranking from Figure~\ref{fig:narrow-percentiles-full}. The
    numerator sums $I(C,V)$ over the selected candidates and all
    ten semantic targets; the denominator sums over all candidates and the
    same targets. The dashed line marks the share proportional to
    subset size. Because negative values are retained, cancellation can produce a
    share above $100\%$; the plotted quantities are net contribution shares,
    not fractions of positive loss reductions.}
    \label{fig:narrow-tail-share}
\end{figure}

\FloatBarrier
\subsection{Experiment on the lasting effect of one batch}
\label{app:forgetting}

The experiment in the left panel of Figure~\ref{fig:forgetting} uses the same
GPT-2 124M model, OpenWebText data, context length of 1024, and global batch of
480 sequences as the semantic target experiment above, but runs for 10,000 steps.
We select a fixed global batch $\mathcal B$ containing 491,520 targets for next token prediction.
Every sequence window that overlaps a token in $\mathcal B$ is excluded from the
remaining training batches used by both branches. The control and inject branches
start from the same complete checkpoint immediately before step 5,000. At step
5,000, the control branch uses the batch from the original training sequence,
whereas the inject branch replaces that batch with $\mathcal B$. The control branch
therefore never trains on $\mathcal B$. Apart from this replacement, stateless
sampling indexed by the run configuration, step, rank, and microstep gives the
two branches the same sequence of batches.

Training uses gradient clipping at 1.0, 200 warmup steps, and cosine decay
through step 10,000. Loss on $\mathcal B$ is evaluated every 50 steps and
additionally at steps 4,999, 5,000, and 5,001; distributed ranks evaluate
disjoint subsets whose losses are averaged. For these measurements, \hymuon{}
is temporarily evaluated at the averaged iterate $\xbf_t$, after which the
training iterate is restored. Muon is evaluated directly at its training iterate.
Because the branches share the complete checkpoint before injection and the
subsequent sampling state, their only protocol difference is the batch used at
step 5,000.

For Muon, the branch for hidden matrices uses peak and minimum learning rates
$2\times10^{-2}$ and $3\times10^{-5}$, momentum $0.95$, and zero weight
decay. Its auxiliary AdamW branch uses peak and minimum learning rates
$3\times10^{-3}$ and $3\times10^{-5}$, betas $(0.9,0.95)$, and weight decay
$0.1$; inverse layer scaling is enabled for the learning rate. For \hymuon{}, the
branch for hidden matrices uses peak and minimum learning rates $0.1$ and
$3\times10^{-5}$, Muon momentum $0.9$, \hym{} mixing weight $0.95$, and
weight decay $0.05$. Its adaptive \hym{} branch uses peak and minimum learning
rates $0.01$ and $3\times10^{-5}$, coefficients for the first and second moments
$0.95$ and $0.99$, and weight decay $0.01$. These settings define the two
loss trajectories for the two optimizers in the left panel. The subsequent intervention
on the slow state isolates the stored contribution while holding the remaining model
and optimizer state fixed.

\subsubsection{Intervention on the slow momentum state}
\label{app:forgetting-state}

The right panel of Figure~\ref{fig:forgetting-state} continues the injection protocol
with an intervention on the slow state at step 7,500. Let $\gbf_{\mathcal B}(s)$ and
$\gbf_{\mathcal R}(s)$ denote the gradients at the common parameters at step $s$ on
$\mathcal B$ and on the batch $\mathcal R$ from the original training sequence,
respectively. Because the inject branch replaces $\mathcal R$ with $\mathcal B$
at step $s=5{,}000$, the contribution introduced into its slow state is the signed
difference $\Delta\gbf_s=\gbf_{\mathcal B}(s)-\gbf_{\mathcal R}(s)$ rather than
$\gbf_{\mathcal B}(s)$ alone. In \hymuon{}'s slow state, a
gradient from step $s$ has coefficient $s/t$ at step $t$; consequently,
the coefficient of $\Delta\gbf_s$ is $2/3$ at the intervention.  The Delete
branch subtracts this exact contribution from the slow states of
the matrix and auxiliary parameter groups.  The Refresh branch first performs
the same deletion and then computes $\gbf_{\mathcal B}(7{,}500)$ at the current
training iterate, rescaling its matrix and auxiliary components separately to match
the norm of the corresponding removed component within each parameter group.

Immediately before the intervention, Keep, Delete, and Refresh are initialized from
identical copies of the inject branch. These copies include the model parameters and
buffers, all optimizer states including the averaged iterate $\xbf_t$, the scheduler
for the learning rate, the training counters, and the states of the random number
generators and data samplers on every rank. The three interventions are then applied
to these copies, followed by the same sequence of training batches. The gradient
computation required by Refresh preserves the random states and model buffers used
for subsequent training.

At the same lag, the corresponding coefficients of the matrix momentum and auxiliary first moment
for Muon are approximately $2.0\times10^{-56}$ and
$4.0\times10^{-116}$. Both are numerically negligible at this lag; hence, the intervention
isolates \hymuon{}'s persistent slow state rather than Muon's memory.
The intervention does not modify model parameters, the averaged iterate $\xbf_t$,
the fast state, or either state for the second moment. Immediately before
the next step in the common training sequence, the code applies
the complete NS or adaptive transformation to the original and
intervened states. After multiplication by the learning rate, the differences
between the resulting parameter updates have norms $0.104$ for Delete and $0.111$
for Refresh. At step
10,000, the losses on $\mathcal B$ are $3.047617$ for Delete, $3.038296$ for
Keep, and $3.025494$ for Refresh.

\FloatBarrier

\section{Benchmark experiment details}
\label{app:full-exp}

This section gives the architectures, training configurations, search dimensions, and
selected hyperparameters used in the reported experiments; each task lists the
hyperparameter families we tuned and the settings we selected. The FineWeb baselines
follow the
published benchmark, and the C4 comparison uses a Muon baseline reproduced
with the training configuration of \citet{gu2026mano}, as described below.

\paragraph{Results summary.}
Table~\ref{tab:main-results} collects the numerical results for
Section~\ref{sec:experiments}. Image classification and IWSLT entries report
the mean and sample standard deviation across repeated runs;
OpenWebText, FineWeb, and C4 entries give validation loss at the final training step.

\begin{table}[H]
\centering
\caption{Main experimental results. Image accuracy and IWSLT metrics are
mean $\pm$ sample SD across repeated runs; OpenWebText, FineWeb,
and C4 give the final validation loss.
ImageNet accuracy is validation top-1; CIFAR and SVHN use test accuracy.
A dash marks a method not included in that comparison. FineWeb losses are rounded
to two decimals.}
\label{tab:main-results}
\small
\setlength{\tabcolsep}{4pt}
\renewcommand{\arraystretch}{1.15}
\begin{tabular}{lrrrr}
\hline
\multicolumn{5}{l}{\textbf{Image classification: accuracy (\%) $\uparrow$}} \\
Dataset / model / budget
 & \shortstack{SGD\\Cosine} & \shortstack{SGD\\Step}
 & \shortstack{Schedule-\\Free} & \hym{} \\
\hline
\shortstack[l]{CIFAR-10 / WRN-16-8\\300 epochs}
 & $95.90\pm0.12$ & $95.43\pm0.17$ & $96.11\pm0.10$ & $96.25\pm0.10$ \\
\shortstack[l]{CIFAR-100 / DenseNet-BC\\300 epochs}
 & $77.66\pm0.39$ & $77.22\pm0.20$ & $79.11\pm0.22$ & $79.14\pm0.22$ \\
\shortstack[l]{SVHN / ResNet-3-96\\300 epochs}
 & $98.24\pm0.05$ & -- & $98.20\pm0.08$ & $98.47\pm0.02$ \\
\shortstack[l]{ImageNet-1K / ResNet-50\\100 epochs}
 & $76.68\pm0.04$ & -- & $76.70\pm0.06$ & $76.84\pm0.10$ \\
\hline
\end{tabular}
\par\vspace{0.7em}
\begin{tabular}{@{}llrrrr@{}}
\hline
\multicolumn{6}{l}{\textbf{Language tasks}} \\
Dataset / model / budget & Metric & AdamW & AdEMAMix & Muon & \hymuon{} \\
\hline
\shortstack[l]{IWSLT14 / LSTM\\55 epochs}
 & Test loss $\downarrow$ & -- & -- & $4.234\pm0.006$ & $4.184\pm0.004$ \\
 & BLEU $\uparrow$ & -- & -- & $33.424\pm0.048$ & $33.874\pm0.072$ \\
\shortstack[l]{OpenWebText / GPT-2 124M\\100k steps}
 & Val. loss $\downarrow$ & $2.905$ & -- & $2.915$ & $2.886$ \\
\shortstack[l]{OpenWebText / GPT-2 124M\\600k steps}
 & Val. loss $\downarrow$ & $2.866$ & -- & $2.886$ & $2.823$ \\
\shortstack[l]{FineWeb / LLaMA-124M\\8k steps ($\approx$1B tokens)}
 & Val. loss $\downarrow$ & $3.53$ & $3.48$ & $3.50$ & $3.47$ \\
\shortstack[l]{FineWeb / LLaMA-124M\\16k steps ($\approx$2.1B tokens)}
 & Val. loss $\downarrow$ & $3.39$ & $3.35$ & $3.39$ & $3.35$ \\
\shortstack[l]{C4 / LLaMA-1.3B\\10k steps\\Muon as in \citet{gu2026mano}}
 & Val. loss $\downarrow$ & -- & -- & $2.903$ & $2.777$ \\
\hline
\end{tabular}
\end{table}

\subsection{Image classification runs}
\label{app:vision-runs}

\paragraph{Models and training budgets.}
Table~\ref{tab:vision-training} gives the common settings for each dataset.
CIFAR-10 uses a WideResNet of depth 16 and width multiplier 8, without
dropout.  CIFAR-100 uses DenseNet-BC with block sizes $(6,12,24,16)$,
growth rate 12, 24 initial channels, bottleneck multiplier 4, compression
factor $0.5$, and no dropout.  The SVHN ResNet has three stages with
$(96,192,384)$ channels and three residual blocks, each with two convolutions, per stage.
ImageNet uses ResNet-50 initialized from scratch.  All four tasks use
cross-entropy loss without label smoothing, mixup, or CutMix; evaluation is
performed after every epoch.

\begin{table}[htbp]
\centering
\caption{Training settings for image classification. Batch sizes are global
numbers of images per optimizer step.}
\label{tab:vision-training}
\small
\begin{tabular}{lrrr}
\hline
Dataset & Epochs & Batch size & Input size \\
\hline
CIFAR-10 & 300 & 128 & $32\times32$ \\
CIFAR-100 & 300 & 64 & $32\times32$ \\
SVHN (train + extra) & 300 & 32 & $32\times32$ \\
ImageNet-1K & 100 & 256 & $224\times224$ \\
\hline
\end{tabular}
\end{table}

\paragraph{Preprocessing.}
CIFAR and SVHN training images use reflection padding of four pixels, random
$32\times32$ crops, and horizontal flips with probability $0.5$; test images
are not augmented.  SVHN training includes both the official training and
extra splits. After scaling pixels to $[0,1]$, CIFAR-10 and SVHN use
mean and standard deviation $(0.5,0.5,0.5)$ for each channel. CIFAR-100 uses
mean $(0.5074,0.4867,0.4411)$ and standard deviation
$(0.2011,0.1987,0.2025)$.  ImageNet training uses random resized crops to
$224\times224$ and random horizontal flips; evaluation resizes the shorter
edge to 256 and takes a $224\times224$ center crop. Its mean and standard deviation
for each channel are also $(0.5,0.5,0.5)$.

\paragraph{Optimizers and schedules.}
Table~\ref{tab:vision-optimizers} lists the settings held fixed across repeated
runs.  SGD uses standard momentum with zero dampening and no Nesterov term;
\hym{} also uses no additional Nesterov term.  Both apply coupled weight decay
to all parameters.  Cosine schedules have linear
warmup over the first $5\%$ of optimizer steps.  The cosine endpoint is
zero for SGD on all four datasets and for \hym{} on SVHN.
For \hym{} on CIFAR-10, CIFAR-100, and ImageNet, the endpoint is instead half the peak
learning rate: $0.025$, $0.25$, and $0.003$, respectively.  The Step baseline has no
warmup and multiplies its learning rate by $0.1$ at epochs 150 and 225, when
epochs are numbered from zero. Schedule-Free uses a $5\%$ linear warmup followed by a
constant learning rate, using the \texttt{SGDScheduleFree} implementation.

\begin{table}[htbp]
\centering
\caption{Fixed optimizer settings for image classification. LR denotes the
peak learning rate. The momentum column gives the SGD or \hym{} momentum
coefficient; $\beta$ gives the Schedule-Free interpolation coefficient or the
\hym{} mixing weight.}
\label{tab:vision-optimizers}
\small
\renewcommand{\arraystretch}{1.1}
\begin{tabular}{llrrrr}
\hline
Dataset & Optimizer & LR & Momentum & $\beta$ & Weight decay \\
\hline
CIFAR-10 & SGD + Cosine & 0.2 & 0.9 & -- & $10^{-4}$ \\
 & SGD + Step & 0.1 & 0.9 & -- & $10^{-4}$ \\
 & Schedule-Free & 10 & -- & 0.9 & $10^{-4}$ \\
 & \hym{} & 0.05 & 0.98 & 0.9 & $5\times10^{-4}$ \\
\hline
CIFAR-100 & SGD + Cosine & 0.05 & 0.9 & -- & $2\times10^{-4}$ \\
 & SGD + Step & 0.05 & 0.9 & -- & $2\times10^{-4}$ \\
 & Schedule-Free & 5 & -- & 0.9 & $2\times10^{-4}$ \\
 & \hym{} & 0.5 & 0.9 & 0.9 & $3\times10^{-4}$ \\
\hline
SVHN & SGD + Cosine & 0.1 & 0.9 & -- & $10^{-4}$ \\
 & Schedule-Free & 1 & -- & 0.9 & $2\times10^{-4}$ \\
 & \hym{} & 0.05 & 0.99 & 0.95 & $3\times10^{-4}$ \\
\hline
ImageNet-1K & SGD + Cosine & 0.05 & 0.9 & -- & $10^{-4}$ \\
 & Schedule-Free & 1.5 & -- & 0.9 & $5\times10^{-5}$ \\
 & \hym{} & 0.006 & 0.99 & 0.90 & $2\times10^{-4}$ \\
\hline
\end{tabular}
\end{table}

\paragraph{Evaluation and repeated runs.}
All displayed methods on CIFAR-10, CIFAR-100, SVHN, and ImageNet are evaluated
over repeated runs. Curves show means across runs; the bands show one standard
error, and the legends report sample standard deviations. ImageNet is evaluated
on the official validation split.

\hym{} and Schedule-Free refresh Batch Normalization statistics before evaluation.
The refresh uses 50 training mini-batches. All ImageNet runs use a global
batch of 256.

\paragraph{Training time and memory.}
We profiled \hym{} and cosine SGD in separate short runs on an NVIDIA H100
80GB GPU. Within each row of Table~\ref{tab:vision-cost}, the model, profiling
batch size, data pipeline, worker count, and optimizer steps per epoch match.
The table reports the mean elapsed time of the training loop over epochs 2--5 and the
maximum PyTorch CUDA allocated memory over those epochs.

\begin{table}[htbp]
\centering
\caption{H100 resource profiles for \hym{} versus cosine SGD. Batch sizes
in parentheses are the profiling batch sizes. Time is seconds per training
epoch; memory is peak CUDA allocated MiB. Changes are \hym{} relative to
cosine SGD.}
\label{tab:vision-cost}
\small
\begin{tabular}{lrrrrrr}
\hline
 & \multicolumn{3}{c}{Training time} & \multicolumn{3}{c}{Peak allocated memory} \\
Dataset (batch) & SGD & \hym{} & Change & SGD & \hym{} & Change \\
\hline
CIFAR-10 (128) & 4.71 & 4.98 & $+5.7\%$ & 1122 & 1254 & $+132$ \\
CIFAR-100 (512) & 6.80 & 6.73 & $-0.9\%$ & 7410 & 7423 & $+12$ \\
ImageNet-1K (256) & 591.62 & 601.43 & $+1.7\%$ & 11257 & 11561 & $+304$ \\
\hline
\end{tabular}
\end{table}

Across the three configurations, \hym{}'s measured training time differs from
cosine SGD by $-0.9\%$ to $+5.7\%$, while peak allocated memory increases by
12--304 MiB.

\begin{figure}[!htbp]
    \centering
    \includegraphics[width=0.49\linewidth]{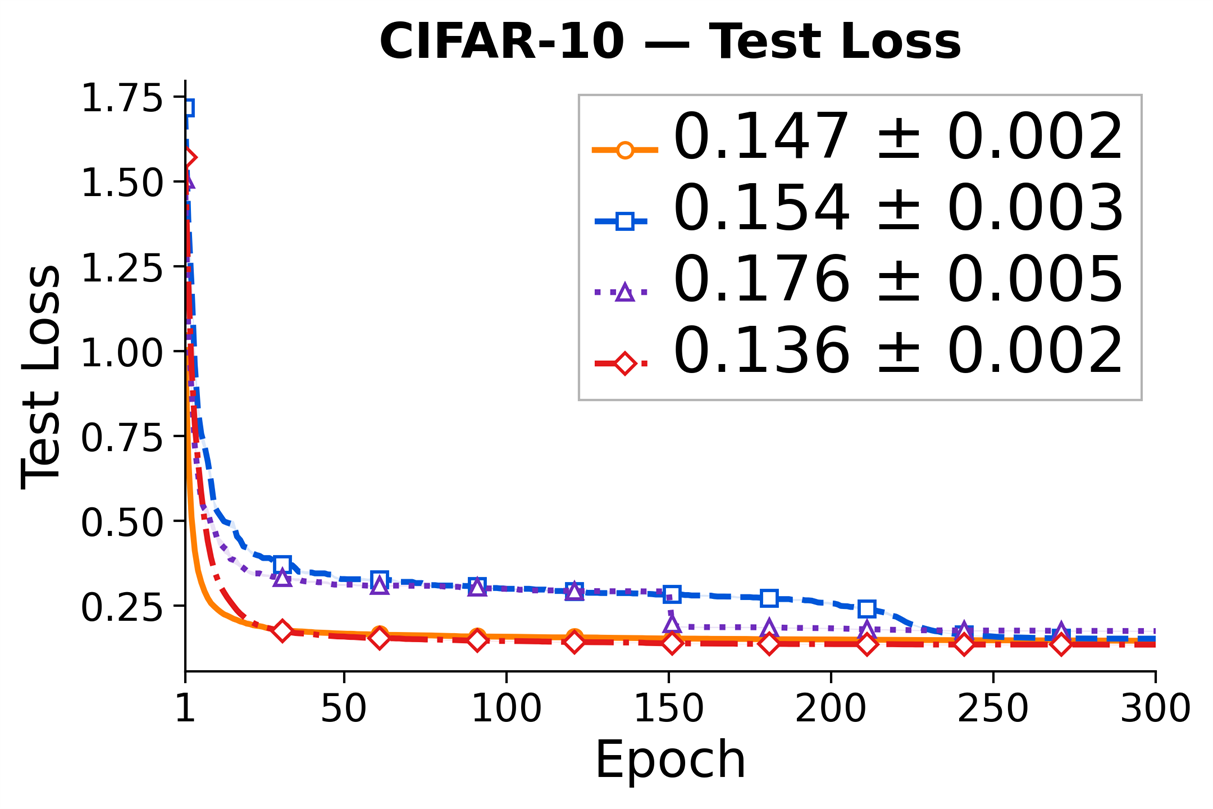}\hfill
    \includegraphics[width=0.49\linewidth]{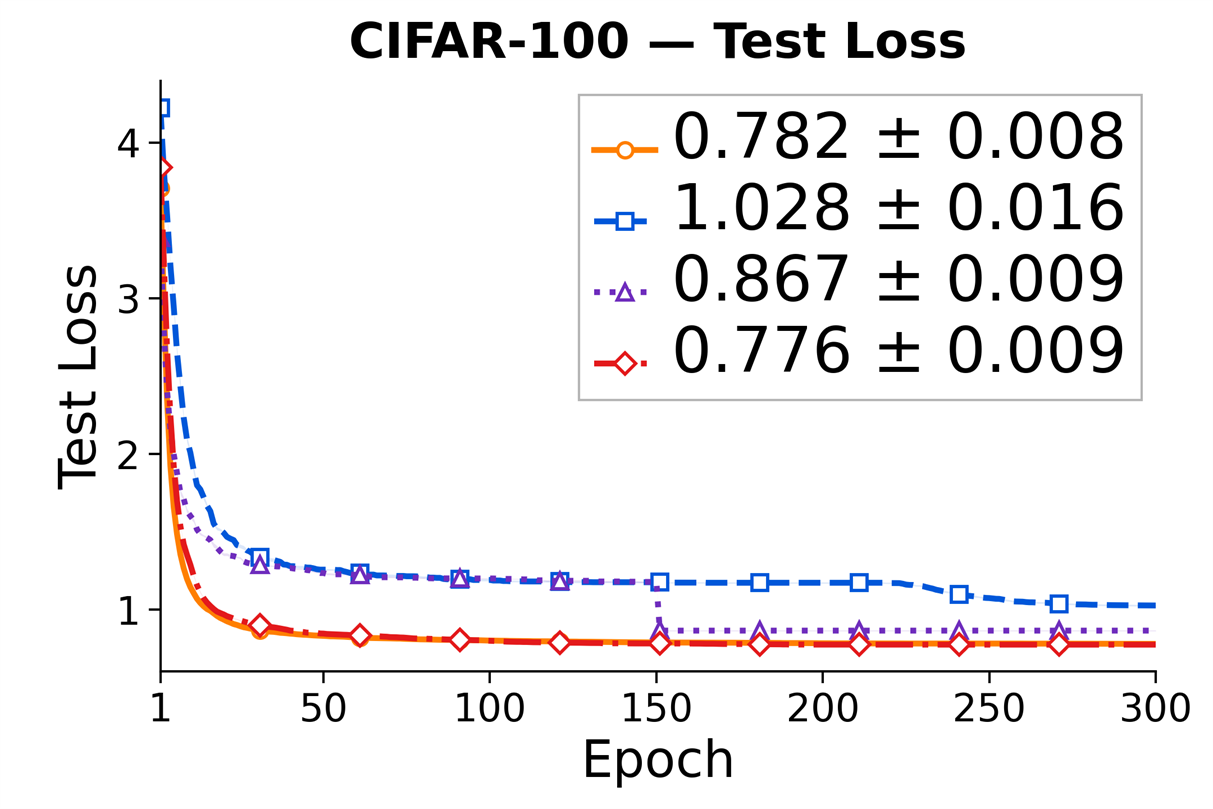}\\[-2pt]
    \includegraphics[width=0.98\linewidth]{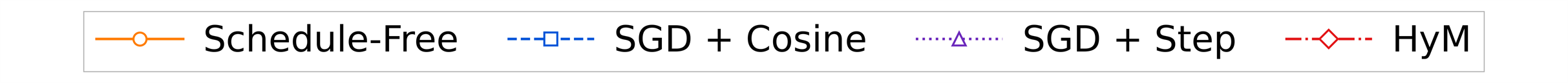}
    \caption{Curves of CIFAR test loss for the runs in
    Figure~\ref{fig:image-classification}. Left: CIFAR-10; right: CIFAR-100.
    Curves are means over repeated runs, and the bands show one standard error; the
    legend gives the sample SD. \hym{} is evaluated at the averaged iterate $\xbf_t$.}
    \label{fig:cifar-test-loss}
\end{figure}

\FloatBarrier

\subsection{Evaluation of \hym{} at the training and averaged iterates}
\label{app:cifar-readout-comparison}

Figure~\ref{fig:cifar-readout-comparison} compares a \hym{} configuration
evaluated at the training iterate \(y\) with a \hym{} configuration evaluated
at the averaged iterate \(x\) on CIFAR-10 and CIFAR-100. Both configurations
train for 300 epochs with the architectures and batch sizes in
Table~\ref{tab:vision-training}.

The configuration evaluated at the training iterate uses momentum \(0.9\), mixing weight
\(\beta=0.5\), and weight decay \(5\times10^{-4}\) on CIFAR-10; the
corresponding CIFAR-100 values are \(0.9\), \(0.3\), and
\(5\times10^{-4}\). Its peak learning rates are \(0.2\) and \(0.05\),
respectively, and both decay to approximately zero. The configuration evaluated at the averaged iterate
is given in Table~\ref{tab:vision-optimizers} and retains half
its peak learning rate. Each configuration uses the best hyperparameters found
in its own search; the comparison therefore evaluates each iterate under a
separately selected setting.

\begin{figure}[htbp]
\centering
\includegraphics[width=\linewidth]{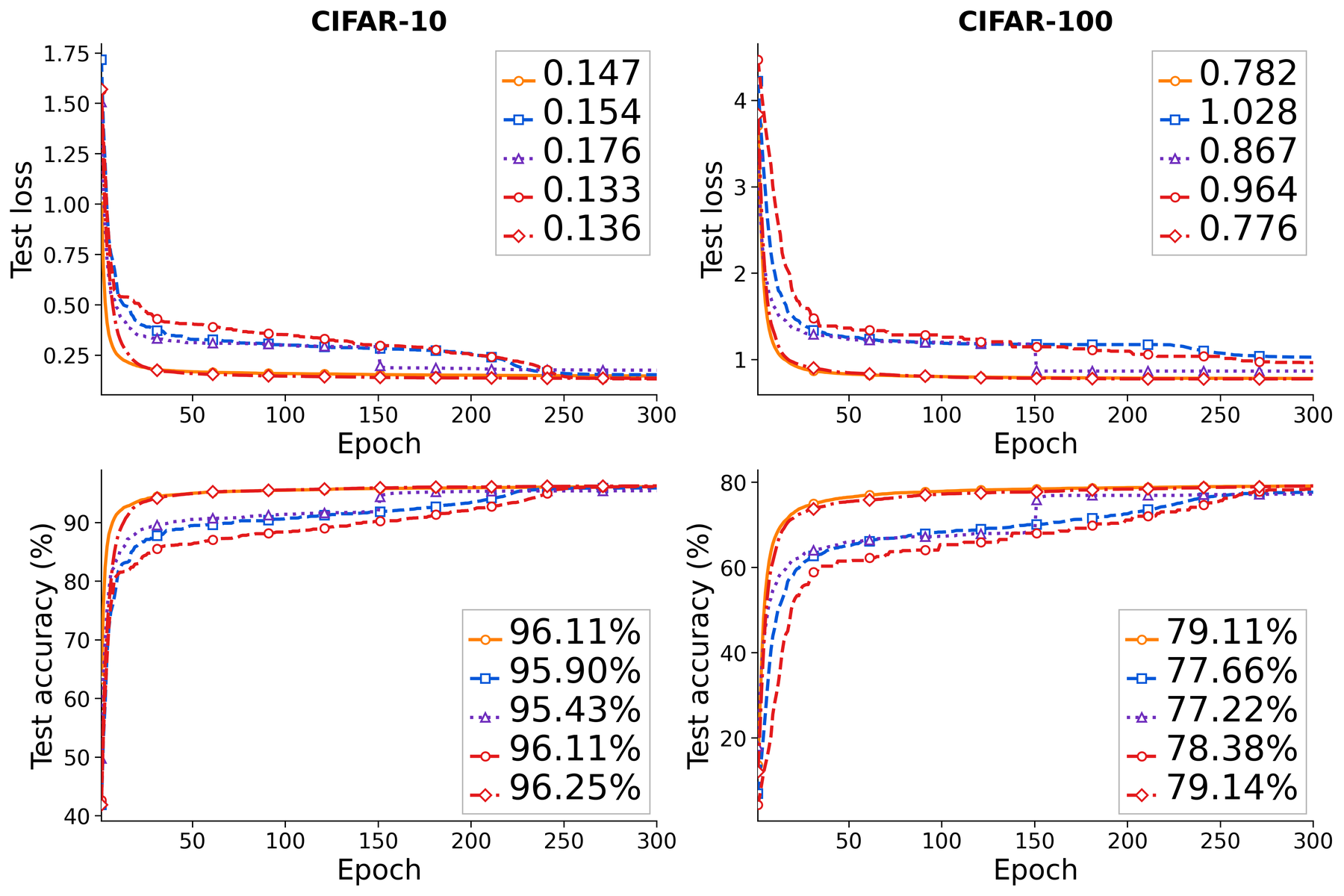}\\[-2pt]
\includegraphics[width=\linewidth]{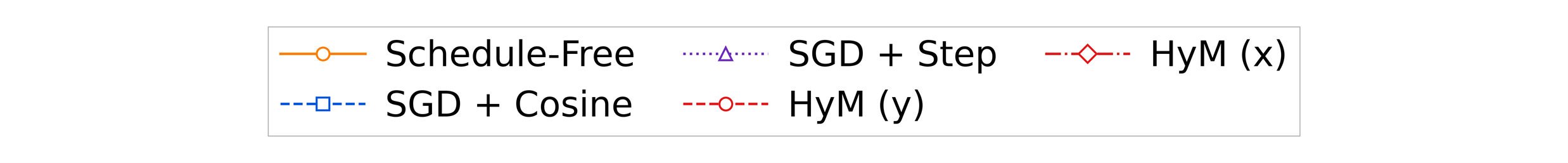}
\caption{CIFAR comparison of the training and averaged iterates. Columns
show CIFAR-10 and CIFAR-100; rows show test loss and test accuracy. \hym{}
\((y)\) is evaluated at the training iterate without iterate averaging, while
\hym{} \((x)\) is evaluated at the averaged iterate. The other curves reuse the
baselines from Figure~\ref{fig:image-classification}.}
\label{fig:cifar-readout-comparison}
\end{figure}

At the training iterate \(y\), \hym{} attains lower test loss than cosine SGD on both
datasets: \(0.1329\) versus \(0.1536\) on CIFAR-10, and \(0.9645\)
versus \(1.0279\) on CIFAR-100. At the averaged iterate \(x\), the corresponding
\hym{} losses are \(0.1358\) and \(0.7764\). Evaluating at \(y\) gives the lower loss on CIFAR-10. Evaluating at \(x\) gives the lower loss on CIFAR-100 and the higher accuracy
on both datasets. Thus, the improvement over cosine SGD is already present at
the training iterate. The paired component ablation in
Appendix~\ref{app:cifar10-component-ablation} evaluates every training trajectory
at both iterates to isolate the additional effect of iterate averaging.
\FloatBarrier

\subsection{Component ablation of \hym{} on CIFAR-10}
\label{app:cifar10-component-ablation}

We compare the fast momentum, slow momentum, their combination,
and three replacements of the slow momentum by an EMA with a fixed decay rate. Each
training trajectory is evaluated at both the training iterate $\ybf$ and the
averaged iterate $\xbf$, which allows the effect of iterate averaging to be separated
from the choice of gradient history.

\paragraph{Training and evaluation.}
We train WideResNet-16-8 for 300 epochs on a fixed split stratified by class
of CIFAR-10: 45,000 training images and 5,000 validation images, with 500
validation images per class. Batch size is 128 without dropping the final batch,
giving 352 steps per epoch and $T=105{,}600$ steps in total. The three peak learning
rates $\eta_{\max}\in\{0.05,0.1,0.2\}$ define separate comparison groups,
all of which are reported. To distinguish the two roles of the mixing weight in
this ablation, we denote the direction and averaging weights by
$\beta_{\mathrm{dir}}$ and $\beta_{\mathrm{avg}}$, respectively; both are set to
$0.95$. All groups use $\mu=0.99$ and coupled L2 decay
$\lambda_{\mathrm{wd}}=5\times10^{-4}$ on all parameters. The learning rate warms up
linearly for 5,280 steps, then follows a cosine decay to
$\eta_{\max}/2$. Training uses cross-entropy, deterministic cuDNN, and no cuDNN benchmarking. Augmentation is a random
$32\times32$ crop after reflection padding of four pixels and a horizontal
flip with probability $0.5$; each channel is normalized as
$(\mathrm{pixel}/255-0.5)/0.5$.

Initialization, the order of training images, and the randomness in cropping and flipping are shared
across configurations. After every epoch, $\ybf$ and $\xbf$ are copied
in turn into a separate evaluation model. For each evaluated iterate, Batch Normalization
statistics are reset and recalibrated by cumulative averaging on the same
6,400 unaugmented training images (50 batches). Neither parameters nor BN
buffers are copied back into the training model. We report epoch-300
validation metrics.

\paragraph{Matching the total weight across directions.}
To compare how the variants distribute weight over the gradient history without
also changing the total scale of the direction, we choose their coefficients so
that every variant gives the same direction when the input gradient remains constant.
Let $\mbf_t$ and $\bbf_t$ be the states of Algorithm~\ref{alg:hym},
with input $\gbf_t=\nabla f(\ybf_t;\xi_t)+\lambda_{\mathrm{wd}}\ybf_t$.
Use the quantities $Z_t,A_t,C_t$ and total coefficient sum $S_t=A_t+C_t$ defined in
\eqref{eq:hym-component-masses}--\eqref{eq:hym-total-gain}, with
$\beta=\beta_{\mathrm{dir}}$. The normalized fast direction is
$\widehat{\bbf}_t:=\bbf_t/Z_t$. For the EMA replacements, define the
direction after bias correction
$\widehat{\gbf}_t^{(\rho)}=\frac{1-\rho}{1-\rho^t}
\sum_{i=1}^t\rho^{t-i}\gbf_i$.
All variants update $\ybf_{t+1}=\ybf_t-\eta_t\mathbf D_t$, where
\[
\mathbf D_t=
\begin{cases}
 S_t\widehat{\bbf}_t, & \text{fast momentum only},\\
 S_t\mbf_t, & \text{slow momentum only},\\
 A_t\mbf_t+C_t\widehat{\bbf}_t, & \text{full \hym{}},\\
 A_t\widehat{\gbf}_t^{(\rho)}+C_t\widehat{\bbf}_t,
   & \text{replacement of the slow state by an EMA},\quad
     \rho\in\{0.99,0.999,0.9999\}.
\end{cases}
\]
For a constant input gradient $\gbf$, every variant gives
$\mathbf D_t=S_t\gbf$. This construction matches the total gradient weight across
variants, while their realized update norms may differ. The fast momentum variant retains
\hym{}'s factor $\mu(t-1)/t$. The EMA variants replace the slow state while retaining the
fast state and the other update rules; they are partial replacements used to isolate the
role of the slow state.
All variants use the same average computed before the update,
$\xbf_t=(1-a_t)\xbf_{t-1}+a_t\ybf_t$, with
$a_t=[1+(1-\beta_{\mathrm{avg}})(t-1)]^{-1}$. After $T$
steps, the paired training and averaged iterates are $\ybf_{T+1}$ and $\xbf_T$.

\begin{table}[htbp]
\centering
\caption{Final CIFAR-10 validation accuracy (\%) in the component ablation.
Each group defined by a learning rate compares the same six configurations; $y$ and $x$
are paired evaluations at epoch 300. EMA rows replace only the slow momentum.
These are validation results on the 45k/5k split, not the test accuracies
in the main image classification experiment.}
\label{tab:cifar10-component-ablation}
\small
\setlength{\tabcolsep}{5pt}
\renewcommand{\arraystretch}{1.08}
\begin{tabular}{lrrrrrr}
\hline
 & \multicolumn{2}{c}{$\eta_{\max}=0.05$}
 & \multicolumn{2}{c}{$\eta_{\max}=0.1$}
 & \multicolumn{2}{c}{$\eta_{\max}=0.2$}\\
Direction & $y$ & $x$ & $y$ & $x$ & $y$ & $x$\\
\hline
Fast momentum only & 84.26 & 90.44 & 76.78 & 83.76 & 65.68 & 72.18\\
Slow momentum only & 89.64 & 92.06 & 86.80 & 89.02 & 81.82 & 84.48\\
Full \hym{} & 93.94 & 96.22 & 93.44 & 96.10 & 91.92 & 95.52\\
EMA, $\rho=0.99$ & 85.04 & 90.86 & 77.60 & 83.40 & 68.50 & 73.18\\
EMA, $\rho=0.999$ & 89.30 & 93.40 & 84.18 & 89.58 & 79.88 & 82.74\\
EMA, $\rho=0.9999$ & 92.94 & 95.62 & 92.56 & 95.64 & 89.46 & 93.76\\
\hline
\end{tabular}
\end{table}

\paragraph{Results.}
Full \hym{} has the highest final validation accuracy at both iterates in
each group defined by a learning rate (Table~\ref{tab:cifar10-component-ablation}).
The advantage over either individual momentum state is already present at the training iterate;
therefore, iterate averaging alone does not account for this difference. Averaging
improves final accuracy in all 18 configurations, by 2.22--6.98 percentage
points; for full \hym{} the gains are 2.28, 2.66, and 3.60 points at the three
learning rates. Among the EMA replacements, $\rho=0.9999$ is closest to
full \hym{}, with accuracies at the averaged iterate of $95.62\%$ versus $96.22\%$ at
$\eta_{\max}=0.05$, and $95.64\%$ versus $96.10\%$ at
$\eta_{\max}=0.1$.

Figure~\ref{fig:cifar10-component-loss} shows the corresponding validation
loss histories. Fast momentum only and the $\rho=0.99$ replacement exhibit a
pronounced intermediate loss increase at the larger learning rates;
full \hym{} and the $\rho=0.9999$ replacement avoid this large rebound.
Full \hym{} also has the lowest final loss at both iterates in each group.
These comparisons support using the two momentum states together under the common
experimental protocol and show an additional improvement from iterate averaging.

\begin{figure}[htbp]
\centering
\includegraphics[width=0.98\linewidth]{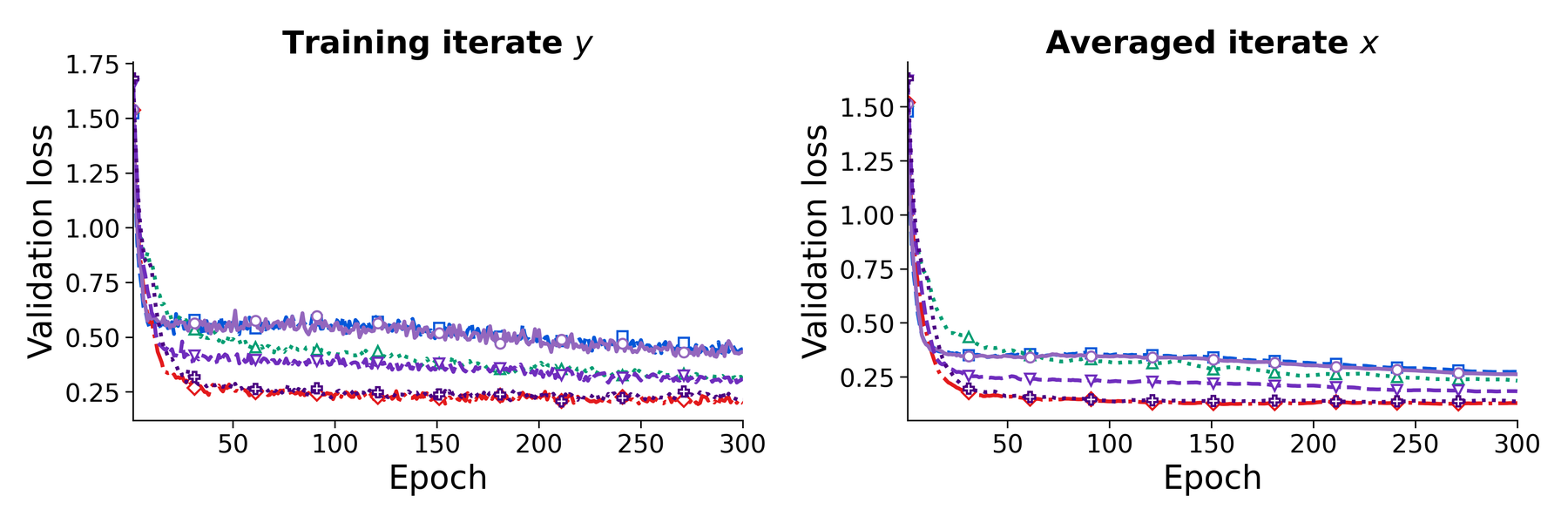}\\[3pt]
\includegraphics[width=0.98\linewidth]{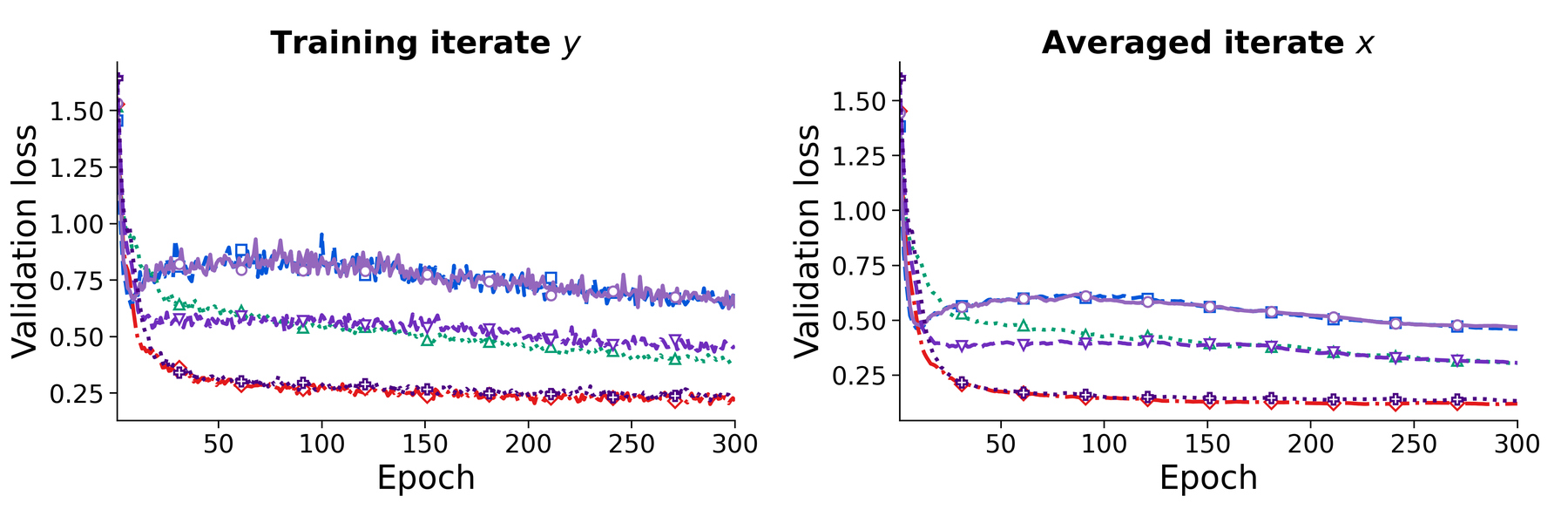}\\[3pt]
\includegraphics[width=0.98\linewidth]{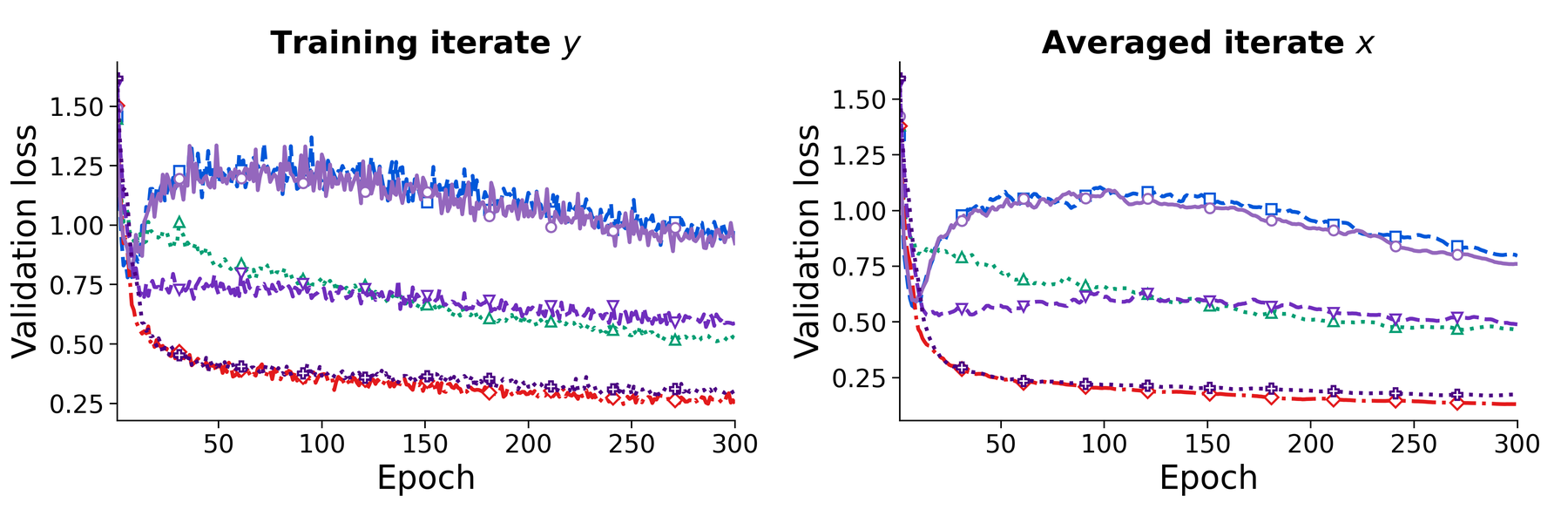}\\[-2pt]
\includegraphics[width=0.98\linewidth]{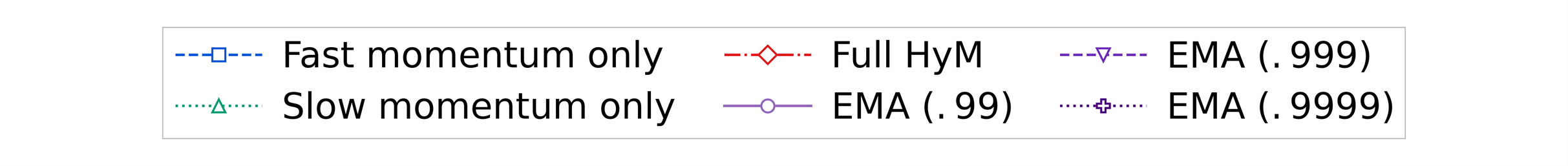}
    \caption{CIFAR-10 component ablation: validation cross-entropy at the
training iterate $y$ (left) and averaged iterate $x$ (right). Rows use peak learning rates
$0.05$, $0.1$, and $0.2$. Curves show values for each epoch.
Every evaluation uses the same BN calibration protocol. ``Slow EMA''
retains the fast momentum and replaces only the
slow momentum.}
\label{fig:cifar10-component-loss}
\end{figure}
\FloatBarrier

\subsection{IWSLT14 German--English}
\label{app:iwslt-runs}

We use the IWSLT14 German--English corpus \citep{cettolo2014iwslt}, with 160,239 training, 7,283 validation, and 6,750 test
sentence pairs. The model uses Fairseq's LSTM encoder--decoder preset with attention
and one layer on each side \citep{ott2019fairseq},
\texttt{lstm\_wiseman\_iwslt\_de\_en}, with
embeddings and hidden states of dimension 256, shared decoder input/output
embeddings, dropout $0.3$, and label smoothing $0.1$. Each task uses one GPU,
a dynamic limit of 4096 tokens per batch and a gradient accumulation factor of one.
Runs are capped at 55 epochs and 60,000 steps, with evaluation after every
epoch. Tokenized BLEU uses beam size 5, length parameters
$1.2$ and 10, space detokenization, and removal of BPE continuation markers.

\paragraph{Optimizer configuration.}
We search the learning rates for the hidden matrices and auxiliary parameters,
weight decay, and the momentum, mixing weight, and coefficient for the auxiliary
second moment used by each optimizer.
Table~\ref{tab:iwslt-optimizers} reports the selected configurations. Both
parameter groups use linear warmup over 4,000 steps followed by
linear decay to zero at step 60,000, retaining their peak learning rate
ratio throughout the schedule.  Gradient clipping is disabled.
Hidden weight matrices use five quintic NS iterations and the shape
factor $\sqrt{\max\{1,m/n\}}$ for an $m\times n$ matrix, without an
additional Nesterov term. Embeddings, output heads, and parameters with one dimension
belong to the auxiliary group. Muon uses AdamW for this group;
\hymuon{} uses the adaptive \hym{} moments described in
Appendix~\ref{app:fast-hymuon} and evaluates both groups at their averaged iterates.

\begin{table}[htbp]
\centering
\caption{IWSLT14 optimizer hyperparameters. The auxiliary coefficients are
AdamW for Muon and the adaptive \hym{} branch for \hymuon{}.}
\label{tab:iwslt-optimizers}
\small
\renewcommand{\arraystretch}{1.1}
\begin{tabular}{lrr}
\hline
Hyperparameter & Muon & \hymuon{} \\
\hline
Peak LR for hidden matrices & 0.05 & 0.05 \\
Momentum $\mu$ for hidden matrices & 0.99 & 0.9 \\
\hym{} mixing weight $\beta$ & -- & 0.9 \\
Weight decay for hidden matrices & 0.05 & 0.05 \\
Auxiliary peak LR & 0.01 & 0.05 \\
Coefficient for the auxiliary first moment & 0.9 & 0.9 \\
Coefficient for the auxiliary second moment & 0.98 & 0.95 \\
Auxiliary weight decay & 0.05 & 0.05 \\
Auxiliary $\varepsilon$ & $10^{-8}$ & $10^{-8}$ \\
\hline
\end{tabular}
\end{table}

\begin{figure}[!htbp]
    \centering
    \includegraphics[width=0.62\linewidth]{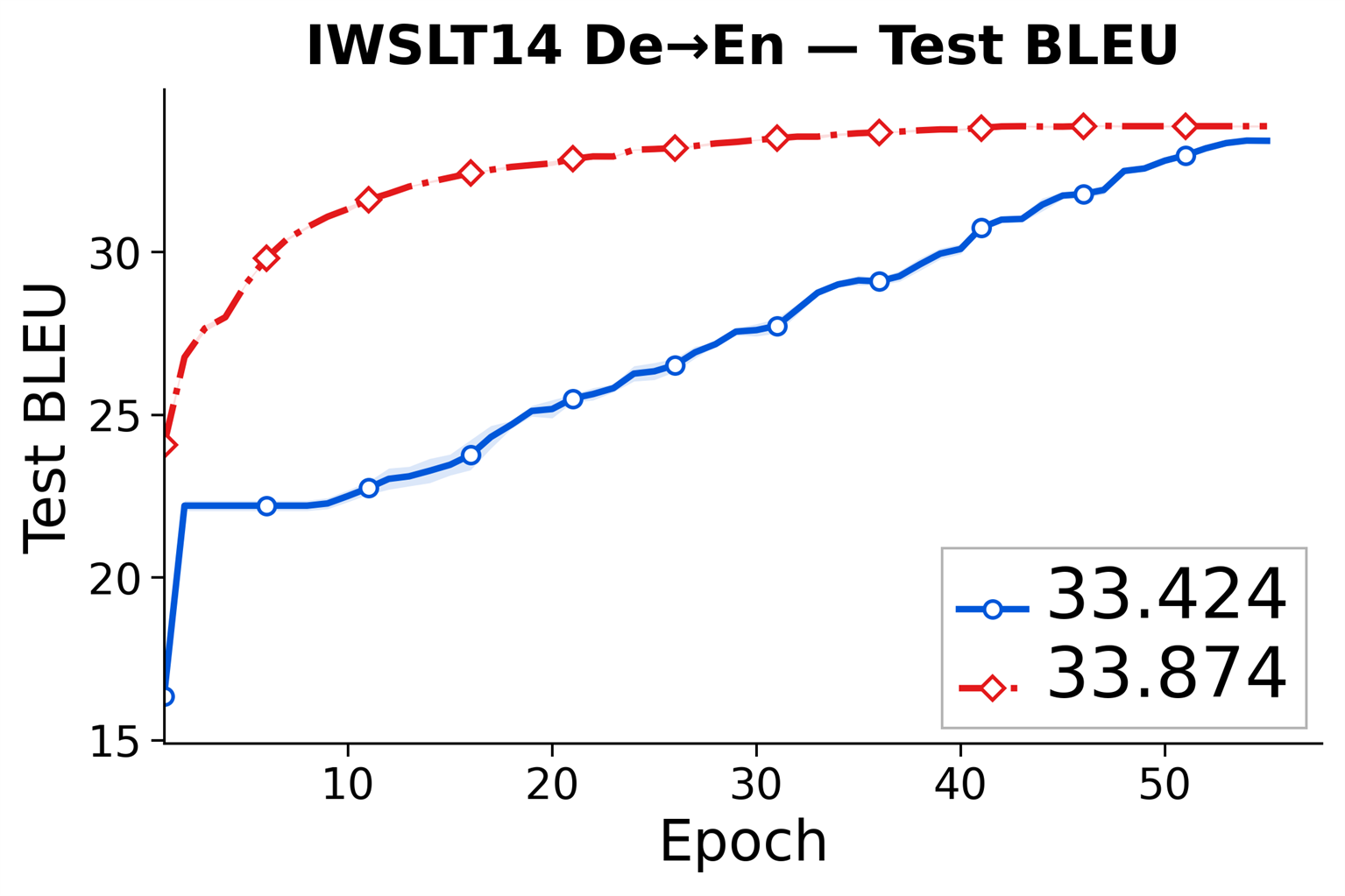}\\[-2pt]
    \includegraphics[width=0.62\linewidth]{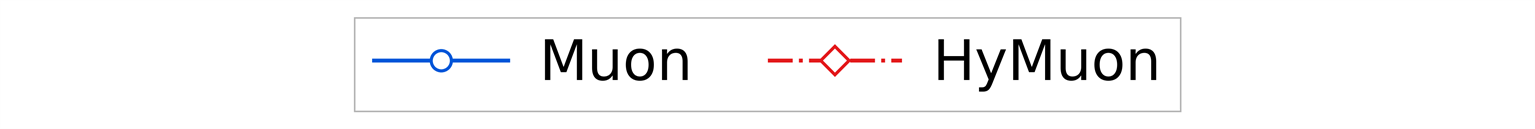}
    \caption{IWSLT14 German--English test BLEU for the repeated runs in
    Figure~\ref{fig:language-tasks}. Curves are means over repeated runs with
    bands showing one standard error; the panel legend gives the endpoint means.}
    \label{fig:iwslt-bleu}
\end{figure}

For the IWSLT14 test loss panel in Figure~\ref{fig:language-tasks} and the
BLEU panel in Figure~\ref{fig:iwslt-bleu}, the selected Muon and \hymuon{}
configurations are each evaluated over repeated runs. At each epoch, the curves
report test metrics from the checkpoint with the highest validation BLEU
observed so far. Curves are aligned by epoch and show the mean with bands
indicating one standard error.

\paragraph{Ablation of the averaged iterate.}
We additionally evaluate a \hymuon{} variant that retains the same matrix and auxiliary updates formed from the two momentum states
as the full method but evaluates and checkpoints
the training iterate \(y\), without constructing the averaged iterate \(x\). The search
covers the matrix and auxiliary learning rates, weight decay, matrix momentum,
mixing weight, and coefficient for the auxiliary second moment. The selected
configuration uses matrix and auxiliary peak learning rates \(0.05\) and
\(0.1\), matrix momentum \(0.9\), mixing weight \(0.9\), auxiliary
coefficients \((0.9,0.9)\), and weight decay \(0.05\) and \(0.1\) for the
matrix and auxiliary groups, respectively, with warmup over 4,000 steps followed
by linear decay. It attains test BLEU \(33.580\) and test loss \(4.233\).
The selected Muon and full \hymuon{} configurations attain test BLEU $33.370$ and
$33.930$, with test losses $4.229$ and $4.186$. These results show that the updates
formed from the two momentum states improve BLEU over Muon without the averaged
iterate, and that evaluation at the averaged iterate yields a further improvement.

\begin{figure}[!htbp]
    \centering
    \includegraphics[width=0.62\linewidth]{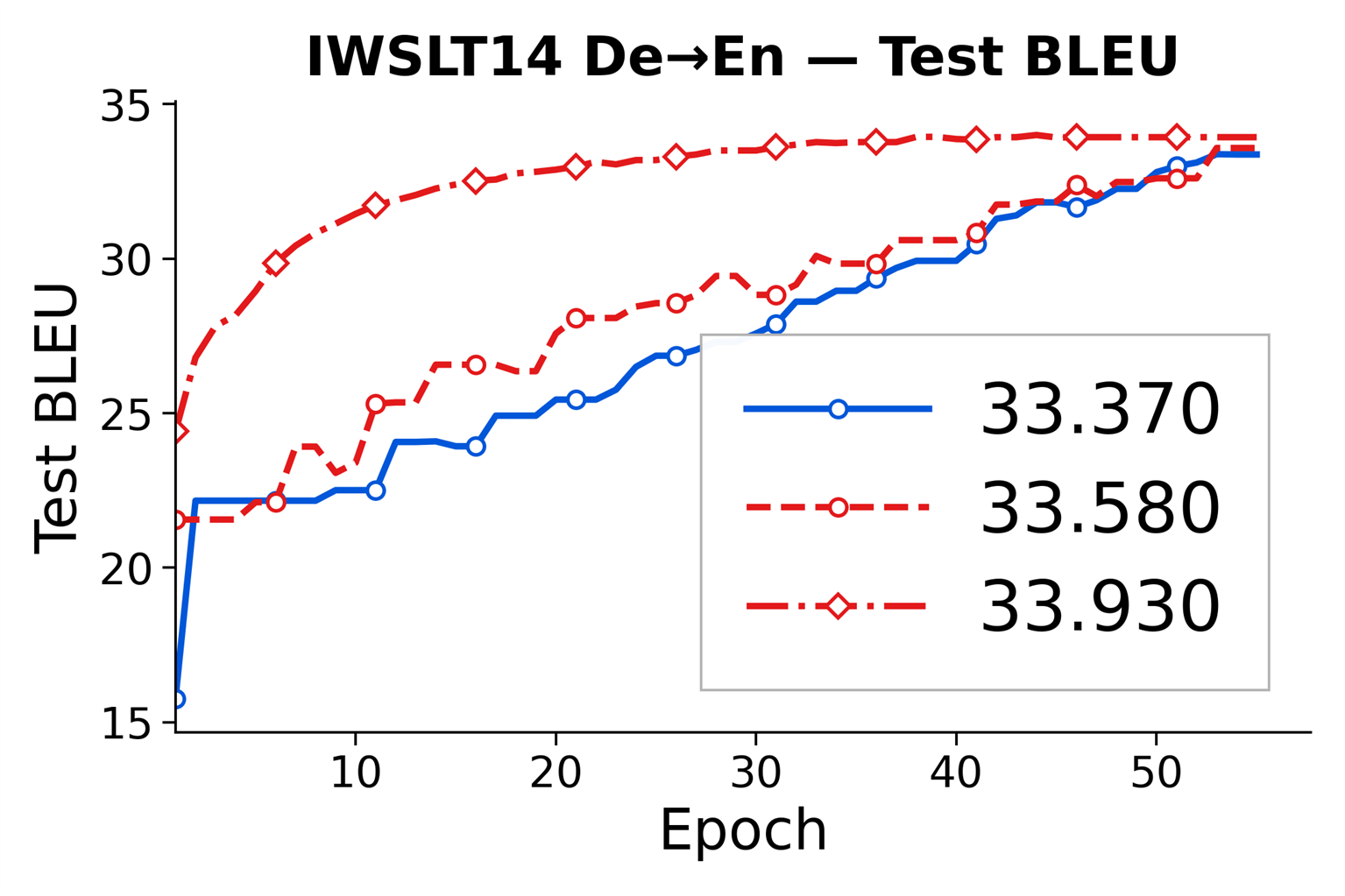}\\[-2pt]
    \includegraphics[width=0.62\linewidth]{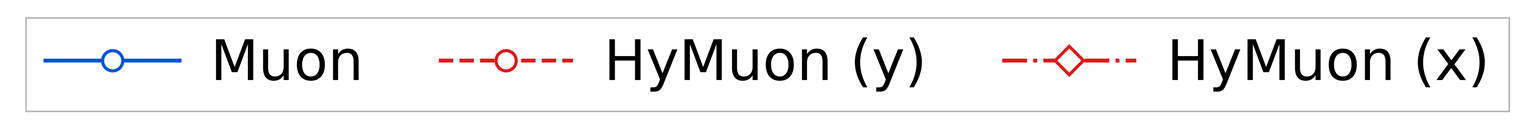}
    \caption{IWSLT14 ablation of the averaged iterate in test BLEU. \hymuon{}
    \((y)\) retains the matrix and auxiliary updates formed from the two momentum states
    and evaluates the training iterate; \hymuon{} \((x)\) evaluates the averaged
    iterate. Curves correspond to the selected configuration of each method.}
    \label{fig:iwslt-readout-ablation}
\end{figure}

\subsection{NanoGPT on OpenWebText}
\label{app:openwebtext-runs}

We train a GPT-2 Small model \citep{radford2019language} with 12 layers, 12
attention heads, hidden size 768, context length 1024, zero dropout, and 124M
parameters using nanoGPT \citep{karpathy2022nanogpt} on OpenWebText
\citep{gokaslan2019openwebtext}. A microbatch contains 12 sequences and the
global gradient accumulation factor
is 40, giving 480 sequences or 491,520 tokens per optimizer step.  Validation
is run every 1000 steps and averages 200 randomly sampled mini-batches on the
master process (2,457,600 validation tokens per evaluation). The two panels
use separate regimes of 100k and 600k steps, corresponding to 49.152B and 294.912B
processed training tokens. Table~\ref{tab:nanogpt-hyper} records the \hymuon{}
configurations used in the two regimes.
Training uses gradient clipping at norm 1 and linear warmup over 2,000 steps. Cosine decay reaches its stated minimum at the end of each
regime.  \hymuon{} uses five NS iterations on hidden matrices and
$\varepsilon=10^{-8}$ in its auxiliary adaptive branch; its auxiliary
coefficient for the first moment equals the mixing weight, $0.99$.

\begin{table}[htbp]
\centering
\caption{\hymuon{} hyperparameters in the GPT-2 124M OpenWebText experiments.}
\label{tab:nanogpt-hyper}
\renewcommand{\arraystretch}{1.15}
\begin{tabular}{|l|c|c|}
\hline
\textbf{Hyperparameter} & \textbf{100k regime} & \textbf{600k regime} \\
\hline
LR for hidden matrices & 0.01 & 0.01 \\
\hline
Minimum LR for hidden matrices & $3\times10^{-5}$ & $6\times10^{-5}$ \\
\hline
Momentum $\mu$ for hidden matrices & 0.95 & 0.95 \\
\hline
\hym{} mixing weight $\beta$ & 0.99 & 0.99 \\
\hline
Weight decay for hidden matrices & 0.2 & 0.2 \\
\hline
Auxiliary LR & 0.005 & 0.005 \\
\hline
Auxiliary minimum LR & $3\times10^{-5}$ & $6\times10^{-5}$ \\
\hline
Auxiliary $\beta_2$ & 0.999 & 0.999 \\
\hline
Auxiliary weight decay & 0.1 & 0.1 \\
\hline
Warmup steps & 2000 & 2000 \\
\hline
Learning rate schedule & cosine & cosine \\
\hline
\end{tabular}
\end{table}

In the 100k setting, AdamW uses learning rate
$6\times10^{-4}\!\to\!3\times10^{-5}$, $(\beta_1,\beta_2)=(0.9,0.95)$,
and weight decay $0.1$. Muon uses a learning rate for hidden matrices of
$2\times10^{-2}\!\to\!3\times10^{-5}$, momentum $0.95$, and zero weight decay;
its auxiliary AdamW branch uses learning rate
$3\times10^{-3}\!\to\!3\times10^{-5}$, $(\beta_1,\beta_2)=(0.9,0.95)$, and
weight decay $0.1$.

In the 600k setting, AdamW uses learning rate
$6\times10^{-4}\!\to\!6\times10^{-5}$, $(\beta_1,\beta_2)=(0.9,0.95)$,
and weight decay $0.1$. Muon uses an auxiliary learning rate of
$6\times10^{-4}$, a matrix learning rate of $0.02$, and a minimum learning
rate of $6\times10^{-5}$. Matrix momentum is $0.95$; the auxiliary
Adam branch uses
$(\beta_1,\beta_2)=(0.9,0.95)$ and $\varepsilon=10^{-8}$.  Both groups use
weight decay $0.01$.  The matrix group contains parameters with at least
two axes whose names contain neither \texttt{embed} nor \texttt{lm\_head};
the remaining parameters use the auxiliary branch. Both groups use
warmup over 2,000 steps and cosine decay through 600k steps.

Each OpenWebText panel in Figure~\ref{fig:language-tasks} shows the run using
the configuration listed above through its stated training budget. The 100k
and 600k panels are separate training regimes rather than continuations of
the same runs. The final validation losses are collected in
Table~\ref{tab:main-results}.

\subsection{LLaMA-1.3B pretraining on C4}
\label{app:c4-runs}

This experiment compares \hymuon{} with Muon for LLaMA-1.3B pretraining on C4.
We reproduce the Muon baseline reported in Table~2 of \citet{gu2026mano}
using the training configuration in their Table~5 and Appendix~B.1.
Muon's peak learning rate $3\times10^{-4}$ follows that configuration.
The model has 1.339B parameters, 24 layers, hidden size 2048, feed-forward size
5461, 32 attention heads, and a 32k vocabulary. C4/en is streamed with
document packing and a \texttt{t5-base} tokenizer; each step contains
512 sequences of length 1024.
Both methods train for 10,000 steps, 1,000 linear warmup steps, cosine
decay to one tenth of the peak learning rate, gradient clipping at 1.0, and an
evaluation on 10,017,216 validation tokens every 1,000 steps. Each method
processes 5,242,880,000 input tokens.

Muon uses the hyperparameters selected by \citet{gu2026mano}: five NS iterations,
momentum 0.95, matrix and auxiliary learning rates $3\times10^{-4}$, and
weight decay 0.1. For \hymuon{}, we search the matrix and auxiliary learning
rates, matrix momentum, mixing weight, the coefficient for the auxiliary second moment,
and weight decay for each parameter group. The selected configuration uses matrix and
auxiliary learning rates
$3\times10^{-2}$, mixing weight 0.99, Muon momentum 0.95,
coefficient for the auxiliary second moment 0.999, and weight decay 0.2 for each group.
Both parameter groups are evaluated at their averaged iterates.

At 10k steps, validation loss and perplexity are $2.903$ and $18.221$
for Muon, and $2.777$ and $16.078$ for \hymuon{}. Figure~\ref{fig:c4-perplexity}
shows the perplexity curves; they are the exponential of the loss curves in
Figure~\ref{fig:language-tasks}.

\begin{figure}[!htbp]
    \centering
    \includegraphics[width=0.62\linewidth]{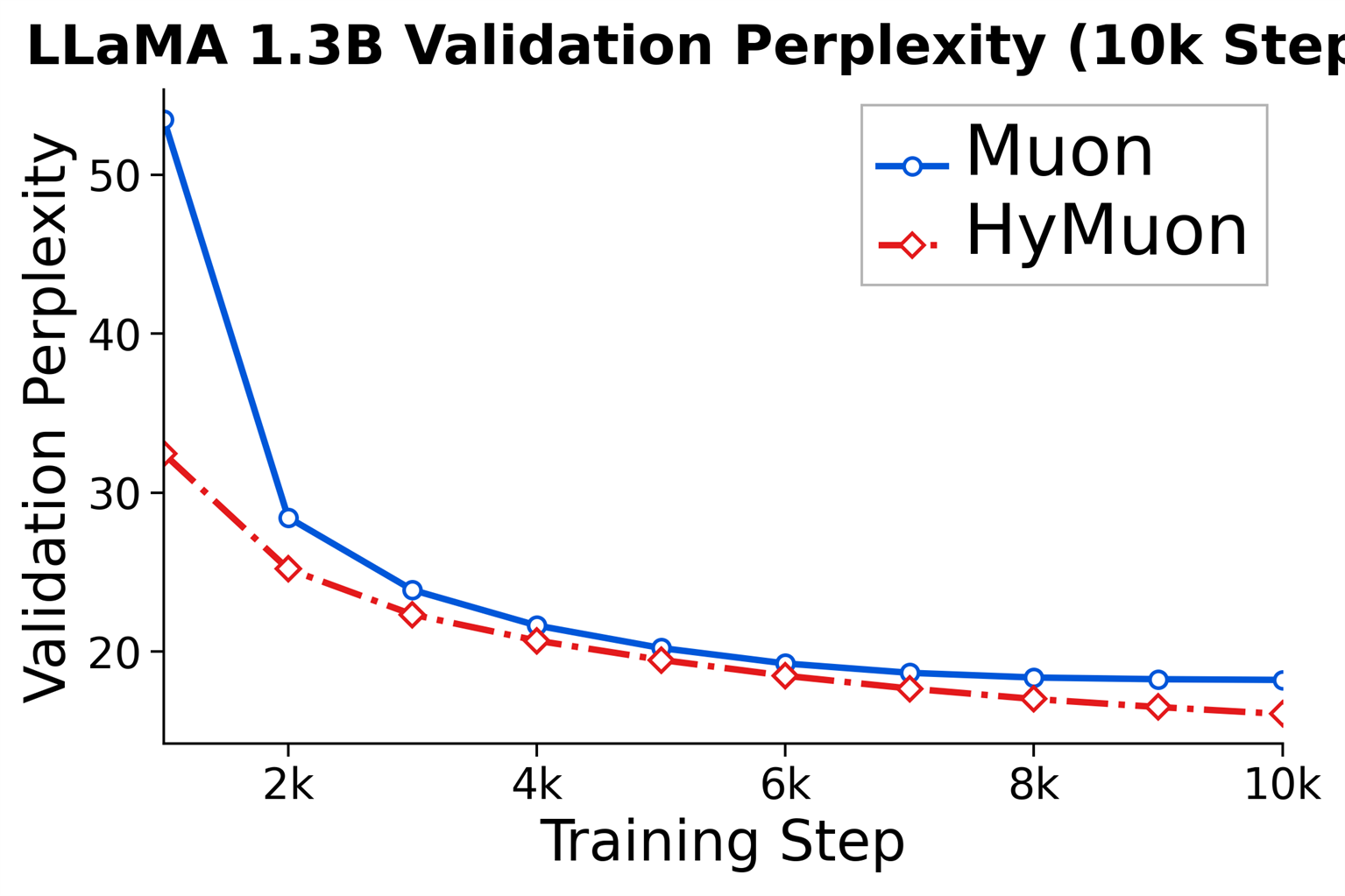}
    \caption{LLaMA-1.3B validation perplexity on C4 with document packing through 10k steps
    in the comparison against the reproduced Muon baseline. This
    is the exponential transform of the curves of validation loss in
    Figure~\ref{fig:language-tasks}; \hymuon{} is evaluated at its averaged iterate.}
    \label{fig:c4-perplexity}
\end{figure}

\subsection{LLaMA-124M FineWeb benchmark}
\label{app:fineweb-124m}

\paragraph{Protocol and data.}
We use the official framework of \citet{semenov2025benchmarking} and its
124M protocol for large batches (Section~4.1 and Figure~3(b) of that work).
We run separate experiments for 8k steps with approximately 1B tokens and
16k steps with 2.1B tokens. Each budget starts a new training run with its
own schedule for the learning rate. The model has 123.59M parameters,
12 layers, hidden size 768, 12 attention heads, SwiGLU feed-forward size
2048, RMSNorm, RoPE, and tied input/output embeddings. GPT-2 tokenization
uses a padded vocabulary of 50,304 and context length 512.

All four methods use a common FineWeb \texttt{sample-10BT} subset
and the same deterministic split. Documents are processed in shard order with an
end of text token appended; the first 10,000,001 token IDs form validation,
and subsequent documents supply 8,400,000,001 training token IDs.
Data parallelism across four ranks uses 64 sequences per rank, giving a global
batch of 256 sequences and 131,072 tokens per step. The two budgets
process 1,048,576,000 and 2,097,152,000 training tokens, respectively.
Dropout is zero, and gradients are clipped at norm 0.5. Each parameter group's learning rate warms up over
2,000 steps and then decays by a cosine schedule to 1\% of its peak.

\paragraph{Optimizer configurations.}
AdamW and AdEMAMix use the benchmark's published 124M configurations; Muon uses
its implementation of basic Muon and configuration for large batches (Figure~4
and Table~12 of \citet{semenov2025benchmarking}). The AdEMAMix algorithm
itself is from \citet{pagliardini2024ademamix}. Table~\ref{tab:fineweb-hparams}
lists the fixed configurations used for Figure~\ref{fig:fineweb-124m}.
AdamW and AdEMAMix apply weight decay to weights in linear layers, excluding
the tied embedding/output weights and normalization parameters.
For Muon and \hymuon{}, hidden attention/MLP matrices form the matrix group;
the tied embedding/output head and RMSNorm parameters form the auxiliary
group. Both matrix optimizers use five NS iterations. Basic Muon uses
Nesterov momentum and applies decay only to its
AdamW auxiliary group. \hymuon{} uses standard momentum, decoupled decay
in both groups, and the adaptive auxiliary recurrence of
Appendix~\ref{app:fast-hymuon}. Its \hym{} mixing weight is $\beta=0.9$.
The auxiliary decay coefficients are $(\rho_1,\rho_2)=(0.97,0.999)$ at
1B tokens and $(0.99,0.999)$ at 2.1B tokens. Since the model uses RMSNorm,
no BatchNorm recalibration is required.

\begin{table}[htbp]
\centering
\caption{Fixed optimizer hyperparameters for the 124M FineWeb comparison.
Baseline columns apply to both budgets; the \hymuon{} column is the 1B
configuration, and the 2.1B changes appear below the table.
Learning rates are peak values. An entry of the form $a/b$ gives the matrix
and auxiliary values; for AdamW and AdEMAMix the single rate applies to all
parameters.}
\label{tab:fineweb-hparams}
\small
\setlength{\tabcolsep}{4pt}
\begin{tabular}{lcccc}
\hline
 & AdamW & AdEMAMix & Muon & \hymuon{} \\
\hline
Learning rate & $0.001$ & $0.001$ & $0.01/0.001$ & $0.045/0.045$ \\
Weight decay & $0.1$ & $0.1$ & $0/0.1$ & $0.1/0.1$ \\
Matrix momentum & -- & -- & $0.95$ & $0.85$ \\
Adaptive $(\beta_1,\beta_2)$ & $(0.8,0.999)$ & $(0.9,0.999)$ & $(0.8,0.999)$ & $(0.97,0.999)$ \\
Adaptive $\varepsilon$ & $10^{-8}$ & $10^{-8}$ & $10^{-8}$ & $10^{-10}$ \\
\hym{} mixing weight $\beta$ & -- & -- & -- & $0.9$ \\
Slow EMA $\beta_3$ & -- & $0.999$ & -- & -- \\
Slow EMA multiplier $\alpha$ & -- & $8$ & -- & -- \\
\hline
\end{tabular}
\end{table}

At 2.1B tokens, \hymuon{} uses matrix/auxiliary learning rates
$0.015/0.015$ and auxiliary betas $(0.99,0.999)$; its remaining settings
are unchanged. AdEMAMix warms up both its slow decay and mixing multiplier
over the full training horizon: 8,000 or 16,000 steps, respectively.
\hymuon{} uses the adaptive \hym{} recurrence in
Appendix~\ref{app:fast-hymuon} for both token budgets. It computes gradients at the training iterate and
evaluates all parameter groups at the averaged iterate; after evaluation, the
training iterates are restored. AdamW,
AdEMAMix, and Muon evaluate their current training parameters.

\paragraph{Evaluation and curves.}
Validation is evaluated at initialization, every 115 steps, and at
the final step (8,000 or 16,000). Intermediate evaluations use a fixed set
of 64 batches of 64 sequences; the final evaluation uses all complete batches
of the validation reader. Figure~\ref{fig:fineweb-124m} shows the resulting
curves of validation loss.

\FloatBarrier

\clearpage
\section{Optimization efficiency on modded-nanoGPT Track 3}
\label{app:track3}

We evaluate \hymuon{} on modded-nanoGPT Track~3
\citep{yau2026emanesterov}. The evaluation compares the number of updates required
by \hymuon{} and a publicly tuned clean Muon baseline to reach a given validation
loss under the same architecture, dataset, and batch size.\footnote{Benchmark setup:
\href{https://github.com/KellerJordan/modded-nanogpt/blob/452cc392f6af4f29834d1af64edeac55d1315397/records/track_3_optimization/README.md}{modded-nanoGPT Track~3}.}

\paragraph{Setup.}
We train a GPT with 162.35M parameters (12 layers, hidden size 768, 6 heads)
on FineWeb10B with context length 1,024, 524,288 tokens per update, and
10,485,760 validation tokens. We run eight seeds (636--643) for
3,200 updates each on one H100 80GB GPU and compare against all ten
public runs of the clean Muon baseline tuned for 3,250 steps.\footnote{Baseline source:
\href{https://github.com/KellerJordan/modded-nanogpt/pull/323}{PR~\#323}
and its \href{https://github.com/KellerJordan/modded-nanogpt/tree/452cc392f6af4f29834d1af64edeac55d1315397/records/track_3_optimization/results/20260610_tuned_baseline_3250}{archived run logs}.}
For this benchmark, we use the \hymuon{} configuration developed for Track~3.
It applies the slow and fast momentum construction to hidden matrices but evaluates
the training parameters directly, without iterate averaging. The clean Muon baseline
is also evaluated at its training parameters.

\paragraph{Optimizer configuration.}
For hidden attention and MLP matrices, we combine Muon's fast EMA with
a linearly weighted average of all past gradients. The slow average has
mixing weight $0.05$ and is scaled to match the EMA's accumulated weight.
We retain Nesterov mixing, momentum $0.95$, and twelve NS iterations.
Embeddings, the output head, and parameters with one dimension use the
baseline AdamW settings, with peak learning rates $0.7$, $0.004$, and $0.015$.
All three groups use betas $(0.8,0.95)$,
$\varepsilon=10^{-10}$, and weight decay $0.001$.
Both configurations keep learning rates constant for the first $30\%$ of
training and decay them linearly over the remaining $70\%$. For matrix
parameters, \hymuon{} uses peak learning rate $0.030$ and weight decay
$0.05$, with the latter decaying alongside the learning rate. Clean Muon
uses peak learning rate $0.025$ and fixed weight decay $0.05$.

\begin{table}[H]
\centering
\caption{Track~3 results. Final validation loss is reported as
mean $\pm$ sample SD.}
\label{tab:track3}
\small
\setlength{\tabcolsep}{5pt}
\renewcommand{\arraystretch}{1.15}
\begin{tabular}{lrrrr}
\hline
Method & Runs & Updates & Tokens (B) & Validation loss \\
\hline
Clean Muon & 10 & 3,250 & 1.704 & $3.27866\pm0.00134$ \\
\hymuon{} & 8 & 3,200 & 1.678 & $3.27832\pm0.00100$ \\
\hline
\end{tabular}
\end{table}

\begin{figure}[H]
    \centering
    \includegraphics[width=0.98\linewidth]{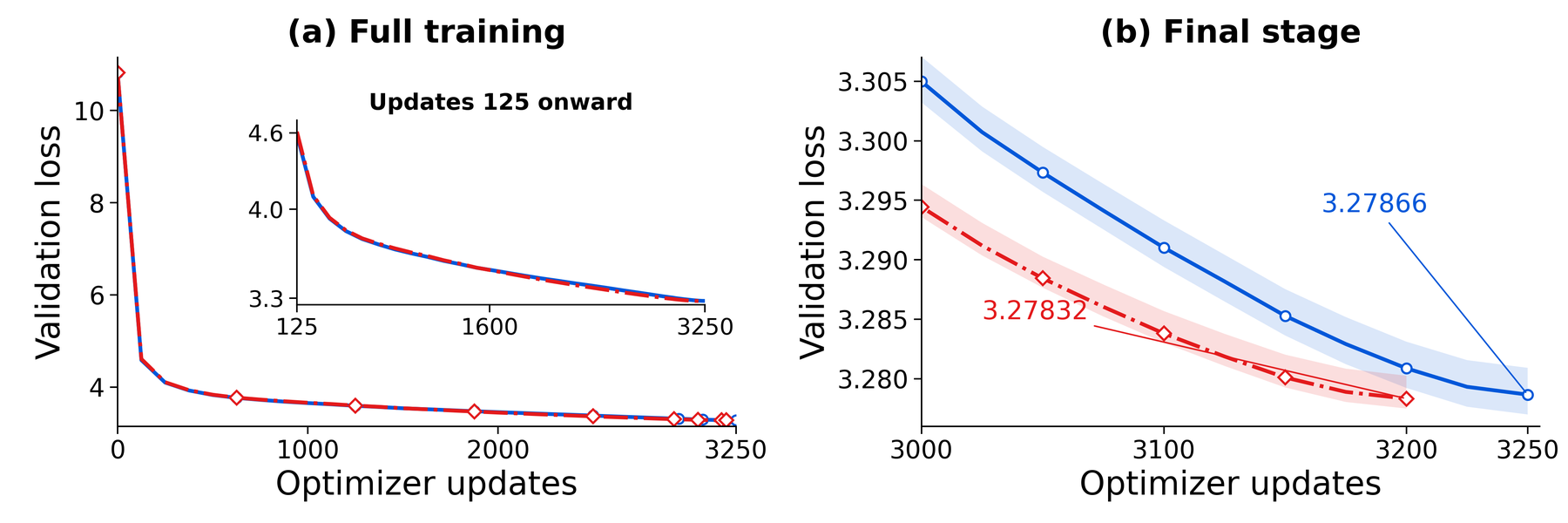}\\[-2pt]
    \includegraphics[width=0.98\linewidth]{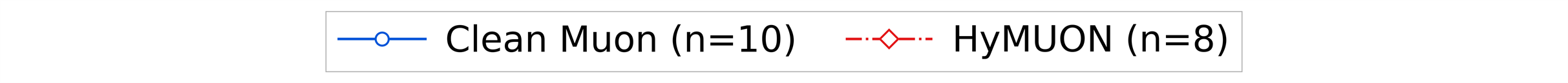}
    \caption{Track~3 validation loss over the full training runs (left)
    and the final stage (right). Curves show run means; shading indicates
    the range across runs.}
    \label{fig:track3-loss}
\end{figure}

\paragraph{Results.}
\hymuon{} achieves a mean validation loss of $3.27832$ after 3,200 updates,
compared with $3.27866$ after 3,250 updates for the publicly tuned clean
Muon baseline (Table~\ref{tab:track3}; Figure~\ref{fig:track3-loss}).
At comparable final validation loss, \hymuon{} uses 50 fewer updates
and 26.2M fewer training tokens, reducing the training budget by $1.54\%$.

\FloatBarrier